\documentclass[10pt]{article}
\usepackage[accepted]{tmlr}

\usepackage{amsmath,amsfonts,bm}

\def\eqref#1{equation~\ref{#1}}

\def\1{\bm{1}}

\DeclareMathAlphabet{\mathsfit}{\encodingdefault}{\sfdefault}{m}{sl}
\SetMathAlphabet{\mathsfit}{bold}{\encodingdefault}{\sfdefault}{bx}{n}

\newcommand{\E}{\mathbb{E}}

\newcommand{\R}{\mathbb{R}}

\newcommand{\A}{{\mathcal A}}

\newcommand{\eye}{{\bm I}}

\newcommand{\y}{{\bm y}}
\newcommand{\z}{{\bm z}}
\newcommand{\x}{{\bm x}}

\usepackage[utf8]{inputenc} 
\usepackage[T1]{fontenc}    
\usepackage{booktabs}       
\usepackage{amsfonts}       
\usepackage{nicefrac}       
\usepackage{microtype}      
\usepackage{tikz}           
\usepackage[dvipsnames]{xcolor}
\usepackage{hyperref}
\usepackage{url}
\usepackage{soul}
\usepackage{algorithm}
\usepackage{algpseudocode}
\usepackage{mathtools}
\usepackage{multirow}
\usepackage{float}
\usepackage{subcaption}
\usepackage{graphicx}
\usepackage{wrapfig}
\usepackage{amsmath}
\usepackage{amssymb}
\usepackage{bm}
\usepackage{bbm}
\usepackage{lineno}
\usepackage{enumitem}
\usepackage{placeins}
\usepackage{makecell}

\usepackage{comment}
\usepackage[colorinlistoftodos]{todonotes}

\setcitestyle{round}

\usepackage{bm}

\usepackage[nameinlink]{cleveref}

\usepackage{pifont}

\hypersetup{
	colorlinks=true,
	linkcolor=myblue,
	citecolor=myblue,
	urlcolor=mygray,
}

\definecolor{myblue}{RGB}{0,0, 150}
\definecolor{mygray}{RGB}{90,90,110}
\definecolor{mygreen}{RGB}{19,84,1}

\definecolor{equi}{RGB}{178,0,90}

\definecolor{rbtl}{RGB}{30,160,40}

\usepackage{amsthm}
\usepackage{thmtools, thm-restate}

\newtheorem{lemma}{Lemma}[section]
\newtheorem{proposition}{Proposition}[section]
\newtheorem{assumption}{Assumption}[section]
\newtheorem{definition}{Definition}[section]

\Crefname{figure}{Figure}{Figures}
\crefname{assumption}{assum.}{assums.}
\Crefname{assumption}{Assumption}{Assumptions}
\crefname{definition}{def.}{defs.}
\Crefname{definition}{Definition}{Definitions}
\crefname{algorithm}{alg.}{algs.}
\Crefname{algorithm}{Algorithm}{Algorithms}
\crefname{table}{tab.}{tabs.}
\Crefname{table}{Table}{Tables}
\crefname{proposition}{prop.}{props.}
\Crefname{proposition}{Proposition}{Propositions}
\crefname{lemma}{lemma}{lemmas}
\Crefname{lemma}{Lemma}{Lemmas}
\crefname{theorem}{theorem}{theorems}
\Crefname{theorem}{Theorem}{Theorems}

\crefname{remark}{remark}{remarks}
\Crefname{remark}{Remark}{Remarks}

\makeatletter
\renewcommand{\ALG@beginalgorithmic}{\small}
\makeatother

\usetikzlibrary{angles, quotes}

\title{EquiReg: Equivariance Regularized Diffusion\\ for Inverse Problems}

\author{\name Bahareh Tolooshams\thanks{Equal contribution.}
\email btolooshams@ualberta.ca\\
\addr University of Alberta\\ Alberta Machine Intelligence Institute (Amii)\\ Canada CIFAR AI Chair
  \AND \name Aditi Chandrashekar$^{*}$ \email ajc10180@nyu.edu \\
  \addr New York University
 \AND \name Rayhan Zirvi$^{*}$ \email rayhanzirvi@caltech.edu\\
 \addr California Institute of Technology
 \AND \name Abbas Mammadov \email abbas.mammadov@reuben.ox.ac.uk\\
 \addr University of Oxford
 \AND \name Jiachen Yao \email jiachen@caltech.edu \\
 \addr California Institute of Technology
  \AND \name Chuwei Wang \email chuweiw@caltech.edu \\
  \addr California Institute of Technology
  \AND \name Anima Anandkumar \email anima@caltech.edu\\
  \addr California Institute of Technology
}

\def\month{08}  
\def\year{2026} 
\def\openreview{\url{https://openreview.net/forum?id=3iaYKJLcyG}} 

\begin{document}

\maketitle

\begin{abstract}
Diffusion models represent the state-of-the-art for solving inverse problems such as image restoration tasks. Diffusion-based inverse solvers incorporate a likelihood term to guide prior sampling, generating data consistent with the posterior distribution. However, due to the intractability of the likelihood, most methods rely on isotropic Gaussian approximations, which can push estimates off the data manifold and produce inconsistent, poor reconstructions. We propose \textit{Equivariance Regularized} (EquiReg) diffusion, a general plug-in framework that improves posterior sampling by penalizing trajectories that deviate from the data manifold. EquiReg formalizes manifold-preferential equivariant functions that exhibit low equivariance error for on-manifold samples and high error for off-manifold ones, thereby guiding sampling toward symmetry-preserving regions of the solution space. We highlight that such functions naturally emerge when training non-equivariant models with augmentation or on data with symmetries. EquiReg's largest gains are under reduced sampling and measurement consistency steps, where many methods suffer severe quality degradation. By regularizing trajectories toward the manifold, EquiReg implicitly accelerates convergence and enables high-quality reconstructions. EquiReg consistently improves performance in linear and nonlinear image restoration tasks and solving partial differential equations. Our code is available at \url{https://github.com/Anima-Lab/EquiReg}.
\end{abstract}


\section{Introduction}\label{sec:intro}
Inverse problems aim to recover an unknown signal $\x^{*} \in \R^{d}$ from undersampled noisy measurements:
\begin{equation}\label{eq:fwd}
    \y = \A(\x^*) + \bm{\nu} \in \R^{m},
\end{equation}

where $\A$ is a known measurement operator, and 
$\bm{\nu}$ is an unknown noise~\citep{groetsch1993inverse}. Inverse problems are widely studied in science and engineering, with applications from medical imaging to astrophotography.

Inverse problems are ill-posed, i.e., the inversion process can have many solutions; hence, they require prior information about the desired solution~\citep{kabanikhin2008definitions}. In the Bayesian formulation, inference targets the posterior distribution $p(\x|\y) \propto p(\y|\x) p(\x)$, where $p(\y|\x)$ is the likelihood of the measurements and $p(\x)$ is a prior describing the signal structure~\citep{stuart2010inverse}. Examples of handcrafted priors include sparsity~\citep{donoho2006compressed} and low-rankness~\citep{candes2011robust}.

This paper focuses on methods that leverage unconditionally pre-trained score‐based generative diffusion models as learned priors~\citep{ho2020denoising,song2019generative} with applications in image restoration~\citep{chung2023diffusion}, medical imaging~\citep{chung2022mr}, and solving partial differential equations (PDEs)~\citep{huang2024diffusionpde,yao2025guideddiffusionsamplingfunction}. These methods define a sequential noising process $\x_0\sim p_{\mathrm{data}}\to \x_t\to \x_T \sim p_T(\x) \approx \mathcal{N}(\bm{0},\eye)$ and a reverse denoising process parameterized by a neural network score $\nabla_{\x_t}\log p_t(\x_t)$~\citep{vincent2011score}. During sampling, these approaches incorporate gradient signals carrying likelihood information to solve inverse problems.

\begin{figure}[t]
    \centering
    \includegraphics[width=1.0\linewidth]{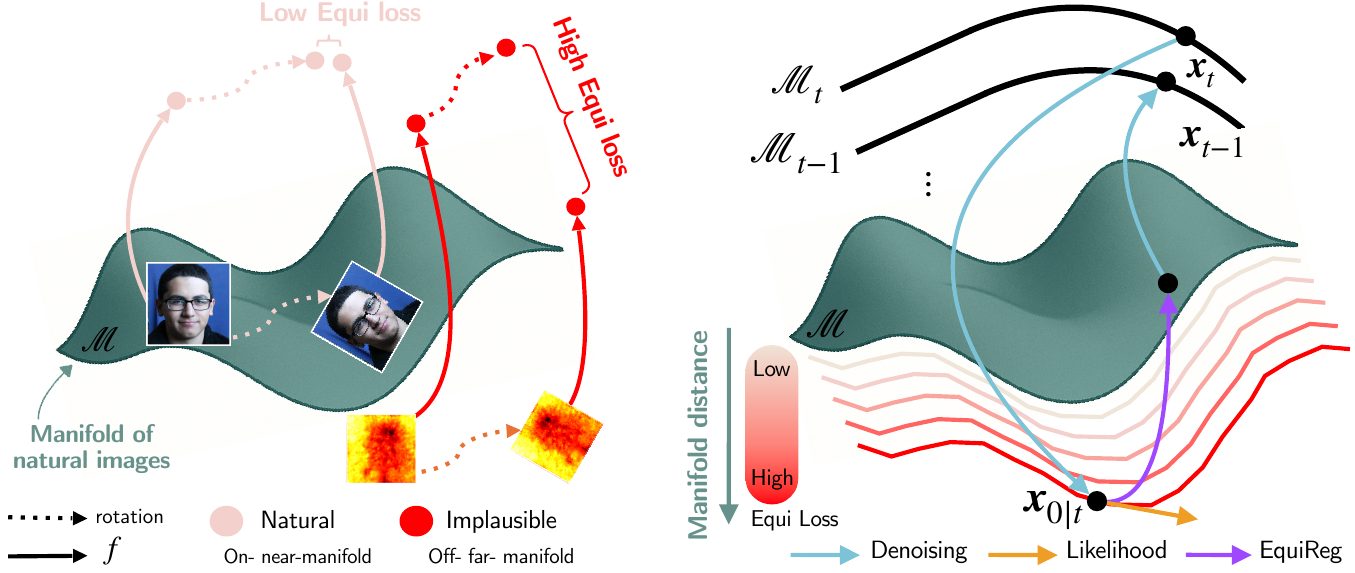} 
    \caption{\textbf{Equivariance Regularized (EquiReg) diffusion for inverse problems.} (left) Manifold preferential equivariance (MPE) functions whose equivariance error is lower for on-manifold and higher for off-manifold data. (right) EquiReg regularizes the posterior sampling trajectory for improved performance. It penalizes off-manifold trajectories via MPE-based regularization. This is a schematic illustrating the motivation of our method; the depicted EquiLoss color gradient is qualitative rather than a direct measurement of manifold distance, with empirical equivariance error shown in~\Cref{fig:mpe_mpecon_f_train}. $\mathcal{M}_t$ denotes a high-density region of $p_t$ at diffusion time $t$, rather than its support, with $\mathcal{M}_0=\mathcal{M}$ denoting the clean data manifold.}
    \label{fig:equireg} 
    \vspace{-2mm}
\end{figure}
Solving inverse problems with diffusion~\citep{zhang2025daps, alkhouri2025sitcom} requires computing the conditional score $\nabla_{\x_t} \log{p_t(\x_t|\y)}$, decomposed into $\nabla_{\x_t} \log{p_t(\x_t)} + \nabla_{\x_t} \log{p_t(\y|\x_t)}$. This introduces challenges, as the likelihood score $\nabla_{\x_t} \log{p_t(\y|\x_t)} = \nabla_{\x_t} \log{\int{p(\y|\x_0)p_t(\x_0|\x_t)\mathrm{d}\x_0}}$ is only computationally tractable when $t=0$. To handle the likelihood for $t>0$, many methods approximate the posterior $p_t(\x_0|\x_t)$ with the isotropic Gaussian distribution~\citep{zhang2025daps},
where the distribution expectation is computed using the optimal denoising score~\citep{robbins1956empirical}. The Gaussian approximation can be inaccurate for complex distributions (\Cref{fig:off_p}), leading to errors in likelihood computation, especially with point estimations~\citep{chung2023diffusion}. Since the posterior expectation is a conditional expectation, a linear combination of all possible $\x_0$, it may lie off the data manifold even when individual samples remain on it. These issues are further amplified in latent diffusion models (LDMs), introducing artifacts~\citep{rout2023solving}.

Prior work has attempted to address this challenge via projection-based~\citep{he2024manifold, zirvi2025diffusion} or decoupled optimization strategies~\citep{zhang2025daps}, aimed at reducing the propagation of measurement errors during sampling. However, they rely on the isotropic Gaussian assumption, which can lead to failures on difficult tasks or at reduced sampling steps. While higher-order statistics reduce such errors~\citep{boys2024tweedie}, most approaches employ a Gaussian approximation for its efficiency, scalability, and simplicity~\citep{alkhouri2025sitcom}, particularly in the context of large-scale LDMs~\citep{peebles2023scalable}. This raises a key question: how can we ensure that conditional diffusion samples remain on the data manifold under this approximation?

Equivariance offers a natural mechanism to keep sampling trajectories close to the data manifold. We address this challenge with a regularization scheme that leverages equivariance to improve posterior sampling by guiding trajectories toward symmetry-preserving solution spaces. Prior work has enforced equivariance directly within generation or denoising processes~\citep{chen2023imaging, terris2024equivariant}, with extensions to probabilistic symmetries~\citep{bloem2020probabilistic} enabling sample efficiency~\citep{wang2024equivariant}. 

\begin{figure}[t]
    \centering
     \includegraphics[width=1.0\textwidth]{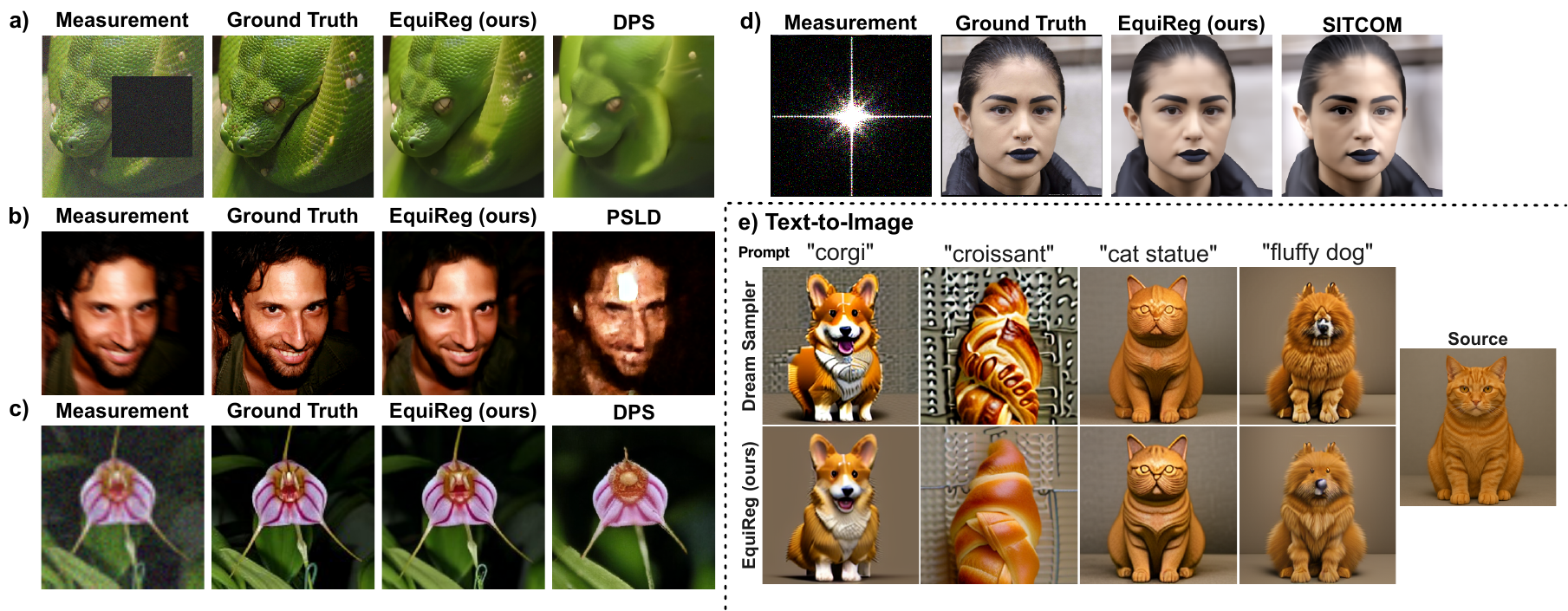} 
    \vspace{-4mm}
    \caption{\textbf{EquiReg's broad applicability.} a-d) image restoration inverse problems and e) text-guided image generation, resulting in artifact reduction and more realistic generation. Here, EquiReg refers to our regularization being applied to the diffusion sampling method on the same row.}
    \vspace{-4mm}
    \label{fig:summary_fig}
\end{figure}
Our approach differs as follows: rather than strictly enforcing equivariance within denoising architectures, which can hinder tasks requiring symmetry breaking~\citep{lawrence2025improving}, we employ equivariance as a plug-in regularizer to guide diffusion trajectories toward the data manifold.

\textbf{Our contributions.}\quad We propose \textit{Equivariance Regularized} (EquiReg) diffusion, an equivariance-based regularization framework for solving inverse problems with diffusion models (\Cref{fig:equireg,fig:summary_fig}). EquiReg leverages equivariance to \emph{regularize} likelihood-induced errors during posterior sampling, guiding diffusion trajectories toward more consistent, on-manifold solutions. Crucially, it employs \textit{Manifold-Preferential Equivariant} (MPE) functions, which discriminate on-manifold from off-manifold data by exhibiting low equivariance error in-distribution and higher error out-of-distribution.
\begin{figure}[t]
    \centering
    \includegraphics[width=0.95\linewidth]{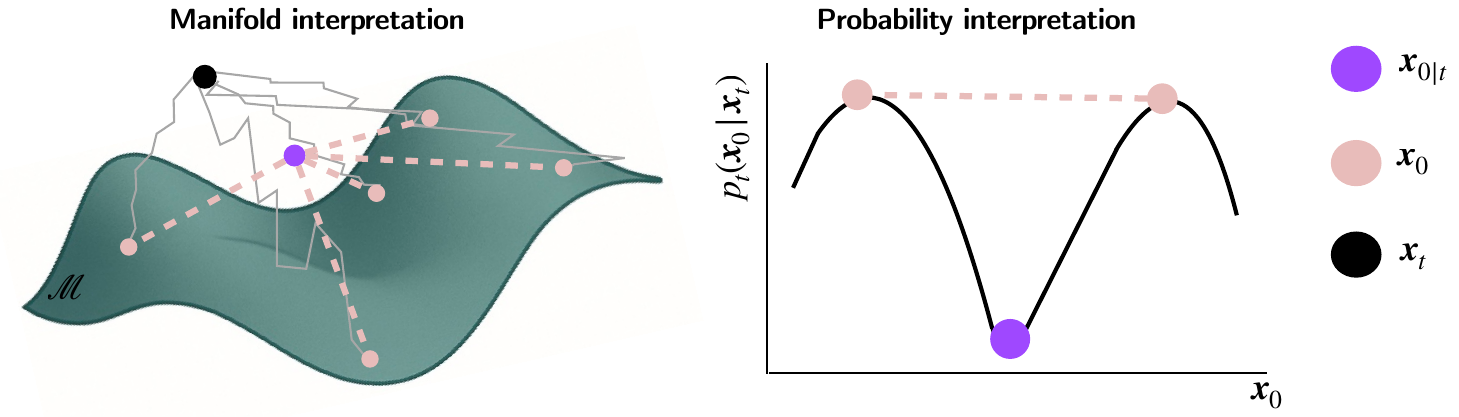} 
    \vspace{-5mm}
    \caption{\textbf{Off-manifold posterior expectation.} This impacts the likelihood score $p_t(\y|\x_t) = \int{p(\y|\x_0)p_t(\x_0|\x_t)\mathrm{d}\x_0}$ computation achieved via isotropic Gaussian modelling of $p_t(\x_0|\x_t)$. The right panel is a schematic 1D illustration of $p_t(\x_0|\x_t)$ for a fixed $\x_t$, with the curve being the (bimodal) posterior density. The Tweedie posterior mean $\x_{0|t}$ is a value on the $\x_0$-axis, and the pink dotted segment connects the two modes to emphasize that $\x_{0|t}$ falls between them in a low-density region.}
    \vspace{-7mm}
\label{fig:off_p}
\end{figure}
We formalize that a regularizer for diffusion posterior sampling should capture such a global property, and MPE functions provide a principled way to direct sampling toward plausible solutions. This design makes EquiReg architecture-agnostic: the regularizer operates independently of the diffusion model itself. With a suitable MPE function, EquiReg improves performance across models, including those with equivariant scores, where likelihood guidance may otherwise push trajectories off the manifold.

We observe that many trained neural networks behave as MPEs: their equivariance error is small on the training or data manifold but grows off-manifold. This behavior arises in learned models trained with data augmentation, as well as in models trained on data with inherent symmetries such as those from physical systems. Rather than treating the degradation off-manifold as a limitation, we exploit it as a signal: equivariance error serves as a natural discriminator for identifying undesirable states during diffusion sampling. Building on this idea, we construct pre-trained MPEs as the foundation of our EquiReg loss. The choice of this function is independent of the denoiser in diffusion models and can be derived separately. For instance, if the diffusion architecture is itself equivariant to a set of group actions, it cannot be regularized via EquiReg using these same equivariances, as the MPE cannot discriminate between on- and off-manifold samples (the equivariance loss would remain low). In this case, a separate MPE with different equivariance properties can be trained. We explore several constructed and pre-trained MPE functions and show that EquiReg's performance is stable under different choices of MPE.

Beyond architectural flexibility, we systematically analyze the diversity properties of EquiReg-guided posterior sampling. We demonstrate that EquiReg achieves improved fidelity without posterior collapse, and as inverse problems become more ill-posed, diversity metrics (intra-LPIPS and pixel-std) increase linearly with task difficulty (\Cref{fig:diversity_difficulty}), indicating that EquiReg naturally expands posterior exploration with growing uncertainty rather than collapsing to a single mode.

We further demonstrate that MPE functions are not rare or specialized constructs, but emerge broadly across widely-used pre-trained networks. We empirically analyze multiple architectures, including the latent diffusion encoder~\citep{rombach2022high}, CNN autoencoders trained with symmetry-based augmentations, pre-trained ResNet-50~\citep{he2016deep}, and CLIP encoders~\citep{radford2021learning}, and show that their equivariance error increases systematically as Gaussian noise pushes samples off the data manifold. We further show that using these MPE functions for EquiReg guidance improves solutions to inverse problems. Overall, this study shows that MPEs are widespread, and EquiReg uses this property as the basis for the regularization loss.

We validate the efficacy of EquiReg through extensive experiments across diverse diffusion models, inverse problems, and datasets. We demonstrate that EquiReg improves perceptual image quality and remains effective in cases where baselines fail. We show that EquiReg improves the performance of SITCOM~\citep{alkhouri2025sitcom} and DPS~\citep{chung2023diffusion} when the number of measurement consistency and sampling steps are reduced, thus moving toward more efficient diffusion-based solvers. The improvement is largest when applied to LDMs. EquiReg reduces failure cases and consistently improves PSLD~\citep{rout2023solving} and ReSample~\citep{song2023solving} on linear and nonlinear image restoration tasks. For example, EquiReg significantly improves the LPIPS~\citep{song2023solving} of ReSample by $51\%$ for motion deblur and the FID of DPS~\citep{chung2023diffusion} by $59\%$ on super-resolution. We also include diversity analyses, demonstrating that EquiReg maintains diversity without collapse of single mode reconstruction.

We extend EquiReg's applicability to function-space diffusion models and demonstrate its added benefit for solving PDEs. EquiReg achieves a $7.3\%$ relative reduction in the $\ell_2$ error of FunDPS~\citep{mammadovdiffusion,yao2025guideddiffusionsamplingfunction} on the Helmholtz equation 
and a $7.5\%$ relative reduction on the Navier-Stokes equation. Lastly, we include preliminary experiments on EquiReg improving the realism and plausibility of text-guided image generation, emphasizing that the benefits of EquiReg extend beyond image restorations. Overall, the flexibility of EquiReg as a plug-in regularization framework suggests that its utility will extend well beyond the specific methods studied in this paper.
%
\section{Preliminaries and Related Works}\label{sec:prelim}
%
\textbf{Diffusion models.}\quad  Diffusion generative models~\citep{ho2020denoising, song2019generative, sohl2015deep, kadkhodaie2021stochastic} are state-of-the-art in computer vision for image~\citep{esser2024scaling} and video generation~\citep{brooks2024video, zhang2025step}, with score-based methods~\citep{song2021score} being among the most widely used. Diffusion models generate data via a reverse noising process. The forward noising process transforms the data sample $\x_{0} \sim p_{\text{data}}$ via a series of additive noise into an approximately Gaussian distribution ($p_{\text{data}} \rightarrow p_{t} \rightarrow \mathcal{N}(0, I)$ as $t \rightarrow \infty$), described by the stochastic differential equation (SDE) $\mathrm{d}\x = - \tfrac{\beta_t}{2}\x\mathrm{d}t + \sqrt{\beta_t}\mathrm{d}{\bm w}$, where $\bm{w}$ is a standard Wiener process, and the drift and diffusion coefficients are parameterized by a monotonically increasing noise scheduler $\beta_t\in (0,1)$ in time $t$~\citep{ho2020denoising}. The time reversal of the forward diffusion process is itself a diffusion process~\citep{anderson1982reverse}
\begin{equation}\label{eqn:norm}
\vspace{-2mm}
\mathrm{d}\x = [-\tfrac{\beta_t}{2}\x - \beta_t\nabla_{\x_t}\log{p_t(\x_t)}]\ \mathrm{d}t + \sqrt{\beta_t}\mathrm{d}\bm{\bar w}
\end{equation}
with $\mathrm{d}t$ moving backward in time or in discrete steps from $T$ to $0$. This reverse SDE is used to sample data from the distribution $p_{\text{data}}$, where the unknown gradient $\nabla_{\x_t}\log{p_t(\x_t)}$ is approximated by a scoring function $s_{\theta}(\x_t,t)$, parameterized by a neural network and learned via denoising score matching methods~\citep{hyvarinen2005estimation, vincent2011score}. Solving inverse problems is described as a conditional generation where the data is sampled from the posterior $p(\x|\y)$:
\begin{equation}\label{eqn:ideal_reverse}
\mathrm{d}\x = [-\tfrac{\beta_t}{2}\x - \beta_t(\nabla_{\x_t}\log{p_t(\x_t)} + \nabla_{\x_t}\log{p_t(\y | \x_t)})]\mathrm{d}t + \sqrt{\beta_t}\mathrm{d}\bm{\bar w}
\end{equation}
For solving general inverse problems where the diffusion is \emph{pre-trained} unconditionally, the prior score $\nabla_{\x_t}\log{p_t(\x_t)}$ can be estimated using $s_{\theta}(\x_t,t)$. However, the likelihood score $\nabla_{\x_t}\log{p_t(\y|\x_t)}$ is only known at $t=0$, otherwise it is computationally intractable.

\textbf{Diffusion models for inverse problems.}\quad Solving inverse problems with pre-trained diffusion models requires approximating the intractable likelihood score $\nabla_{\x_t}\log{p_t(\y|\x_t)}$. Training-free solvers differ in how they approximate $p_t(\y|\x_t)$ and combine it with the prior $p_t(\x_t)$~\citep{peng2024improving}. Since $p_t(\y|\x_t) = \int{p(\y|\x_0)p_t(\x_0|\x_t)\mathrm{d}\x_0}$, the common choice is to approximate $p_t(\x_0|\x_t)$ by an isotropic Gaussian $\mathcal{N}(\x_{0|t}, r_t^2\eye)$~\citep{chung2023diffusion, song2023pseudoinverseguided, zhu2023denoising, zhang2025daps}. With an optimal denoising score $s_\theta(\x_t,t)$, the posterior mean $\x_{0|t} \coloneqq \E[\x_0|\x_t]$ follows from Tweedie's formula~\citep{robbins1956empirical,miyasawa1961empirical,efron2011tweedie}. While an MMSE estimate, $p_t(\x_0|\x_t)$ may not be concentrated around its mean for complex or multimodal distributions, leading to off-manifold solutions (see~\Cref{fig:off_p}).

\textbf{Equivariance.}\quad Equivariance describes how functions transform predictably under group actions and provides a principled way to incorporate symmetry into deep learning~\citep{bronstein2021geometric}. It has been applied to graphs~\citep{satorras2021n}, convolutional networks~\citep{cohen2016group, romero2022learning}, Lie groups for dynamical systems~\citep{finzi2020generalizing}, and diffusion models~\citep{wang2024equivariant}, with applications spanning molecular generation~\citep{hoogeboom2022equivariant, cornet2024equivariant}, autonomous driving~\citep{chen2023equidiff}, robotics~\citep{brehmer2023edgi}, crystal structure prediction~\citep{jiao2023crystal}, and audio inverse problems~\citep{moliner2023audioequi}. Equivariance guidance has also improved temporal consistency in video generation~\citep{daras2024warped}, and its role as a prior in inverse problems is theoretically supported in compressed sensing~\citep{tachella2023sensing}.
\begin{definition}[Equivariance]\label{def:equivariance}
Let $G$ act on $\mathcal Z$ via $T_g:\mathcal Z\to\mathcal Z$ and on $\mathcal X$ via $S_g:\mathcal X\to\mathcal X$. 
A function $f:\mathcal Z\to\mathcal X$ is equivariant if for all $g\in G$ and $\z\in\mathcal Z$, $f(T_g(\z))=S_g(f(\z))$.
\end{definition}
An equivariant function preserves structure under group transformations (\Cref{def:equivariance}). While prior work leverages exact equivariance to encode symmetries directly into neural networks, recent studies investigate approximate equivariance to relax strict symmetry assumptions that may not hold in real-world data, aiming to improve performance~\citep{wang2022approximately}. These works introduce a formal definition of approximate equivariance (\Cref{def:approxequivairance}) and an equivariance error to quantify deviations from perfect symmetry.
\begin{definition}[Approximate Equivariant Functions]\label{def:approxequivairance}
Let $G$ act on $\mathcal Z$ via $T_g:\mathcal Z\to\mathcal Z$ and on $\mathcal X$ via $S_g:\mathcal X\to\mathcal X$. A function $f: \mathcal{Z} \rightarrow \mathcal{X}$ is $\epsilon$-approximate equivariant if for all $g\in G$ and $\z \in \mathcal{Z}$,  $\| S_g(f(\z)) - f(T_g(\z)) \| \leq \epsilon$. The equivariance error of the function $f: \mathcal{Z} \rightarrow \mathcal{X}$ is defined as $\sup_{\z, g} \| S_g(f(\z)) - f(T_g(\z)) \|$. Hence, $f$ is $\epsilon$-approximate equivariant iff its error $\leq \epsilon$.
\end{definition}

Several recent works incorporate equivariance into image restoration and posterior sampling, and EquiReg differs from each in a specific way. \citet{terris2024equivariant} use equivariance to symmetrize the denoiser inside a plug-and-play fixed-point iteration, producing a deterministic MAP-style estimate. \citet{hoogeboom2022equivariant} build equivariance directly into the architecture of the diffusion model, producing denoisers that are \emph{exactly} equivariant by construction (e.g., E(n)-equivariant networks for 3D molecule generation); a broader theoretical line of work studies probabilistic symmetries and invariant neural networks~\citep{bloem2020probabilistic}. \citet{daras2024warped} use equivariance to enforce temporal consistency across video frames in image diffusion models. EquiReg sits in a different part of this design space. Equivariance enters as a \emph{posterior-sampling regularizer}, applied through an MPE function that is \emph{separate} from the diffusion denoiser and only \emph{approximately} equivariant, and the equivariance error acts as a manifold-preference signal rather than as a hard constraint. This decoupling allows EquiReg to be applied on top of any existing diffusion-based inverse solver without modifying the underlying denoiser, and accommodates symmetry-breaking tasks where architectural equivariance would be inappropriate.

We use the term manifold which refers to the data manifold hypothesis (see~\Cref{assum:manifoldhyp} )~\citep{cayton2005algorithms} that assumes data is sampled from a low-dimensional manifold embedded in a high-dimensional space. This hypothesis is popular in machine learning~\citep{bordt2023manifold} and diffusion-based solvers~\citep{he2024manifold,chung2022improving,chung2023diffusion}, supported by empirical evidence for imaging~\citep{weinberger2006unsupervised}.

\section{EquiReg: Equivariance Regularized Diffusion}\label{sec:methods}
\textbf{Regularizing diffusion models.}\quad  We begin by presenting a generalized regularization framework for improving diffusion-based inverse solvers. We focus on the property of \textit{equivariance} and introduce a new class of functions whose equivariance errors are distribution-dependent (low for on- or near-manifold samples and high for off-manifold samples). We leverage these functions to regularize diffusion models, guiding sampling trajectories toward better inverse solutions.

This paper addresses the propagation error introduced by the approximation of posterior $p_t(\x_0|\x_t)$ by incorporating an explicit regularization term. The proposed framework is general and can be applied as a plug-in on a wide range of pixel and latent-space diffusion models. Given $p_t(\y|\x_t) = \int{p(\y|\x_0)p_t(\x_0|\x_t)\mathrm{d}\x_0}$, let $\tilde p_t(\x_0|\x_t)$ denote an approximation of the posterior to make the likelihood tractable. We formulate the regularized reverse diffusion dynamics as
\begin{equation}\label{eqn:reg_diffusion}
\mathrm{d}\x = [-\tfrac{\beta_t}{2}\x - \beta_t\nabla_{\x_t}\big(\log{p_t(\x_t)} + \log{{\scalebox{0.9}{$\int$}}p(\y|\x_0)\tilde p_t(\x_0|\x_t)\mathrm{d}\x_0} - \mathcal{R}(\x_t)\big)]\mathrm{d}t + \sqrt{\beta_t}\mathrm{d}\bm{\bar w},
\end{equation}
\begin{wrapfigure}[12]{r}{0.5\textwidth}
\vspace{-7mm}
    \begin{minipage}{1.0\linewidth}
    \begin{algorithm}[H]
        \caption{Equi-DPS for Inverse Problems.}
        \begin{algorithmic}[1]
            \label{alg:equidps}
            \Require $T, \bm{y}, \{\zeta_t\}_{t=1}^T, \{\tilde{\sigma}_t\}_{t=1}^T, \bm{s}_\theta, \mathcal{R}(\cdot)$, $\{\lambda_t\}_{t=1}^T$
            \State $\x_T \sim \mathcal{N}(\bm{0}, \bm{I})$
            \For {$t = T-1$ \textbf{to} $0$}
                \State $\hat{\bm{s}} \gets \bm{s}_\theta(\x_t, t)$
                \State ${\x}_{0|t} \gets \frac{1}{\sqrt{\Bar{\alpha_t}}}(\x_t + (1 - \Bar{\alpha_t})\hat{\bm{s}})$
                \State $\bm\epsilon \sim \mathcal{N}(\bm{0}, \bm{I})$
                \State $\x'_{t-1} \gets \frac{\sqrt{\alpha_t}(1-\Bar{\alpha}_{t-1})}{1-\Bar{\alpha}_t} \x_t + \frac{\sqrt{\Bar{\alpha}_{t-1}}\beta_t}{1-\Bar{\alpha}_t} {\x}_{0|t} + \tilde{\sigma}_t \bm\epsilon$
                \State $\x_{t-1} \gets \x'_{t-1} - \zeta_t( \nabla_{\x_t} \|\bm{y} - \mathcal{A}({\x}_{0|t})\|^2_2 +$ $\lambda_t \nabla_{\x_t} \mathcal{R}(\x_t)$)
            \EndFor
            \State \Return ${\x}_{0}$
        \end{algorithmic}
    \end{algorithm}
    \end{minipage}
\end{wrapfigure}
where $\mathcal{R}(\x_t)$ is the regularizer. \Cref{eqn:norm,eqn:ideal_reverse,eqn:reg_diffusion} use the reverse-time convention, with $t$ decreasing from $T$ to $0$. Thus, the prior and likelihood score terms act to increase log-density, whereas the regularization term acts to decrease $\mathcal{R}$.

Applying this to DPS~\citep{chung2023diffusion} takes the form of~\Cref{alg:equidps}. This formulation brings us to the key contribution: how to design the regularizer to improve posterior sampling of diffusion models. Regularization in optimization aims to penalize undesired solutions, moving the algorithm towards solutions with low regularizer value. In the context of diffusion models for sampling from data distributions, we interpret an ideal regularizer as follows: an ideal regularizer should yield low values for on-manifold and high values for off-manifold samples, enabling accurate posterior sampling even when the likelihood score is approximated.

In terms of sampling dynamics, i.e., when applied at each reverse-diffusion step, the regularizer should effectively penalize trajectories leaving the data manifold and reinforce those aligned with high-probability regions. This motivates designing a regularizer that applies a global correction to the entire functional, in contrast to prior works that focus only on locally reducing likelihood error. The ideal property of a regularizer would be to produce high error on undesirable samples and low error on desirable samples.

We offer a Wasserstein-flow perspective on conditional diffusion sampling. Reinterpreting the ideal reverse-conditional dynamics as a time-inhomogeneous Wasserstein gradient flow~\citep{ferreira2018gradient} identifies the functional in~\Cref{prop:one}. This perspective motivates viewing EquiReg as a regularization of the reverse sampling dynamics toward states with lower manifold-preference loss.

\textbf{Equivariance-based regularizer.}\quad We use equivariance, a global property that enforces geometric symmetries to instantiate a regularizer to guide the diffusion process toward the data manifold. To realize this idea, we seek functions that exhibit approximate equivariance and discriminate on- from off-manifold samples.

We propose to quantify the equivariance of a function relative to a data distribution. Specifically, while the literature has primarily studied the equivariance properties of functions for general inputs, we propose a new definition for functions in which their equivariance error is distribution-dependent and defined under the support of an input data distribution (\Cref{def:equierror_distdept}).
\begin{definition}[Distribution-Dependent Equivariant Functions]\label{def:equierror_distdept}
Let $G$ act on $\mathcal Z$ via $T_g:\mathcal Z\to\mathcal Z$ and on $\mathcal X$ via $S_g:\mathcal X\to\mathcal X$. The equivariance error of the function $f: \mathcal{Z} \rightarrow \mathcal{X}$ under the distribution $p$ is defined as $\sup_{g} \E_{\z \sim p} [\| S_g(f(\z)) - f(T_g(\z)) \|]$.
\end{definition}
The equivariance error (\Cref{def:equierror_distdept}) is a \emph{pointwise} property of a learned function $f$ at a sample $\z$, distinct from \emph{distributional} invariance (whether the distribution itself is preserved under the group action). For instance, the distribution of i.i.d.\ Gaussian noise is invariant under arbitrary pixel permutations, yet a network trained on natural images is not pointwise equivariant on noise samples; different noise draws produce wildly different outputs. This definition enables us to define functions whose equivariance error can differentiate on-manifold samples from off-manifold ones. We aim to find functions whose equivariance error is low for on-manifold data and high elsewhere. The equivariance error is non-local, defined at the distribution level. When used to regularize the reverse conditional diffusion process, it is computed via local evaluations over the sampled data.

To define our method, we term a class of \emph{manifold-preferential equivariant (MPE)} functions, whose equivariance error is lower for samples on the data manifold than for off-manifold samples. EquiReg is a regularization framework, not a manifold projection method. EquiReg penalizes states that deviate from symmetry-preserving regions; when an MPE function is used, these regions align with the data manifold. In practice, MPE functions can emerge from training with symmetry-preserving augmentation and from inherent data symmetries; we illustrate both with concrete examples below. MPE can emerge when functions are trained with symmetry-preserving mechanisms such as data augmentation. Prior work has studied equivariant properties of learned representations in deep networks~\citep{lenc2015understanding}, showing that data augmentations~\citep{krizhevsky2012imagenet} and representation compression via reduced model capacity~\citep{bruintjes2023affects} promote equivariant features even when equivariance is not explicitly built into the architecture. Importantly, the trained network is only approximately equivariant, and prior studies have noted that symmetry-preserving properties degrade for inputs deviating from in-distribution data~\citep{azulay2019deep}. A few studies have leveraged this emergent MPE in trained networks for out-of-distribution detection~\citep{zhou2022rethinking, kaur2022idecode, kaur2023codit}.

To demonstrate the widespread MPE property of learned mappings, we have considered additional pre-trained models and quantified their equivariance loss for several datasets, i.e., natural images and corrupted ones (see~\Cref{app:mpe} of Appendix.) \Cref{fig:mpe_mpecon_f_train} illustrates the MPE property, emergent via training with augmentations, of $\mathcal{E}$-$\mathcal{D}$ of a pre-trained autoencoder of an LDM. Specifically, it shows that the equivariance error is lower for natural images and increases when images deviate from the clean data distribution. Based on~\Cref{def:equierror_distdept}, we propose the \textit{Equi} loss for diffusion-based inverse solvers using an MPE function $f$, defined as
\begin{equation}
\label{eq:equi_losses}
\begin{aligned}
\mathcal{R}_{f_{\text{MPE}}} (\x_t) &= \| S_g(f_{\text{MPE}}(\x_{0|t})) - f_{\text{MPE}}(T_g(\x_{0|t}))\|_2^2\\
\end{aligned}
\end{equation}
where $\x_{0|t}$ and $\z_{0|t}$ are functions of $\x_t$ and $\z_t$, respectively.

%
\begin{figure}[t]
\vspace{-7mm}
    \centering  
    \begin{subfigure}[b]{0.485\linewidth}
    \centering    
    \includegraphics[width=1.0\linewidth]{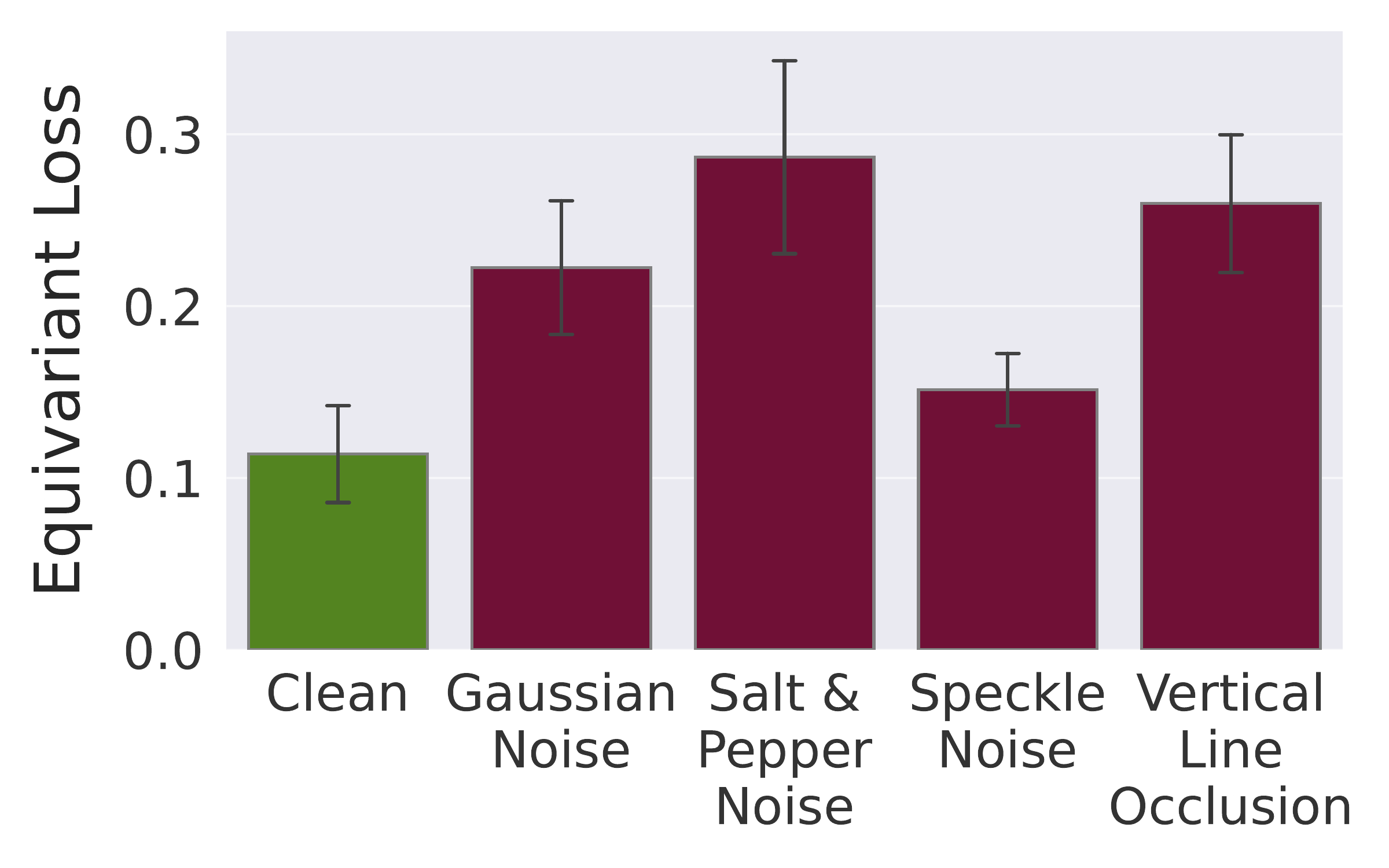}
     \vspace{-5mm}
    \caption{Equivariance from augmentation, error computed with $f_{\text{MPE}} = \mathcal{E}$ from autoencoder pretrained on FFHQ-256 and evaluated on examples from FFHQ-256.}
    \label{fig:mpe_mpecon_f_train}
    \end{subfigure}
    \hfill
    \begin{subfigure}[b]{0.485\linewidth}
    \centering    
    \includegraphics[width=1.0\linewidth]{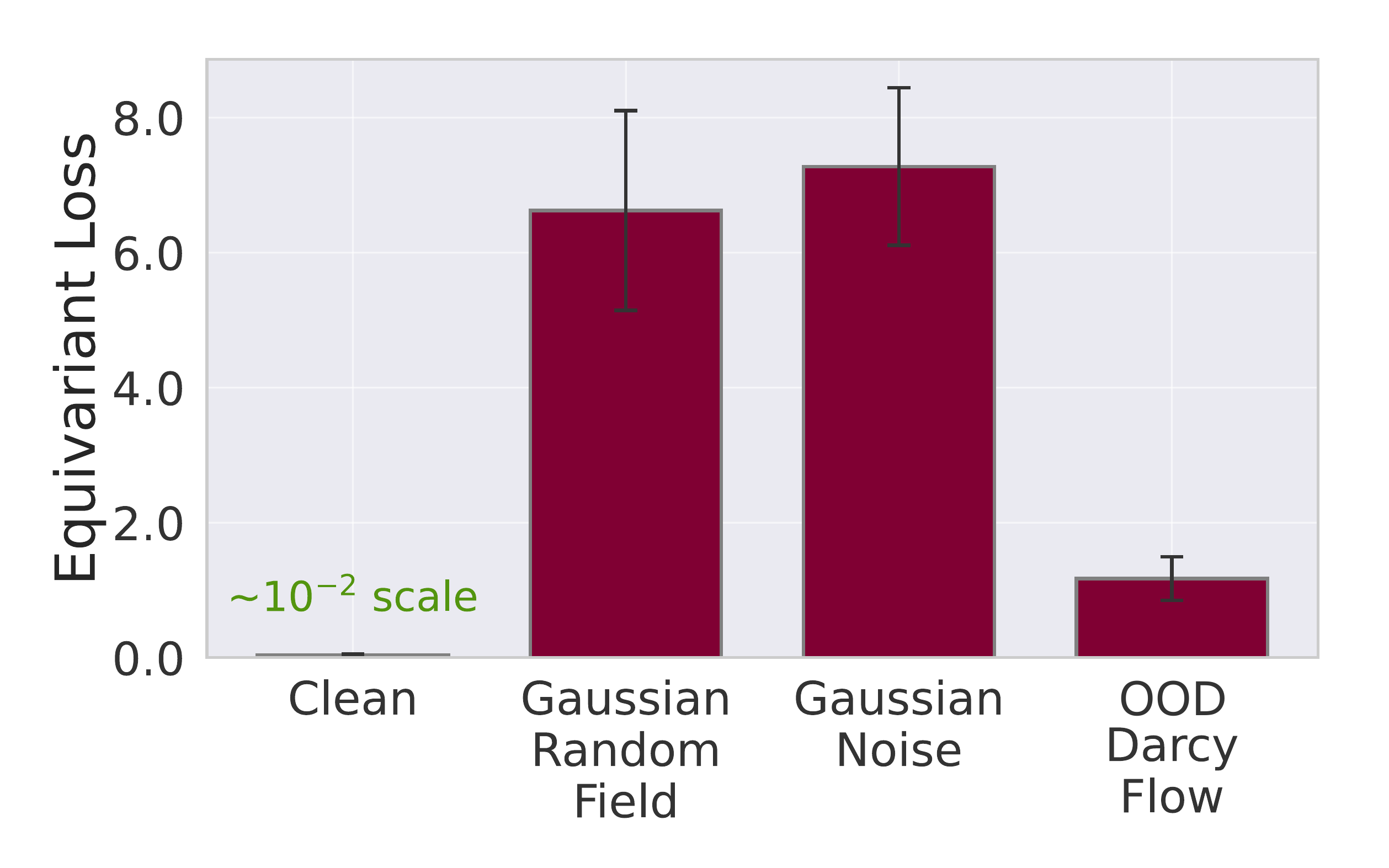}
     \vspace{-5mm}
     \caption{Equivariance arising from data symmetries. Error computed with $f_{\text{MPE}} = \text{FNO}$ and evaluated on Navier-Stokes data.}
    \label{fig:mpe_mpecon_f_data}
    \end{subfigure}
    \vspace{-2mm}
    \caption{\textbf{Equivariance error is consistently lower for clean vs perturbed examples.} In both subfigures the $y$-axis is the per-sample equivariance loss, averaged over images and over the chosen group actions. Perturbations are common image corruptions in (a) (Gaussian noise, salt-and-pepper, speckle, vertical occlusion) and physics inputs in (b) including an out-of-distribution Darcy-flow example.}
    \vspace{-5mm}
\end{figure}
We refer to the algorithm instantiated with this loss as \emph{Equi-X}, where X is the underlying solver; \emph{EquiCon-X} denotes the manifold-constrained variant that uses the equivariance error defined in~\Cref{app:add_back}. The EquiReg gradient update can be applied every $P$ DDIM steps (period $P$, default $P=1$); \Cref{tab:equireg_charac}a characterizes performance as a function of $P$. The actions $T_g$ on the input and $S_g$ on the output are determined by a chosen symmetry group $G$. In practice, $G$ is selected to match the symmetries that the pre-trained $f_{\text{MPE}}$ approximately respects. For example, an autoencoder trained with horizontal-flip augmentation is approximately equivariant under that flip, and a neural operator trained on rotation-symmetric physics data is approximately equivariant under discrete rotations. The specific $f_{\text{MPE}}$ and group used in our experiments are listed in~\Cref{sec:results}, with broader guidelines on choosing $G$ in~\Cref{app:add_back}.

MPE can also emerge due to symmetries present in the data itself during training. This often occurs in physics systems where coefficient functions, boundary values, and solution functions of PDEs remain valid under invertible coordinate transformations. Formally, let $\mathcal{G}(a)\mapsto u$ be a PDE operator mapping initial condition $a$ to solution $u$, and let $T_g$ and $S_g$ be invertible transformations that preserve PDE structure and boundary conditions. Then, $S_g(\mathcal{G}(a))=\mathcal{G}(T_g(a))$. Neural operators~\citep{kovachki2021neural}, popular architectures for modelling physics, trained on PDEs with such inherent symmetries can learn equivariance properties. \Cref{fig:mpe_mpecon_f_data} shows that an MPE function can be constructed using Fourier Neural Operators (FNOs~\citep{li2021fourier}) trained on non-augmented Navier–Stokes physics data. When a pre-trained FNO is used as the MPE function, the equivariance loss in~\eqref{eq:equi_losses} is lower for in-distribution data than for out-of-distribution data under reflection as the group action.

The key message from our MPE examples is that MPE behaviour naturally emerges when a function (e.g., a neural network) is trained with appropriate augmentations or when the data exhibits inherent symmetries. Our results extend this analysis and demonstrate that MPE functions are widely present and can be integrated into the EquiReg framework to improve posterior sampling (\Cref{tab:sitcom_ffhq_mblur_mpe_widespread}). We leverage this property to distinguish on-manifold from off-manifold samples and to regularize the posterior sampling trajectory toward high-probability regions. Finally, we note that the choice of symmetry group may often be a challenge depending on application domain, a shared challenge in the broader equivariance literature. We provide guidelines on how to choose symmetry groups in~\Cref{app:add_back} with references on automatic symmetry discovery from data~\citep{zhou2021meta, quessard2020learning, dehmamy2021automatic, mohapatra2025symmetry}.

\vspace{-2mm}
\section{Results}\label{sec:results}
%
This section provides experimental results evaluating the performance of EquiReg on inverse problems, including linear and nonlinear image restorations and PDE solving. To fairly assess the impact of EquiReg, we adopt a paired comparison setting (e.g., PSLD vs. Equi-PSLD) across experiments, keeping all other factors such as architecture, training, and sampling fixed. This design ensures that any observed improvements can be attributed specifically to EquiReg rather than to differences in the underlying model or inference procedure. For fair comparison, we set a fixed seed for the initial noise and evaluate the impact of EquiReg under reduced measurement consistency and sampling steps, providing a path toward faster diffusion-based inversion. Results emphasize EquiReg's usefulness when the baseline performance deteriorates. We analyze the diversity trade-offs, and demonstrate the broad emergence of MPEs and their usefulness within EquiReg. Lastly, we provide preliminary analysis on EquiReg improving the realism of text-guided image generation.

\begin{wraptable}[17]{r}{0.45\textwidth}
\vspace{-4mm}
\centering
\caption{\textbf{EquiReg improves SITCOM under reduced measurement consistency steps ($K_{\text{meas}}$).} We reduce $K_{\text{meas}}$ and add an equal amount of EquiReg steps ($K_{\text{EquiReg}}$). Evaluated with $50$ DDIM steps on motion deblur for FFHQ.}
\vspace{-2mm}
\label{tab:motiondeblur_ffhq_reducedopt_modif}
\begin{subtable}{\linewidth}
\fontsize{8}{11}\selectfont
\setlength{\tabcolsep}{4pt}
\centering
\begin{tabular}{l*{4}{c}}
\toprule
$K_{\text{meas.}}$ & $K_{\text{EquiReg}}$ & PSNR$\uparrow$ & SSIM$\uparrow$ & Runtime (s)
\\
\midrule
10 & N/A & 28.06 & 0.81  & 21.57 \\
5 & 5   & \bf 29.26  & \bf 0.83  & \bf 11.09 \\
\midrule
20 & N/A & 27.04  & 0.79  & 38.85 \\
10 & 10  & \bf 28.93  & \bf 0.82  & \bf 20.92 \\
\midrule
30 & N/A & 27.79 & 0.80  & 58.84 \\
15 & 15   & \bf 29.63  & \bf 0.84  & \bf 30.19 \\
\midrule
40 & N/A & \bf 30.40 & \bf 0.85  & 78.08 \\
20 & 20  & 29.50 & 0.83 & \bf 41.02 \\
\midrule
60 & N/A & 28.35 & 0.81 & 108.57\\
30 & 30  & \bf 31.36 & \bf 0.87 & \bf 59.38 \\
\bottomrule
\end{tabular}
\end{subtable}
\end{wraptable}
\textbf{Image restoration tasks.}\quad We evaluate EquiReg when applied to: SITCOM~\citep{alkhouri2025sitcom}, PSLD~\citep{rout2023solving}, ReSample~\citep{song2023solving}, and DPS~\citep{chung2023diffusion}. We compare against manifold-preserving or geometry-constraint approaches including MCG~\citep{chung2022improving}, MPGD-AE~\citep{he2024manifold}, and DiffStateGrad~\citep{zirvi2025diffusion}. We measure performance via perceptual similarity (LPIPS), distribution alignment (FID), pixel-wise fidelity (PSNR), and structural consistency (SSIM). We test EquiReg on a) the FFHQ $256 \times 256$~\citep{karras2021stylebased} and b) ImageNet $256 \times 256$ validation set~\citep{deng2009imagenet}. For pixel-based experiments, we use i) the pre-trained model from~\citep{chung2023diffusion} on FFHQ, and ii) the pre-trained model from \citep{dhariwal2021diffusion} on ImageNet. For latent diffusion experiments, we use i) the unconditional LDM-VQ-4 model~\citep{rombach2022high} on FFHQ, and ii) the Stable Diffusion v1.5~\citep{rombach2022high} model on ImageNet.

\begin{wraptable}[27]{r}{0.49\textwidth}
\vspace{-4mm}
\caption{\textbf{EquiReg for ReSample on linear and nonlinear tasks.} FFHQ 256\,$\times$\,256 with $\sigma_{\y}=0.01$.}
\vspace{-2mm}
\centering
  \label{tab:ffhq_resample}
  \fontsize{7.5}{9}\selectfont
  \setlength{\tabcolsep}{2pt}
  \begin{tabular}{c l c c c c}
    \toprule
    \textbf{Task} & \textbf{Method} & \textbf{LPIPS$\downarrow$} & \textbf{FID$\downarrow$} & \textbf{PSNR$\uparrow$} & \textbf{SSIM$\uparrow$} \\
    \midrule
    \multicolumn{6}{l}{\textit{Linear}} \\
    \multirow{3}{*}{\makecell{Gaussian\\deblur}}
      & ReSample                  & 0.253 & 55.65 & 27.78 & 0.757 \\
      & Equi-ReSample    & 0.197 & 64.86 & \textbf{29.08} & \textbf{0.825} \\
      & EquiCon-ReSample  & \textbf{0.156} & \textbf{54.72} & 28.18 & 0.777 \\
    \midrule
    \multirow{3}{*}{\makecell{Motion\\deblur}}
      & ReSample                  & 0.160 & 40.14 & 30.55 & 0.854 \\
      & Equi-ReSample      & 0.120 & 46.28 & \textbf{30.92} & \textbf{0.870} \\
      & EquiCon-ReSample    & \textbf{0.078} & \textbf{37.61} & 30.73 & 0.860 \\
    \midrule
    \multirow{3}{*}{\makecell{Super-res.\\($\times4$)}}
      & ReSample                  & 0.204 & 40.46 & 28.02 & 0.790 \\
      & Equi-ReSample      & \textbf{0.098} & 43.56 & \textbf{29.74} & \textbf{0.849} \\
      & EquiCon-ReSample  & 0.112 & \textbf{40.38} & 28.27 & 0.801 \\
    \midrule
    \multirow{3}{*}{\makecell{Box\\ inpainting}}
      & ReSample                  & 0.198 & 108.30 & 19.91 & 0.807 \\
      & Equi-ReSample    & \textbf{0.150} & \textbf{59.69} & \textbf{22.56} & \textbf{0.832} \\
      & EquiCon-ReSample  & 0.171 & 110.70 & 21.04 & 0.815 \\
    \midrule
    \multirow{3}{*}{\makecell{Random\\inpainting}}
      & ReSample                  & 0.115 & 36.12 & 31.27 & 0.892 \\
      & Equi-ReSample   & \textbf{0.047} & 29.88 & \textbf{31.47} & \textbf{0.908} \\
      & EquiCon-ReSample   & \textbf{0.047} & \textbf{28.81} & 31.21 & 0.904 \\
    \midrule
    \multicolumn{6}{l}{\textit{Nonlinear}} \\
    \multirow{3}{*}{\makecell{HDR}}
      & ReSample                  & 0.190 & \textbf{49.06} & \textbf{24.88} & \textbf{0.819} \\
      & Equi-ReSample     & \textbf{0.133} & 49.52 & 24.71 & 0.815 \\
      & EquiCon-ReSample   & 0.135 & 49.98 & 24.67 & 0.817 \\
    \midrule
    \multirow{3}{*}{\makecell{Phase\\retrieval}}
      & ReSample                  & 0.237 & 97.86    & 27.61 & 0.750 \\
      & Equi-ReSample     & \textbf{0.155} & \textbf{85.22} & \textbf{28.16} & 0.770 \\
      & EquiCon-ReSample  & 0.159 & 88.75 & 28.11 & \textbf{0.774} \\
    \midrule
    \multirow{3}{*}{\makecell{Nonlinear\\deblur}}
      & ReSample                  & 0.188 & 56.06 & 29.54 & 0.842 \\
      & Equi-ReSample    & 0.128 & 55.09 & 29.45 & 0.840 \\
      & EquiCon-ReSample   & \textbf{0.125} & \textbf{54.62} & \textbf{29.55} & \textbf{0.843} \\
    \bottomrule
  \end{tabular}
\end{wraptable}
We evaluate EquiReg on linear and nonlinear restoration tasks for natural images (see~\Cref{app:equireg_imp} for task details). We adopt the publicly-released LDM autoencoder of~\citet{rombach2022high} (LDM-VQ-4 for FFHQ, Stable Diffusion v1.5 for ImageNet) as our MPE function. For latent-space solvers (Equi-PSLD, Equi-ReSample), $f_{\text{MPE}}$ is the decoder $\mathcal{D}$ applied to $\z_{0|t}$, so the equivariance loss in~\eqref{eq:equi_losses} becomes $\|S_g(\mathcal{D}(\z_{0|t})) - \mathcal{D}(T_g(\z_{0|t}))\|_2^2$; this $\mathcal{D}$ is the same one those solvers already use internally for sampling.

For pixel-space solvers (Equi-DPS, Equi-SITCOM), $f_{\text{MPE}}$ is the encoder $\mathcal{E}$ from the same autoencoder applied to $\x_{0|t}$, giving $\|S_g(\mathcal{E}(\x_{0|t})) - \mathcal{E}(T_g(\x_{0|t}))\|_2^2$. The EquiCon variants instead use the manifold-constrained error from~\Cref{app:add_back}, which involves both $\mathcal{E}$ and $\mathcal{D}$. No additional pre-training is performed. For FFHQ, we use vertical reflection as the symmetry group, which preserves upright facial orientation. For ImageNet, we define a rotation group $G = \{ 0, \nicefrac{\pi}{2}, \pi, \nicefrac{3\pi}{2}\}$, and select the group action uniformly at random for each sample. Finally, the loss functions given in~\Cref{eq:equi_losses} are used to regularize. While our main experiments explore the reflection and rotation groups with small cardinality, EquiReg does not rely on full group coverage. Sampling even a sparse or randomly chosen subset of group actions is sufficient, as long as the function used for regularization exhibits the MPE property across the group (\Cref{tab:results_subsetgroup}).

We characterize EquiReg under change of its period and $\lambda_t$. EquiReg can be applied efficiently with longer period; it preserves performance when applied with lower frequency (\Cref{tab:equireg_charac}a); and it is robust to the choice of its hyperparameter $\lambda_t$, demonstrated on PSLD (\Cref{tab:equireg_charac}b; see also~\Cref{fig:lambdaeta} for performance as a function of $\lambda$ and $\zeta$). EquiReg also allows the user to reduce the number of measurement consistency optimization steps, thereby introducing an implicit acceleration in solving the inverse problem. \Cref{tab:motiondeblur_ffhq_reducedopt_modif} shows that EquiReg regularization consistently enables SITCOM to achieve superior performance with significantly reduced runtime by using fewer measurement consistency steps.

%
\begin{table}[t]
\vspace{-2mm}
\centering
\caption{\textbf{EquiReg for diffusion models on FFHQ}. $256  \times 256$ with $\sigma_{\y} = 0.05$.}
\vspace{-1mm}
\label{tab:ffhq_latentandpixel}
\begin{subtable}{\textwidth}
\fontsize{7.5}{10}\selectfont
\setlength{\tabcolsep}{0.5pt}
\centering
\begin{tabular}{l*{15}{c}}
\toprule
Method & \multicolumn{3}{c}{Gaussian deblur} & \multicolumn{3}{c}{Motion deblur} & \multicolumn{3}{c}{Super-resolution ($\times4$)} & \multicolumn{3}{c}{Box inpainting} & \multicolumn{3}{c}{Random inpainting} \\
\cmidrule(lr){2-4} \cmidrule(lr){5-7} \cmidrule(lr){8-10} \cmidrule(lr){11-13} \cmidrule(lr){14-16}
& LPIPS$\downarrow$ & FID$\downarrow$ & PSNR$\uparrow$ & LPIPS$\downarrow$ & FID$\downarrow$ & PSNR$\uparrow$ & LPIPS$\downarrow$ & FID$\downarrow$ & PSNR$\uparrow$ & LPIPS$\downarrow$ & FID$\downarrow$ & PSNR$\uparrow$ & LPIPS$\downarrow$ & FID$\downarrow$ & PSNR$\uparrow$ \\
\midrule
PSLD & 0.357 & 106.2 & 22.87
& \textbf{0.322} & \textbf{84.62} & 24.25
& 0.313 & \underline{89.72} & 24.51
& 0.158 & 43.02 & \underline{24.22}
& 0.246 & 49.77 & 29.05
\\
Equi-PSLD
  & \underline{0.344}  & \underline{94.09}   & \textbf{24.42}
  & \underline{0.338}  & {99.14}  &\underline{24.83} 
  & \underline{0.289}  & 90.88   & \textbf{26.32} 
  & \underline{0.098}  & \textbf{31.54}  & 24.19 
  & \textbf{0.188}  & \underline{41.61}  & \textbf{30.43} 
  \\
EquiCon-PSLD
  & \textbf{0.320}  & \textbf{83.18}   & \underline{24.38} 
  & \textbf{0.322}  & \underline{89.87}   & \textbf{25.14}
  & \textbf{0.277}  & \textbf{79.39}   & \underline{26.14}
  & \textbf{0.092}  & \underline{35.07}  & \textbf{24.26}
  & \underline{0.204}  & \textbf{40.75}  & \underline{29.99}
  \\
\bottomrule
\end{tabular}
\caption{Latent diffusion.}
\end{subtable}
\begin{subtable}{\textwidth}
\vspace{-1mm}
\fontsize{7.5}{10}\selectfont
\setlength{\tabcolsep}{0.5pt}
\centering
\begin{tabular}{l*{15}{c}}
\toprule
Method & \multicolumn{3}{c}{Gaussian deblur} & \multicolumn{3}{c}{Motion deblur} & \multicolumn{3}{c}{Super-resolution ($\times4$)} & \multicolumn{3}{c}{Box inpainting} & \multicolumn{3}{c}{Random inpainting} \\
\cmidrule(lr){2-4} \cmidrule(lr){5-7} \cmidrule(lr){8-10} \cmidrule(lr){11-13} \cmidrule(lr){14-16}
& LPIPS$\downarrow$ & FID$\downarrow$ & PSNR$\uparrow$ & LPIPS$\downarrow$ & FID$\downarrow$ & PSNR$\uparrow$ & LPIPS$\downarrow$ & FID$\downarrow$ & PSNR$\uparrow$ & LPIPS$\downarrow$ & FID$\downarrow$ & PSNR$\uparrow$ & LPIPS$\downarrow$ & FID$\downarrow$ & PSNR$\uparrow$ \\
\midrule
  DPS & 0.119 &  61.76 &  26.05
  &  0.100 &  54.71 &  27.77
  &  0.126 &  65.15 &  26.72
  &  0.108 &  54.14 &  22.84
  &  0.073 &  49.60 & 31.49
  \\
Equi-DPS
& \textbf{0.114} & \textbf{48.76} & \textbf{26.32}
& \textbf{0.094} & \textbf{41.71} & \textbf{28.23}
& \textbf{0.120} & \textbf{51.00} & \textbf{27.15}
& \textbf{0.099} & \underline{40.47} & \underline{23.39}
& \textbf{0.068} & \underline{33.65} & \textbf{32.16}
\\
\cmidrule(lr){2-16}
DiffStateGrad-DPS &  \underline{0.128} & \underline{52.73} & \underline{26.29} 
& \underline{0.118} & \underline{50.14} & \underline{27.61} 
& 0.186 & 73.02 & 24.65
& \underline{0.114} & 47.53 & \textbf{24.10} 
& \underline{0.107} & 49.42 &  \underline{30.15}
\\
MCG &
0.340 & 101.2 & 6.72
& 0.702 & 310.5 & 6.72
& 0.520 & 87.64 & 20.05 
& 0.309 & \textbf{40.11} & 19.97
& 0.286 & \textbf{29.26} & 21.57
\\
MPGD-AE & 0.150 & 114.9 & 24.42
& 0.120 & 104.5  & 25.72 
& \underline{0.168} & 137.7 & 24.01
& 0.138 & 248.7 & 21.59 
& 0.172 & 339.0  & 25.22
\\
\bottomrule
\end{tabular}
\caption{Pixel-based diffusion.}
\end{subtable}
\vspace{-5mm}
\end{table}
%
\begin{table}[t]
\vspace{-1mm}
\centering
 \caption{\textbf{EquiReg for latent diffusion models on ImageNet}. $256  \times 256$ with $\sigma_{\y} = 0.05$.}
\vspace{-1mm}
\label{tab:imagenet_ldm}
\begin{subtable}{1.00\textwidth}
\fontsize{8}{10}\selectfont
\setlength{\tabcolsep}{1.5pt}
 \centering
\begin{tabular}{l*{10}{c}}
\toprule
Method & \multicolumn{2}{c}{Gaussian deblur} & \multicolumn{2}{c}{Motion deblur} & \multicolumn{2}{c}{Super-resolution (x4)} & \multicolumn{2}{c}{Box inpainting} & \multicolumn{2}{c}{Random inpainting} \\
\cmidrule(lr){2-3} \cmidrule(lr){4-5} \cmidrule(lr){6-7} \cmidrule(lr){8-9} \cmidrule(lr){10-11}
& \makebox[0.8cm][c]{FID$\downarrow$} & \makebox[0.8cm][c]{PSNR$\uparrow$} & 
\makebox[0.8cm][c]{FID$\downarrow$} & \makebox[0.8cm][c]{PSNR$\uparrow$} & 
\makebox[0.8cm][c]{FID$\downarrow$} & \makebox[0.8cm][c]{PSNR$\uparrow$} & 
\makebox[0.8cm][c]{FID$\downarrow$} & \makebox[0.8cm][c]{PSNR$\uparrow$} & 
\makebox[0.8cm][c]{FID$\downarrow$} & \makebox[0.8cm][c]{PSNR$\uparrow$} \\
\midrule
PSLD & 263.9 & 20.70 & 252.1 & 21.26 & 224.3 & 22.29 & 151.4 & 16.28 & 83.22 & 26.56 \\
EquiCon-PSLD & \textbf{214.5} & \textbf{22.01} & \textbf{196.3} & \textbf{22.69} & \textbf{198.5} & \textbf{22.34} & \textbf{137.6} & \textbf{19.25} & \textbf{65.14} & \textbf{27.03} \\
\bottomrule
\end{tabular}
\end{subtable}
\vspace{-6mm}
\end{table}
\Cref{tab:ffhq_latentandpixel}a, \Cref{tab:imagenet_ldm}, and \Cref{tab:ffhq_resample} highlight the benefits of EquiReg for latent models. EquiReg consistently improves the performance of ReSample and PSLD across several tasks on FFHQ and ImageNet. We attribute this improvement in part to the reduction of failure cases (\Cref{fig:decoder_equi_loss_vs_gaussian_noise}c). EquiReg also significantly improves the performance of pixel-based methods (see Equi-DPS vs. DPS, \Cref{tab:ffhq_latentandpixel}b).

We observe that EquiReg achieves its largest improvements on perceptual metrics (FID and LPIPS), suggesting it generates more realistic images that lie closer to the data manifold (see Appendix E for supporting qualitative results). EquiReg improves performance under high measurement noise (\Cref{fig:decoder_equi_loss_vs_gaussian_noise}b). This result aligns with~\Cref{fig:decoder_equi_loss_vs_gaussian_noise}a, which shows the equivariance error is lower on clean images than noisy ones, indicating that EquiReg enforces denoising. Lastly, we note that EquiReg is robust to regularizing hyperparameter $\lambda_t$ (\Cref{fig:decoder_equi_loss_vs_gaussian_noise}c, see~\Cref{app:add_vis} for details). For qualitative performance of EquiReg, see~\Cref{fig:qual_comparison_tasks_psld}.

%
\begin{wraptable}[11]{r}{0.50\textwidth}
\vspace{-4mm}
 \centering
 \caption{\textbf{Solving PDEs from sparse observations.} Relative $\ell_2$ error (\%, lower is better) on the Helmholtz and Navier-Stokes benchmarks, averaged over 100 test samples; bold marks the best result in each column.}
 \label{tab:pde_results}
 \vspace{-2mm}
\fontsize{8}{10}\selectfont
\setlength{\tabcolsep}{1.9pt}
 \begin{tabular}{lc cc cc} 
 \toprule
 & \multirow{2}{*}{Steps $(N)$} & \multicolumn{2}{c}{Helmholtz}
 & \multicolumn{2}{c}{Navier-Stokes} \\
 \cmidrule(lr){3-4}
 \cmidrule(lr){5-6}
  & & Forward & Inverse & Forward & Inverse \\
 \midrule
 DiffusionPDE & 2000
 & 12.64\% & 19.07\%
 & 3.78\% & 9.63\% \\
 FunDPS & 500
 & 2.13\% & 17.16\%
 & 3.32\% & 8.48\% \\
 Equi-FunDPS & 500
 & \textbf{2.12\%} & \textbf{15.91\%} & \textbf{3.06\%} & \textbf{7.84\%} \\
 \bottomrule
 \end{tabular}
\end{wraptable}
%
\textbf{Solving PDEs from sparse observations.}\quad EquiReg is evaluated on two important problems: the Helmholtz and Navier-Stokes equations (see~\Cref{app:pde}). The objective is to solve both forward and inverse problems in sparse sensor settings. The forward problem involves predicting the solution function or the final state using measurements from $3\%$ of the coefficient field or the initial state.
The inverse problem, conversely, aims to predict the input conditions from observations of $3\%$ of the system's output.
This task is challenging due to the nonlinearity of the equations, the complex structure of Gaussian random fields, and the sparsity of observations. Recent studies~\citep{huang2024diffusionpde,mammadovdiffusion,yao2025guideddiffusionsamplingfunction} have demonstrated the superiority of diffusion models over deterministic single-forward methods for solving PDEs. DiffusionPDE~\citep{huang2024diffusionpde} decomposes the conditional log-likelihood into a learned diffusion prior and a measurement score. FunDPS~\citep{yao2025guideddiffusionsamplingfunction} extends the sampling process to a more natural infinite-dimensional spaces, achieving better accuracy and speed via function space models. 

\begin{figure}[t]
\centering
\vspace{-1mm}
    \includegraphics[width=1.0\linewidth]{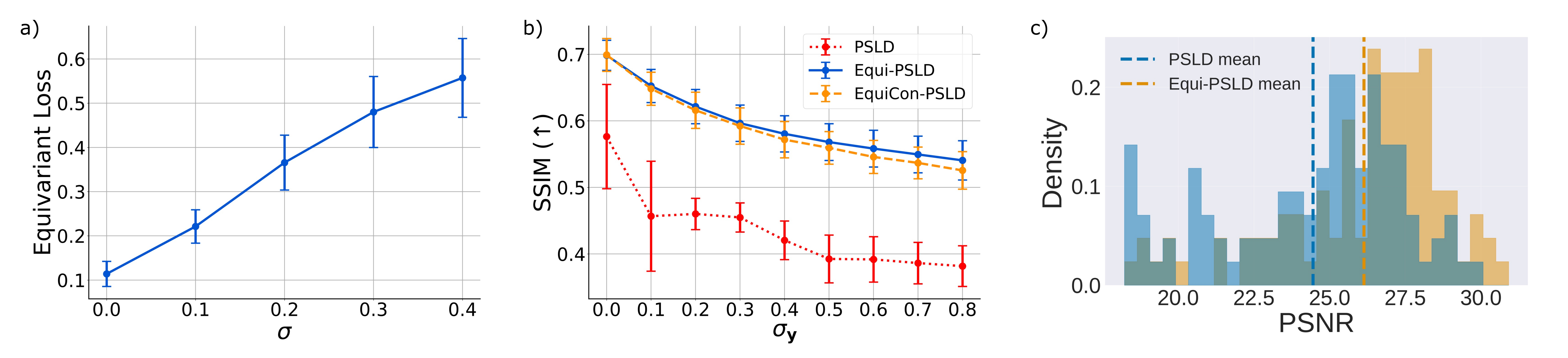}
\vspace{-5mm}
\caption{\textbf{EquiReg effectiveness and robustness across a range of measurement noise levels and regularization parameter.} (a) Equivariance error of the pre-trained LDM decoder $\mathcal{D}$ evaluated on FFHQ images with Gaussian noise added at varying standard deviations $\sigma$ (averaged over 100 images and the chosen symmetry group). (b) Reconstruction quality (SSIM) as a function of measurement noise level $\sigma_{\y}$, on super-resolution ($4\times$) with PSLD vs Equi-PSLD vs EquiCon-PSLD on FFHQ. (c) Overlaid distributions of per-image PSNR for DPS vs Equi-DPS on super-resolution on FFHQ; dashed lines mark the means. The shift in the Equi-DPS distribution toward higher PSNR illustrates the reduction of low-quality failure cases.}
\label{fig:decoder_equi_loss_vs_gaussian_noise}
\vspace{-3mm}
\end{figure}
We integrate EquiReg into the state-of-the-art FunDPS framework~\citep{mammadovdiffusion,yao2025guideddiffusionsamplingfunction}, where the equivariance loss is computed using an FNO trained on the corresponding inverse problem. We employ reflection symmetry (i.e., flipping along the $y = x$ axis) and observe no significant performance differences when using alternative transformations such as rotations or alternating flips. Equi-FunDPS improves performance (\Cref{tab:pde_results}), measured by relative $\ell_2$ error, across multiple tasks, particularly in inverse problems where a strong data prior is essential. Importantly, these results highlight that EquiReg is not limited to conventional neural networks and is naturally applicable to neural operators, enabling regularization in both finite-dimensional and function-space diffusion models.

\textbf{Diversity analysis.}\quad To study posterior sampling diversity of EquiReg, we generated $K=10$ posterior samples for $20$ test images across three inverse problems of box inpainting, Gaussian deblurring, $4\times$ super-resolution, and measured diversity using two complementary metrics: Intra-LPIPS for perceptual diversity and Pixel-Std for spatial diversity. \Cref{tab:diversity_analysis} demonstrates that Equi-DPS does not show posterior collapse. We further investigated diversity scaling by varying box inpainting mask size from $128\times128$ to $192\times192$ pixels (\Cref{fig:diversity_difficulty}). Results show that diversity metrics increase linearly with task difficulty, demonstrating that Equi-DPS expands sampling as problems become more ill-posed rather than artificially constraining solutions. This linear relationship indicates consistent, non-collapsed posterior sampling across the difficulty spectrum. Lastly, \Cref{fig:diversity_qualitative,fig:diversity_qualitative_combined} provide qualitative confirmation through visual examples showing four posterior samples per image. Observable variations in facial features, expressions, and eye gaze validate our quantitative measurements, confirming EquiReg can generate genuinely diverse reconstructions rather than collapsing to a single solution (for additional discussion and analysis, see~\Cref{app:diverse}).

\textbf{MPE behavior emerges across neural networks and improves EquiReg's performance.}\quad
We empirically evaluate whether MPE behaviour naturally emerges across commonly used neural networks. We examine i) the emergence of MPE properties in different functions (neural networks) and ii) the effect of using these functions within EquiReg on identical inverse problem settings. For FFHQ and ImageNet, we analyze four function classes: the LDM encoder~\citep{rombach2022high}, a CNN autoencoder trained with symmetry augmentations (flip for FFHQ and rotation for ImageNet), pre-trained ResNet-50~\citep{he2016deep}, and CLIP~\citep{radford2021learning}. Across all architectures, equivariance error increases systematically as Gaussian noise pushes samples off the data manifold, confirming clear MPE behaviour. The strength of this effect varies: the CNN autoencoder exhibits the strongest MPE signal, while the LDM encoder shows the weakest; ResNet-50 and CLIP lie between these extremes (see \Cref{fig:mpe_mpecon,fig:mpe_loss}). Notably, the strongest MPE behaviour of the CNN autoencoder is in line with our systematic guidelines for constructing MPE functions; this is precisely the regime where the train distribution matches the test.

\begin{table}[t]
\centering
\caption{\textbf{Fidelity and diversity across inverse problems.} EquiReg preserves sampling diversity ($20$ test images with $K=10$ samples per image.)}
\label{tab:diversity_analysis}
\fontsize{8}{11}\selectfont
\setlength{\tabcolsep}{6pt}
\begin{tabular}{ll*{4}{c}}
\toprule
Task & Method & \multicolumn{2}{c}{Fidelity Metrics} & \multicolumn{2}{c}{Diversity Metrics} \\
\cmidrule(lr){3-4} \cmidrule(lr){5-6}
& & \makebox[0.8cm][c]{LPIPS$\downarrow$} & \makebox[0.8cm][c]{FID$\downarrow$} & \makebox[1.3cm][c]{Intra-LPIPS$\uparrow$} & \makebox[1.3cm][c]{Pixel-Std$\uparrow$} \\
\midrule
\multirow{1}{*}{Box inpainting} 
& Equi-DPS & \textbf{0.112} & \textbf{59.70} & \textbf{0.118} & \textbf{10.59} \\
\midrule
\multirow{1}{*}{Gaussian deblur} 
& Equi-DPS & \textbf{0.120} & \textbf{63.02} & 0.092 & 5.669 \\
\midrule
\multirow{1}{*}{Super-resolution ($\times 4$)} 
& Equi-DPS & \textbf{0.703} & \textbf{87.52} & \textbf{0.187} & \textbf{23.52} \\
\bottomrule
\end{tabular}
\vspace{-3mm}
\end{table}
\begin{figure}[t]
\centering
\includegraphics[width=0.9\textwidth]{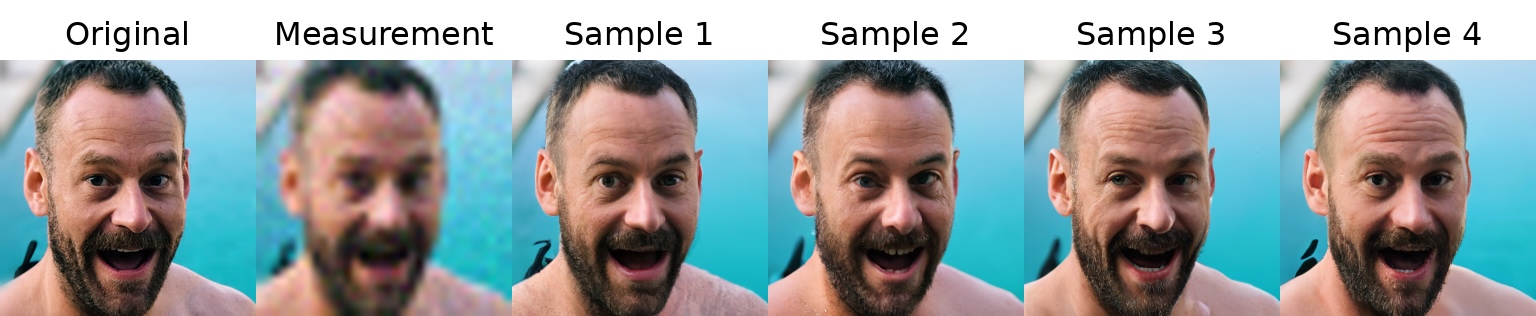}
\vspace{2mm}
\includegraphics[width=0.9\textwidth]{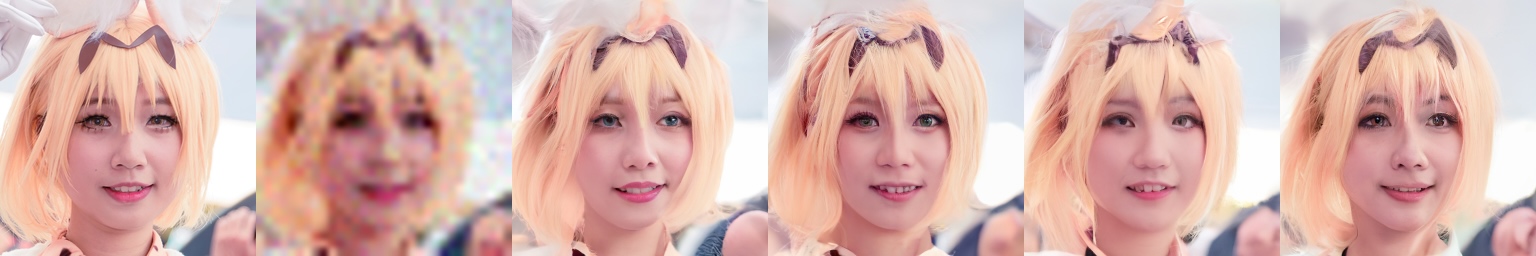}
\vspace{-4mm}
\caption{\textbf{Qualitative diversity examples for super-resolution.} We show $K=4$ posterior samples for two test images. Samples differ in facial features (i.e., teeth in the first test image, eye color and eyelashes in the second test image) while maintaining high fidelity to the ground truth, demonstrating EquiReg generates diverse plausible reconstructions rather than collapsing to a single mode.}
\label{fig:diversity_qualitative}
\vspace{-6mm}
\end{figure}
We then integrate each function into EquiReg under identical inverse problem settings (FFHQ/ImageNet, DPS/SITCOM, super-resolution/motion deblurring). In all cases, EquiReg improves reconstruction quality over the corresponding baseline without regularization (\Cref{tab:sitcom_ffhq_mblur_mpe_widespread,tab:dps_sr_mpe}). These results demonstrate that i) MPE behaviour naturally arises in widely used pre-trained networks, and ii) EquiReg is robust to the choice of MPE function, even when the MPE signal is relatively weak. Notably, our main experiments use the weakest MPE (the LDM encoder), suggesting further gains are possible with stronger MPE constructions.

\textbf{Architectural and training factors behind MPE effectiveness.} The ranking observed in~\Cref{tab:sitcom_ffhq_mblur_mpe_widespread,tab:dps_sr_mpe} and~\Cref{fig:mpe_loss} reflects several architectural and training factors. The dominant factor is the match between the MPE function's training distribution and the inverse-problem distribution. The CNN autoencoder, trained directly on FFHQ with flip augmentation, produces the largest on-/off-manifold separation when evaluated on FFHQ-derived inverse problems, while the LDM encoder, trained on a comparable distribution but with milder augmentation and an additional latent-compression objective, exhibits a noticeably smaller separation. A secondary contribution comes from architectural inductive bias (\Cref{fig:mpe_arch}). The convolutional and ResNet autoencoders both show the MPE property; however, this property may not emerge in all architectures. Finally, it is important to distinguish equivariant functions from MPE functions. If a network is explicitly designed to be equivariant to a particular group transformation, its equivariance behavior is independent of the input distribution. In such cases, the equivariance error will not distinguish between on- and off-manifold samples under that group, making it not suitable as an MPE function for that transformation (\Cref{fig:cnn_equi}).

\begin{table}[t]
    \centering
    \caption{\textbf{MPE behaviour emerges across diverse network architectures and consistently improves diffusion-based inverse problems with EquiReg.} SITCOM motion deblurring on FFHQ 256 with $\lambda = 0.05$. Results are reported as mean (standard deviation).}
    \label{tab:sitcom_ffhq_mblur_mpe_widespread}
    \fontsize{8}{11}\selectfont
    \setlength{\tabcolsep}{6pt}
    \begin{tabular}{lccc}
        \toprule
        MPE function & PSNR & SSIM & LPIPS \\
        \midrule
        None                   & 27.670 (1.343) & 0.790 (0.031) & 0.221 (0.040) \\
        LDM Encoder (FFHQ)     & 28.357 (1.379) & 0.806 (0.031) & 0.200 (0.036) \\
        CNN Autoencoder (FFHQ) & \textbf{28.852 (1.376)} & \textbf{0.819 (0.044)} & \textbf{0.193 (0.033)} \\
        Pretrained ResNet50    & 28.682 (1.388) & 0.811 (0.036) & 0.198 (0.036) \\
        \bottomrule
    \end{tabular}
    \vspace{-3mm}
\end{table}

\vspace{-1mm}
\section{Conclusion}\label{sec:conclusion}
\vspace{-3mm}
We introduce \textit{Equivariance Regularized} (EquiReg) diffusion for inverse problems. EquiReg regularizes sampling trajectories to stay closer to the data manifold, leveraging manifold-preferential equivariance (MPE): functions with low equivariance error on-manifold and high error off-manifold. Such functions can emerge in trained networks and can serve as plug-in regularizers without modifying the diffusion denoiser. EquiReg is agnostic across pixel- and latent-space diffusion models and maintains performance under reduced sampling, accelerating convergence. Across diverse inverse problems, it consistently improves perceptual and reconstruction metrics while reducing failure cases, highlighting its generality and efficiency.

\textbf{Beyond inverse problems.}\quad EquiReg also applies to text-to-image guidance using DreamSampler~\citep{kim2024dreamsampler}, where it improves perceptual quality and reduces artifacts on FFHQ (see~\Cref{app:text_to_image}).

EquiReg operates as a plug-in regularization framework and therefore builds upon the quality of the underlying diffusion solver. It does not modify the diffusion architecture itself, but instead improves sampling by guiding trajectories toward more consistent, on-manifold solutions. As a regularization mechanism, its impact is most pronounced in challenging regimes where likelihood guidance alone may degrade or become unstable. Applying EquiReg requires selecting appropriate symmetry groups and constructing suitable MPE functions for the task at hand. While we provide systematic guidelines for imaging and PDE settings, extending these constructions to new domains remains an important direction for future work; we have provided reference to prior work on how symmetries for various applications can be learned and set. Finally, a deeper theoretical understanding of when MPE behavior emerges in trained networks and how it may interact with joint training of diffusion models remains an interesting avenue for further study.

\textbf{Broader Impact Statement.}\quad  On the positive side, high‑fidelity image restoration can improve downstream tasks in medical imaging, remote‑sensing and environmental monitoring (e.g., denoising satellite observations to track pollution or deforestation). Likewise, accelerated PDE‑solving via learned diffusion priors may enable faster, more accurate simulations for climate modeling, fluid‑dynamics research, and engineering design. On the other hand, robust reconstruction methods could be misappropriated for privacy‑invasive surveillance or to create deceptive imagery. We emphasize that our method does not amplify these existing risks.

\subsubsection*{Acknowledgements} B.T. acknowledges the partial support of Natural Sciences and Engineering Research Council of Canada (NSERC), RGPIN-2026-05959, and funding from the Canada CIFAR AI Chairs Program. B.T. was also supported by the Swartz Foundation Fellowship for Postdoctoral Research in Theoretical Neuroscience while at Caltech. A.C.'s contributions to this work were completed as part of her B.S. thesis at Caltech. A.A. was supported by Bren Chair Professor and AI2050 Senior Fellow.

\bibliography{references}

@inproceedings{chung2023diffusion,
title={{Diffusion Posterior Sampling for General Noisy Inverse Problems}},
author={Hyungjin Chung and Jeongsol Kim and Michael Thompson Mccann and Marc Louis Klasky and Jong Chul Ye},
booktitle={The Eleventh International Conference on Learning Representations},
year={2023},
url={https://openreview.net/forum?id=OnD9zGAGT0k}
}

@inproceedings{song2023pseudoinverseguided,
title={Pseudoinverse-Guided Diffusion Models for Inverse Problems},
author={Jiaming Song and Arash Vahdat and Morteza Mardani and Jan Kautz},
booktitle={International Conference on Learning Representations},
year={2023},
url={https://openreview.net/forum?id=9_gsMA8MRKQ}
}

@inproceedings{zirvi2025diffusion,
title={Diffusion State-Guided Projected Gradient for Inverse Problems},
author={Rayhan Zirvi and Bahareh Tolooshams and Anima Anandkumar},
booktitle={The Thirteenth International Conference on Learning Representations},
year={2025},
url={https://openreview.net/forum?id=kRBQwlkFSP}
}

@article{kawar2022denoising,
  title={Denoising diffusion restoration models},
  author={Kawar, Bahjat and Elad, Michael and Ermon, Stefano and Song, Jiaming},
  journal={Advances in Neural Information Processing Systems},
  volume={35},
  pages={23593--23606},
  year={2022}
}

@article{kadkhodaie2021stochastic,
  title={Stochastic solutions for linear inverse problems using the prior implicit in a denoiser},
  author={Kadkhodaie, Zahra and Simoncelli, Eero},
  journal={Advances in Neural Information Processing Systems},
  volume={34},
  pages={13242--13254},
  year={2021}
}

@inproceedings{kim2024dreamsampler,
  title={Dreamsampler: Unifying diffusion sampling and score distillation for image manipulation},
  author={Kim, Jeongsol and Park, Geon Yeong and Ye, Jong Chul},
  booktitle={European Conference on Computer Vision},
  pages={398--414},
  year={2024},
  organization={Springer}
}

@misc{mammadov2024amortizedposteriorsamplingdiffusion,
      title={Amortized Posterior Sampling with Diffusion Prior Distillation}, 
      author={Abbas Mammadov and Hyungjin Chung and Jong Chul Ye},
      year={2024},
      eprint={2407.17907},
      archivePrefix={arXiv},
      primaryClass={cs.CV},
      url={https://arxiv.org/abs/2407.17907}, 
}

@inproceedings{peebles2023scalable,
  title={Scalable diffusion models with transformers},
  author={Peebles, William and Xie, Saining},
  booktitle={Proceedings of the IEEE/CVF international conference on computer vision},
  pages={4195--4205},
  year={2023}
}

@article{brooks2024video,
  title={Video generation models as world simulators},
  author={Brooks, Tim and Peebles, Bill and Holmes, Connor and DePue, Will and Guo, Yufei and Jing, Li and Schnurr, David and Taylor, Joe and Luhman, Troy and Luhman, Eric and others},
  journal={OpenAI Blog},
  volume={1},
  pages={8},
  year={2024}
}

@inproceedings{esser2024scaling,
  title={Scaling rectified flow transformers for high-resolution image synthesis},
  author={Esser, Patrick and Kulal, Sumith and Blattmann, Andreas and Entezari, Rahim and M{\"u}ller, Jonas and Saini, Harry and Levi, Yam and Lorenz, Dominik and Sauer, Axel and Boesel, Frederic and others},
  booktitle={Forty-first international conference on machine learning},
  year={2024}
}

@inproceedings{tran2021explore,
  title={{Explore Image Deblurring via Encoded Blur Kernel Space}},
  author={Tran, Phong and Tran, Anh and Phung, Quynh and Hoai, Minh},
  booktitle={Proceedings of the IEEE/CVF Conference on Computer Vision and Pattern Recognition (CVPR)},
  year={2021},
  pages={},
  organization={IEEE}
}

@inproceedings{zhang2025daps,
  title={Improving diffusion inverse problem solving with decoupled noise annealing},
  author={Zhang, Bingliang and Chu, Wenda and Berner, Julius and Meng, Chenlin and Anandkumar, Anima and Song, Yang},
  booktitle={Proceedings of the Computer Vision and Pattern Recognition Conference},
  pages={20895--20905},
  year={2025}
}

@inproceedings{robbins1956empirical,
  title={An empirical Bayes approach to statistics},
  author={Robbins, Herbert Ellis},
  booktitle={Proc. Third Berkley Symposium on Mathematcal Statistics},
  pages={157--163},
  year={1956}
}

@article{miyasawa1961empirical,
  title={An empirical Bayes estimator of the mean of a normal population},
  author={Miyasawa, Koichi and others},
  journal={Bull. Inst. Internat. Statist},
  volume={38},
  number={181-188},
  pages={1--2},
  year={1961}
}

@article{efron2011tweedie,
  title={Tweedie’s formula and selection bias},
  author={Efron, Bradley},
  journal={Journal of the American Statistical Association},
  volume={106},
  number={496},
  pages={1602--1614},
  year={2011},
  publisher={Taylor \& Francis}
}

@inproceedings{rout2023solving,
title={{Solving Linear Inverse Problems Provably via Posterior Sampling with Latent Diffusion Models}},
author={Litu Rout and Negin Raoof and Giannis Daras and Constantine Caramanis and Alex Dimakis and Sanjay Shakkottai},
booktitle={Thirty-seventh Conference on Neural Information Processing Systems},
year={2023},
url={https://openreview.net/forum?id=XKBFdYwfRo}
}

@article{debortoli2022convergence,
title={Convergence of denoising diffusion models under the manifold hypothesis},
author={Valentin De Bortoli},
journal={Transactions on Machine Learning Research},
issn={2835-8856},
year={2022},
url={https://openreview.net/forum?id=MhK5aXo3gB},
}

@article{weinberger2006unsupervised,
  title={Unsupervised learning of image manifolds by semidefinite programming},
  author={Weinberger, Kilian Q and Saul, Lawrence K},
  journal={International journal of computer vision},
  volume={70},
  pages={77--90},
  year={2006},
  publisher={Springer}
}

@article{chung2022improving,
  title={{Improving Diffusion Models for Inverse Problems using Manifold Constraints}},
  author={Chung, Hyungjin and Sim, Byeongsu and Ryu, Dohoon and Ye, Jong Chul},
  journal={Advances in Neural Information Processing Systems},
  volume={35},
  pages={25683--25696},
  year={2022}
}

@article{candes2011robust,
  title={Robust principal component analysis?},
  author={Cand{\`e}s, Emmanuel J and Li, Xiaodong and Ma, Yi and Wright, John},
  journal={Journal of the ACM (JACM)},
  volume={58},
  number={3},
  pages={1--37},
  year={2011},
  publisher={ACM New York, NY, USA}
}

@inproceedings{bordt2023manifold,
  title={The manifold hypothesis for gradient-based explanations},
  author={Bordt, Sebastian and Upadhyay, Uddeshya and Akata, Zeynep and von Luxburg, Ulrike},
  booktitle={Proceedings of the IEEE/CVF Conference on Computer Vision and Pattern Recognition},
  pages={3697--3702},
  year={2023}
}

@article{narayanan2010sample,
  title={Sample complexity of testing the manifold hypothesis},
  author={Narayanan, Hariharan and Mitter, Sanjoy},
  journal={Advances in neural information processing systems},
  volume={23},
  year={2010}
}

@article{fefferman2016testing,
  title={Testing the manifold hypothesis},
  author={Fefferman, Charles and Mitter, Sanjoy and Narayanan, Hariharan},
  journal={Journal of the American Mathematical Society},
  volume={29},
  number={4},
  pages={983--1049},
  year={2016}
}

@inproceedings{song2023solving,
title={{Solving Inverse Problems with Latent Diffusion Models via Hard Data Consistency}},
author={Bowen Song and Soo Min Kwon and Zecheng Zhang and Xinyu Hu and Qing Qu and Liyue Shen},
booktitle={Conference on Parsimony and Learning (Recent Spotlight Track)},
year={2023},
url={https://openreview.net/forum?id=iHcarDCZLn}
}

@inproceedings{sohl2015deep,
  title={Deep unsupervised learning using nonequilibrium thermodynamics},
  author={Sohl-Dickstein, Jascha and Weiss, Eric and Maheswaranathan, Niru and Ganguli, Surya},
  booktitle={International conference on machine learning},
  pages={2256--2265},
  year={2015},
  organization={pmlr}
}

@inproceedings{song2021score,
title={{Score-based Generative Modeling through Stochastic Differential Equations}},
author={Yang Song and Jascha Sohl-Dickstein and Diederik Kingma and Abhishek Kumar and Stefano Ermon and Ben Poole},
booktitle={The International Conference on Learning Representations},
year={2021},
url={https://openreview.net/pdf/ef0eadbe07115b0853e964f17aa09d811cd490f1.pdf}
}

@article{daras2024warped,
  title={Warped diffusion: Solving video inverse problems with image diffusion models},
  author={Daras, Giannis and Nie, Weili and Kreis, Karsten and Dimakis, Alex and Mardani, Morteza and Kovachki, Nikola and Vahdat, Arash},
  journal={Advances in Neural Information Processing Systems},
  volume={37},
  pages={101116--101143},
  year={2024}
}

@article{donoho2006compressed,
  title={Compressed sensing},
  author={Donoho, David L},
  journal={IEEE Transactions on information theory},
  volume={52},
  number={4},
  pages={1289--1306},
  year={2006},
  publisher={IEEE}
}

@inproceedings{rombach2022high,
  title={{High-resolution Image Synthesis with Latent Diffusion Models}},
  author={Rombach, Robin and Blattmann, Andreas and Lorenz, Dominik and Esser, Patrick and Ommer, Bj{\"o}rn},
  booktitle={Proceedings of the IEEE/CVF conference on computer vision and pattern recognition},
  pages={10684--10695},
  year={2022}
}

@inproceedings{jalal2021robust,
title={{Robust Compressed Sensing MRI with Deep Generative Priors}},
author={Jalal, Ajil and Arvinte, Marius and Daras, Giannis and Price, Eric and Dimakis, Alexandros G and Tamir, Jon},
booktitle={Advances in Neural Information Processing Systems},
volume={34},
pages={14938--14954},
year={2021},
organization={Curran Associates, Inc.}
}

@article{dhariwal2021diffusion,
  title={Diffusion models beat gans on image synthesis},
  author={Dhariwal, Prafulla and Nichol, Alexander},
  journal={Advances in neural information processing systems},
  volume={34},
  pages={8780--8794},
  year={2021}
}

@article{zhang2025step,
  title={STeP: A General and Scalable Framework for Solving Video Inverse Problems with Spatiotemporal Diffusion Priors},
  author={Zhang, Bingliang and Wu, Zihui and Feng, Berthy T and Song, Yang and Yue, Yisong and Bouman, Katherine L},
  journal={preprint arXiv:2504.07549},
  year={2025}
}

@inproceedings{song2022solving,
title={{Solving Inverse Problems in Medical Imaging with Score-Based Generative Models}},
author={Song, Yang and Shen, Liyue and Xing, Lei and Ermon, Stefano},
booktitle={International Conference on Learning Representations},
year={2022}
}

@article{chung2022score,
title={{Score-Based Diffusion Models for Accelerated MRI}},
author={Chung, Hyungjin and Ye, Jong Chul},
journal={Medical Image Analysis},
pages={102479},
year={2022},
publisher={Elsevier}
}

@inproceedings{li2024rethinking,
  title={Rethinking diffusion model for multi-contrast mri super-resolution},
  author={Li, Guangyuan and Rao, Chen and Mo, Juncheng and Zhang, Zhanjie and Xing, Wei and Zhao, Lei},
  booktitle={Proceedings of the IEEE/CVF Conference on Computer Vision and Pattern Recognition},
  pages={11365--11374},
  year={2024}
}

@article{bian2024diffusion,
  title={Diffusion modeling with domain-conditioned prior guidance for accelerated mri and qmri reconstruction},
  author={Bian, Wanyu and Jang, Albert and Zhang, Liping and Yang, Xiaonan and Stewart, Zachary and Liu, Fang},
  journal={IEEE Transactions on Medical Imaging},
  year={2024},
  publisher={IEEE}
}

@article{shysheya2024conditional,
  title={On conditional diffusion models for PDE simulations},
  author={Shysheya, Aliaksandra and Diaconu, Cristiana and Bergamin, Federico and Perdikaris, Paris and Hern{\'a}ndez-Lobato, Jos{\'e} Miguel and Turner, Richard and Mathieu, Emile},
  journal={Advances in Neural Information Processing Systems},
  volume={37},
  pages={23246--23300},
  year={2024}
}

@inproceedings{
huang2024diffusionpde,
title={Diffusion{PDE}: Generative {PDE}-Solving under Partial Observation},
author={Jiahe Huang and Guandao Yang and Zichen Wang and Jeong Joon Park},
booktitle={The Thirty-eighth Annual Conference on Neural Information Processing Systems},
year={2024},
}

@article{kazerouni2023diffusion,
  title={Diffusion models in medical imaging: A comprehensive survey},
  author={Kazerouni, Amirhossein and Aghdam, Ehsan Khodapanah and Heidari, Moein and Azad, Reza and Fayyaz, Mohsen and Hacihaliloglu, Ilker and Merhof, Dorit},
  journal={Medical image analysis},
  volume={88},
  pages={102846},
  year={2023},
  publisher={Elsevier}
}

@article{hung2023med,
title={{Med-CDiff: Conditional Medical Image Generation with Diffusion Models}},
author={Hung, Alex Ling Yu and Zhao, Kai and Zheng, Haoxin and Yan, Ran and Raman, Steven S and Terzopoulos, Demetri and Sung, Kyunghyun},
journal={Bioengineering},
volume={10},
number={11},
pages={1258},
year={2023},
publisher={MDPI}
}

@article{dorjsembe2024conditional,
title={{Conditional Diffusion Models for Semantic 3D Brain MRI Synthesis}},
author={Dorjsembe, Zolnamar and Pao, Hsing-Kuo and Odonchimed, Sodtavilan and Xiao, Furen},
journal={IEEE Journal of Biomedical and Health Informatics},
year={2024},
publisher={IEEE}
}

@article{chung2022mr,
title={{MR Image Denoising and Super-Resolution Using Regularized Reverse Diffusion}},
author={Chung, Hyungjin and Lee, Eun Sun and Ye, Jong Chul},
journal={IEEE Transactions on Medical Imaging},
volume={42},
number={4},
pages={922--934},
year={2022},
publisher={IEEE}
}

@article{
boys2024tweedie,
title={Tweedie Moment Projected Diffusions for Inverse Problems},
author={Benjamin Boys and Mark Girolami and Jakiw Pidstrigach and Sebastian Reich and Alan Mosca and Omer Deniz Akyildiz},
journal={Transactions on Machine Learning Research},
issn={2835-8856},
year={2024},
url={https://openreview.net/forum?id=4unJi0qrTE},
note={Featured Certification}
}

@inproceedings{song2019generative,
title={{Generative Modeling by Estimating Gradients of the Data Distribution}},
author={Song, Yang and Ermon, Stefano},
booktitle={Advances in Neural Information Processing Systems},
volume={32},
year={2019}
}

@article{vincent2011score,
  title     = {{A Connection Between Score Matching and Denoising Autoencoders}},
  author    = {Pascal Vincent},
  journal   = {Neural Computation},
  volume    = {23},
  number    = {7},
  pages     = {1661--1674},
  year      = {2011},
  doi       = {10.1162/NECO_a_00142}
}

@article{hyvarinen2005estimation,
  title={Estimation of non-normalized statistical models by score matching.},
  author={Hyv{\"a}rinen, Aapo and Dayan, Peter},
  journal={Journal of Machine Learning Research},
  volume={6},
  number={4},
  year={2005}
}

@article{anderson1982reverse,
  title={Reverse-time diffusion equation models},
  author={Anderson, Brian DO},
  journal={Stochastic Processes and their Applications},
  volume={12},
  number={3},
  pages={313--326},
  year={1982},
  publisher={Elsevier}
}

@inproceedings{
cardoso2024monte,
title={Monte Carlo guided Denoising Diffusion models for Bayesian linear inverse problems},
author={Gabriel Cardoso and Yazid Janati el idrissi and Sylvain Le Corff and Eric Moulines},
booktitle={The Twelfth International Conference on Learning Representations},
year={2024},
url={https://openreview.net/forum?id=nHESwXvxWK}
}

@article{karras2021stylebased,
  title={{A Style-Based Generator Architecture for Generative Adversarial Networks}},
  author={Karras, Tero and Laine, Samuli and Aila, Timo},
  journal={IEEE Transactions on Pattern Analysis \& Machine Intelligence},
  volume={43},
  number={12},
  pages={4217--4228},
  year={2021},
  month={Dec}
}

@book{isakov2006inverse,
  title={{Inverse problems for Partial Differential Equations}},
  author={Isakov, Victor},
  volume={127},
  year={2006},
  publisher={Springer}
}

@article{kingma2013auto,
  title={{Auto-encoding Variational Bayes}},
  author={Kingma, Diederik P},
  journal={arXiv preprint arXiv:1312.6114},
  year={2013}
}

@article{goodfellow2014generative,
  title={{Generative Adversarial Nets}},
  author={Goodfellow, Ian and Pouget-Abadie, Jean and Mirza, Mehdi and Xu, Bing and Warde-Farley, David and Ozair, Sherjil and Courville, Aaron and Bengio, Yoshua},
  journal={Advances in neural information processing systems},
  volume={27},
  year={2014}
}

@inproceedings{bora2017compressed,
  title={{Compressed Sensing using Generative Models}},
  author={Bora, Ashish and Jalal, Ajil and Price, Eric and Dimakis, Alexandros G},
  booktitle={International conference on machine learning},
  pages={537--546},
  year={2017},
  organization={PMLR}
}

@book{groetsch1993inverse,
  title={Inverse problems in the mathematical sciences},
  author={Groetsch, Charles W},
  volume={52},
  year={1993},
  publisher={Springer}
}

@inproceedings{li2021fourier,
  title={Fourier Neural Operator for Parametric Partial Differential Equations},
  author={Li, Zongyi and Kovachki, Nikola Borislavov and Azizzadenesheli, Kamyar and Bhattacharya, Kaushik and Stuart, Andrew and Anandkumar, Anima and others},
  booktitle={International Conference on Learning Representations},
  year={2021}
}

@article{stuart2010inverse,
  title={Inverse problems: a Bayesian perspective},
  author={Stuart, Andrew M},
  journal={Acta numerica},
  volume={19},
  pages={451--559},
  year={2010},
  publisher={Cambridge University Press}
}

@article{metzler2016denoising,
  title={From denoising to compressed sensing},
  author={Metzler, Christopher A and Maleki, Arian and Baraniuk, Richard G},
  journal={IEEE Transactions on Information Theory},
  volume={62},
  number={9},
  pages={5117--5144},
  year={2016},
  publisher={IEEE}
}

@article{metzler2017learned,
  title={Learned D-AMP: Principled neural network based compressive image recovery},
  author={Metzler, Chris and Mousavi, Ali and Baraniuk, Richard},
  journal={Advances in neural information processing systems},
  volume={30},
  year={2017}
}

@inproceedings{zhang2017learning,
  title={Learning deep CNN denoiser prior for image restoration},
  author={Zhang, Kai and Zuo, Wangmeng and Gu, Shuhang and Zhang, Lei},
  booktitle={Proceedings of the IEEE conference on computer vision and pattern recognition},
  pages={3929--3938},
  year={2017}
}

@article{romano2017little,
  title={The little engine that could: Regularization by denoising (RED)},
  author={Romano, Yaniv and Elad, Michael and Milanfar, Peyman},
  journal={SIAM Journal on Imaging Sciences},
  volume={10},
  number={4},
  pages={1804--1844},
  year={2017},
  publisher={SIAM}
}

@article{saharia2022image,
  title={Image super-resolution via iterative refinement},
  author={Saharia, Chitwan and Ho, Jonathan and Chan, William and Salimans, Tim and Fleet, David J and Norouzi, Mohammad},
  journal={IEEE transactions on pattern analysis and machine intelligence},
  volume={45},
  number={4},
  pages={4713--4726},
  year={2022},
  publisher={IEEE}
}

@inproceedings{zhu2023denoising,
  title={Denoising diffusion models for plug-and-play image restoration},
  author={Zhu, Yuanzhi and Zhang, Kai and Liang, Jingyun and Cao, Jiezhang and Wen, Bihan and Timofte, Radu and Van Gool, Luc},
  booktitle={Proceedings of the IEEE/CVF Conference on Computer Vision and Pattern Recognition},
  pages={1219--1229},
  year={2023}
}

@inproceedings{peng2024improving,
  title={Improving Diffusion Models for Inverse Problems Using Optimal Posterior Covariance},
  author={Peng, Xinyu and Zheng, Ziyang and Dai, Wenrui and Xiao, Nuoqian and Li, Chenglin and Zou, Junni and Xiong, Hongkai},
  booktitle={Forty-first International Conference on Machine Learning},
  year={2024}
}

@inproceedings{deng2009imagenet,
  title={{ImageNet: A Large-Scale Hierarchical Image Database}},
  author={Deng, Jia and Dong, Wei and Socher, Richard and Li, Li-Jia and Li, Kai and Fei-Fei, Li},
  booktitle={2009 IEEE Conference on Computer Vision and Pattern Recognition},
  pages={248--255},
  year={2009},
  organization={IEEE}
}

@inproceedings{lugmayr2022repaint,
  title={Repaint: Inpainting using denoising diffusion probabilistic models},
  author={Lugmayr, Andreas and Danelljan, Martin and Romero, Andres and Yu, Fisher and Timofte, Radu and Van Gool, Luc},
  booktitle={Proceedings of the IEEE/CVF conference on computer vision and pattern recognition},
  pages={11461--11471},
  year={2022}
}

@inproceedings{cohen2016group,
  title={Group equivariant convolutional networks},
  author={Cohen, Taco and Welling, Max},
  booktitle={International conference on machine learning},
  pages={2990--2999},
  year={2016},
  organization={PMLR}
}

@article{brehmer2023edgi,
  title={EDGI: Equivariant diffusion for planning with embodied agents},
  author={Brehmer, Johann and Bose, Joey and De Haan, Pim and Cohen, Taco S},
  journal={Advances in Neural Information Processing Systems},
  volume={36},
  pages={63818--63834},
  year={2023}
}

@inproceedings{xu2024equivariant,
  title={Equivariant Graph Neural Operator for Modeling 3D Dynamics},
  author={Xu, Minkai and Han, Jiaqi and Lou, Aaron and Kossaifi, Jean and Ramanathan, Arvind and Azizzadenesheli, Kamyar and Leskovec, Jure and Ermon, Stefano and Anandkumar, Anima},
  booktitle={International Conference on Machine Learning},
  pages={55015--55032},
  year={2024},
  organization={PMLR}
}

@inproceedings{zhou2021meta,
  title={Meta-learning Symmetries by Reparameterization},
  author={Zhou, Allan and Knowles, Tom and Finn, Chelsea},
  booktitle={International Conference on Learning Representations},
  year={2021},
}

@article{dehmamy2021automatic,
  title={Automatic symmetry discovery with lie algebra convolutional network},
  author={Dehmamy, Nima and Walters, Robin and Liu, Yanchen and Wang, Dashun and Yu, Rose},
  journal={Advances in Neural Information Processing Systems},
  volume={34},
  pages={2503--2515},
  year={2021}
}

@article{quessard2020learning,
  title={Learning disentangled representations and group structure of dynamical environments},
  author={Quessard, Robin and Barrett, Thomas and Clements, William},
  journal={Advances in Neural Information Processing Systems},
  volume={33},
  pages={19727--19737},
  year={2020}
}

@inproceedings{mohapatra2025symmetry,
  title={Symmetry-Driven Discovery of Dynamical Variables in Molecular Simulations},
  author={Mohapatra, Jeet and Dehmamy, Nima and Both, Csaba and Das, Subhro and Jaakkola, Tommi},
  booktitle={Forty-second International Conference on Machine Learning},
  year={2025}
}

@inproceedings{park2025approximate,
  title={Approximate Equivariance in Reinforcement Learning},
  author={Park, Jung Yeon and Bhatt, Sujay and Zeng, Sihan and Wong, Lawson LS and Koppel, Alec and Ganesh, Sumitra and Walters, Robin},
  booktitle={International Conference on Artificial Intelligence and Statistics},
  pages={4177--4185},
  year={2025},
  organization={PMLR}
}

@article{thomas2018tensor,
  title={Tensor field networks: Rotation-and translation-equivariant neural networks for 3d point clouds},
  author={Thomas, Nathaniel and Smidt, Tess and Kearnes, Steven and Yang, Lusann and Li, Li and Kohlhoff, Kai and Riley, Patrick},
  journal={arXiv preprint arXiv:1802.08219},
  year={2018}
}

@article{fuchs2020se,
  title={Se (3)-transformers: 3d roto-translation equivariant attention networks},
  author={Fuchs, Fabian and Worrall, Daniel and Fischer, Volker and Welling, Max},
  journal={Advances in neural information processing systems},
  volume={33},
  pages={1970--1981},
  year={2020}
}

@inproceedings{finzi2020generalizing,
  title={Generalizing convolutional neural networks for equivariance to lie groups on arbitrary continuous data},
  author={Finzi, Marc and Stanton, Samuel and Izmailov, Pavel and Wilson, Andrew Gordon},
  booktitle={International Conference on Machine Learning},
  pages={3165--3176},
  year={2020},
  organization={PMLR}
}

@inproceedings{satorras2021n,
  title={E (n) equivariant graph neural networks},
  author={Satorras, V{\i}ctor Garcia and Hoogeboom, Emiel and Welling, Max},
  booktitle={International conference on machine learning},
  pages={9323--9332},
  year={2021},
  organization={PMLR}
}

@inproceedings{
wang2024equivariant,
title={Equivariant Diffusion Policy},
author={Dian Wang and Stephen Hart and David Surovik and Tarik Kelestemur and Haojie Huang and Haibo Zhao and Mark Yeatman and Jiuguang Wang and Robin Walters and Robert Platt},
booktitle={8th Annual Conference on Robot Learning},
year={2024},
url={https://openreview.net/forum?id=wD2kUVLT1g}
}

@inproceedings{hoogeboom2022equivariant,
  title={Equivariant diffusion for molecule generation in 3d},
  author={Hoogeboom, Emiel and Satorras, V{\i}ctor Garcia and Vignac, Cl{\'e}ment and Welling, Max},
  booktitle={International conference on machine learning},
  pages={8867--8887},
  year={2022},
}

@inproceedings{chen2023equidiff,
  title={Equidiff: A conditional equivariant diffusion model for trajectory prediction},
  author={Chen, Kehua and Chen, Xianda and Yu, Zihan and Zhu, Meixin and Yang, Hai},
  booktitle={2023 IEEE 26th International Conference on Intelligent Transportation Systems (ITSC)},
  pages={746--751},
  year={2023},
  organization={IEEE}
}

@article{jiao2023crystal,
  title={Crystal structure prediction by joint equivariant diffusion},
  author={Jiao, Rui and Huang, Wenbing and Lin, Peijia and Han, Jiaqi and Chen, Pin and Lu, Yutong and Liu, Yang},
  journal={Advances in Neural Information Processing Systems},
  volume={36},
  pages={17464--17497},
  year={2023}
}

@article{kabanikhin2008definitions,
  title={Definitions and examples of inverse and ill-posed problems},
  author={Kabanikhin, SI},
  journal={Journal of Inverse and Ill-Posed Problems},
  volume={16},
  number={4},
  pages={317--357},
  year={2008},
  publisher={Brill Academic Publishers}
}

@article{cornet2024equivariant,
  title={Equivariant neural diffusion for molecule generation},
  author={Cornet, Fran{\c{c}}ois and Bartosh, Grigory and Schmidt, Mikkel and Andersson Naesseth, Christian},
  journal={Advances in Neural Information Processing Systems},
  volume={37},
  pages={49429--49460},
  year={2024}
}

@article{ho2020denoising,
  title={Denoising diffusion probabilistic models},
  author={Ho, Jonathan and Jain, Ajay and Abbeel, Pieter},
  journal={Advances in neural information processing systems},
  volume={33},
  pages={6840--6851},
  year={2020}
}

@article{lawrence2025improving,
  title={Improving equivariant networks with probabilistic symmetry breaking},
  author={Lawrence, Hannah and Portilheiro, Vasco and Zhang, Yan and Kaba, S{\'e}kou-Oumar},
  journal={International Conference on Learning Representations},
  year={2025}
}

@article{bronstein2021geometric,
  title={Geometric deep learning: Grids, groups, graphs, geodesics, and gauges},
  author={Bronstein, Michael M and Bruna, Joan and Cohen, Taco and Veli{\v{c}}kovi{\'c}, Petar},
  journal={preprint arXiv:2104.13478},
  year={2021}
}

@article{bloem2020probabilistic,
  title={Probabilistic symmetries and invariant neural networks},
  author={Bloem-Reddy, Benjamin and Whye, Yee and others},
  journal={Journal of Machine Learning Research},
  volume={21},
  number={90},
  pages={1--61},
  year={2020}
}

@inproceedings{terris2024equivariant,
  title={Equivariant plug-and-play image reconstruction},
  author={Terris, Matthieu and Moreau, Thomas and Pustelnik, Nelly and Tachella, Julian},
  booktitle={Proceedings of the IEEE/CVF Conference on Computer Vision and Pattern Recognition},
  pages={25255--25264},
  year={2024}
}

@article{chen2023imaging,
  title={Imaging with equivariant deep learning: From unrolled network design to fully unsupervised learning},
  author={Chen, Dongdong and Davies, Mike and Ehrhardt, Matthias J and Sch{\"o}nlieb, Carola-Bibiane and Sherry, Ferdia and Tachella, Juli{\'a}n},
  journal={IEEE Signal Processing Magazine},
  volume={40},
  number={1},
  pages={134--147},
  year={2023},
  publisher={IEEE}
}

@article{tachella2023sensing,
  title={Sensing theorems for unsupervised learning in linear inverse problems},
  author={Tachella, Juli{\'a}n and Chen, Dongdong and Davies, Mike},
  journal={Journal of Machine Learning Research},
  volume={24},
  number={39},
  pages={1--45},
  year={2023}
}

@inproceedings{kaur2022idecode,
  title={idecode: In-distribution equivariance for conformal out-of-distribution detection},
  author={Kaur, Ramneet and Jha, Susmit and Roy, Anirban and Park, Sangdon and Dobriban, Edgar and Sokolsky, Oleg and Lee, Insup},
  booktitle={Proceedings of the AAAI conference on artificial intelligence},
  volume={36},
  number={7},
  pages={7104--7114},
  year={2022}
}

@inproceedings{kaur2023codit,
  title={Codit: Conformal out-of-distribution detection in time-series data for cyber-physical systems},
  author={Kaur, Ramneet and Sridhar, Kaustubh and Park, Sangdon and Yang, Yahan and Jha, Susmit and Roy, Anirban and Sokolsky, Oleg and Lee, Insup},
  booktitle={Proceedings of the ACM/IEEE 14th International Conference on Cyber-Physical Systems (with CPS-IoT Week 2023)},
  pages={120--131},
  year={2023}
}

@INPROCEEDINGS{moliner2023audioequi,
  author={Moliner, Eloi and Lehtinen, Jaakko and Välimäki, Vesa},
  booktitle={ICASSP 2023 - 2023 IEEE International Conference on Acoustics, Speech and Signal Processing (ICASSP)}, 
  title={Solving Audio Inverse Problems with a Diffusion Model}, 
  year={2023},
  volume={},
  number={},
  pages={1-5},}

@inproceedings{wang2022approximately,
  title={Approximately equivariant networks for imperfectly symmetric dynamics},
  author={Wang, Rui and Walters, Robin and Yu, Rose},
  booktitle={International Conference on Machine Learning},
  pages={23078--23091},
  year={2022},
  organization={PMLR}
}

@inproceedings{alkhouri2025sitcom,
  title={SITCOM: Step-wise Triple-Consistent Diffusion Sampling For Inverse Problems},
  author={Alkhouri, Ismail and Liang, Shijun and Huang, Cheng-Han and Dai, Jimmy and Qu, Qing and Ravishankar, Saiprasad and Wang, Rongrong},
  booktitle={Forty-second International Conference on Machine Learning},
  year={2025}
}

@article{ferreira2018gradient,
  title={Gradient flows of time-dependent functionals in metric spaces and applications to PDEs},
  author={Ferreira, Lucas CF and Valencia-Guevara, Julio C},
  journal={Monatshefte f{\"u}r Mathematik},
  volume={185},
  number={2},
  pages={231--268},
  year={2018},
  publisher={Springer}
}

@article{azulay2019deep,
  title={Why do deep convolutional networks generalize so poorly to small image transformations?},
  author={Azulay, Aharon and Weiss, Yair},
  journal={Journal of Machine Learning Research},
  volume={20},
  number={184},
  pages={1--25},
  year={2019}
}

@inproceedings{he2024manifold,
title={Manifold Preserving Guided Diffusion},
author={Yutong He and Naoki Murata and Chieh-Hsin Lai and Yuhta Takida and Toshimitsu Uesaka and Dongjun Kim and Wei-Hsiang Liao and Yuki Mitsufuji and J Zico Kolter and Ruslan Salakhutdinov and Stefano Ermon},
booktitle={International Conference on Learning Representations},
year={2024},
}

@book{cayton2005algorithms,
  title={Algorithms for manifold learning},
  author={Cayton, Lawrence and others},
  year={2005},
  publisher={Univ. of California at San Diego Tech. Rep}
}

@book{ma2012manifold,
  title={Manifold learning theory and applications},
  author={Ma, Yunqian and Fu, Yun},
  volume={434},
  year={2012},
  publisher={CRC press Boca Raton}
}

@inproceedings{shao2018riemannian,
  title={The riemannian geometry of deep generative models},
  author={Shao, Hang and Kumar, Abhishek and Thomas Fletcher, P},
  booktitle={Proceedings of the IEEE Conference on Computer Vision and Pattern Recognition Workshops},
  pages={315--323},
  year={2018}
}

@inproceedings{anders2020fairwashing,
  title={Fairwashing explanations with off-manifold detergent},
  author={Anders, Christopher and Pasliev, Plamen and Dombrowski, Ann-Kathrin and M{\"u}ller, Klaus-Robert and Kessel, Pan},
  booktitle={International Conference on Machine Learning},
  pages={314--323},
  year={2020},
  organization={PMLR}
}

@inproceedings{lenc2015understanding,
  title={Understanding image representations by measuring their equivariance and equivalence},
  author={Lenc, Karel and Vedaldi, Andrea},
  booktitle={Proceedings of the IEEE conference on computer vision and pattern recognition},
  pages={991--999},
  year={2015}
}

@article{krizhevsky2012imagenet,
  title={Imagenet classification with deep convolutional neural networks},
  author={Krizhevsky, Alex and Sutskever, Ilya and Hinton, Geoffrey E},
  journal={Advances in neural information processing systems},
  volume={25},
  year={2012}
}

@article{romero2022learning,
  title={Learning partial equivariances from data},
  author={Romero, David W and Lohit, Suhas},
  journal={Advances in Neural Information Processing Systems},
  volume={35},
  pages={36466--36478},
  year={2022}
}

@inproceedings{bruintjes2023affects,
  title={What affects learned equivariance in deep image recognition models?},
  author={Bruintjes, Robert-Jan and Motyka, Tomasz and van Gemert, Jan},
  booktitle={Proceedings of the IEEE/CVF Conference on Computer Vision and Pattern Recognition},
  pages={4839--4847},
  year={2023}
}

@inproceedings{zhou2022rethinking,
  title={Rethinking reconstruction autoencoder-based out-of-distribution detection},
  author={Zhou, Yibo},
  booktitle={Proceedings of the IEEE/CVF Conference on Computer Vision and Pattern Recognition},
  pages={7379--7387},
  year={2022}
}

@article{mammadovdiffusion,
  title={Diffusion-Based Inverse Solver on Function Spaces With Applications to PDEs},
  author={Mammadov, Abbas and Berner, Julius and Azizzadenesheli, Kamyar and Ye, Jong Chul and Anandkumar, Anima},
 journal={Machine Learning and the Physical Sciences Workshop at NeurIPS},
  year={2024},
  url={https://ml4physicalsciences.github.io/2024/files/NeurIPS_ML4PS_2024_253.pdf}
}

@article{liu2023genphys,
  title={Genphys: From physical processes to generative models},
  author={Liu, Ziming and Luo, Di and Xu, Yilun and Jaakkola, Tommi and Tegmark, Max},
  journal={arXiv preprint arXiv:2304.02637},
  year={2023}
}

@article{li2025generative,
  title={Generative Latent Neural PDE Solver using Flow Matching},
  author={Li, Zijie and Zhou, Anthony and Farimani, Amir Barati},
  journal={arXiv preprint arXiv:2503.22600},
  year={2025}
}

@article{baldassari2023conditional,
  title={Conditional score-based diffusion models for Bayesian inference in infinite dimensions},
  author={Baldassari, Lorenzo and Siahkoohi, Ali and Garnier, Josselin and Solna, Knut and de Hoop, Maarten V},
  journal={Advances in Neural Information Processing Systems},
  volume={36},
  pages={24262--24290},
  year={2023}
}

@article{kovachki2021neural,
   author    = {Nikola B. Kovachki and
                  Zongyi Li and
                  Burigede Liu and
                  Kamyar Azizzadenesheli and
                  Kaushik Bhattacharya and
                  Andrew M. Stuart and
                  Anima Anandkumar},
   title     = {Neural Operator: Learning Maps Between Function Spaces},
   journal   = {CoRR},
   volume    = {abs/2108.08481},
   year      = {2021},
}

@article{li2020fno,
  title   = {Fourier neural operator for parametric partial differential equations},
  author  = {Li, Zongyi and Kovachki, Nikola and Azizzadenesheli, Kamyar and Liu, Burigede and Bhattacharya, Kaushik and Stuart, Andrew and Anandkumar, Anima},
  journal = {arXiv preprint arXiv:2010.08895},
  year    = {2020}
}

@misc{yao2025guideddiffusionsamplingfunction,
      title={Guided Diffusion Sampling on Function Spaces with Applications to PDEs}, 
      author={Jiachen Yao and Abbas Mammadov and Julius Berner and Gavin Kerrigan and Jong Chul Ye and Kamyar Azizzadenesheli and Anima Anandkumar},
      year={2025},
      journal={Advances in neural information processing systems},
}

@article{jordan1998variational,
  title={The variational formulation of the Fokker--Planck equation},
  author={Jordan, Richard and Kinderlehrer, David and Otto, Felix},
  journal={SIAM journal on mathematical analysis},
  volume={29},
  number={1},
  pages={1--17},
  year={1998},
  publisher={SIAM}
}

@inproceedings{radford2021learning,
  title        = {Learning Transferable Visual Models From Natural Language Supervision},
  author       = {Radford, Alec and Kim, Jong Wook and Hallacy, Chris and Ramesh, Aditya and Goh, Gabriel and Agarwal, Sandhini and Sastry, Girish and Askell, Amanda and Mishkin, Pamela and Clark, Jack and Krueger, Gretchen and Sutskever, Ilya},
  booktitle    = {Proceedings of the 38th International Conference on Machine Learning (ICML)},
  year         = {2021}
}

@inproceedings{he2016deep,
  title        = {Deep Residual Learning for Image Recognition},
  author       = {He, Kaiming and Zhang, Xiangyu and Ren, Shaoqing and Sun, Jian},
  booktitle    = {Proceedings of the IEEE Conference on Computer Vision and Pattern Recognition (CVPR)},
  year         = {2016},
  pages        = {770--778}
}
\bibliographystyle{tmlr}

\newpage

\newcommand{\bb}{{\boldsymbol b}}
\newcommand{\db}{{\boldsymbol d}}
\newcommand{\p}{{\boldsymbol p}}
\newcommand{\tb}{{\boldsymbol t}}
\newcommand{\s}{{\boldsymbol s}}
\newcommand{\ub}{{\boldsymbol u}}
\newcommand{\wb}{{\boldsymbol w}}
\newcommand{\f}{{\boldsymbol f}}
\newcommand{\w}{{\boldsymbol w}}
\newcommand{\Ib}{{\boldsymbol I}}
\newcommand{\Jb}{{\boldsymbol J}}
\newcommand{\Ab}{{\boldsymbol A}}
\newcommand{\Cb}{{\boldsymbol C}}
\newcommand{\Db}{{\boldsymbol D}}
\newcommand{\Sb}{{\boldsymbol S}}
\newcommand{\Pb}{{\boldsymbol P}}
\newcommand{\Mb}{{\boldsymbol M}}
\newcommand{\Wb}{{\boldsymbol W}}
\newcommand{\Rb}{{\boldsymbol R}}
\newcommand{\Hb}{{\boldsymbol H}}
\newcommand{\Fb}{{\boldsymbol F}}
\newcommand{\Rd}{{\mathbb R}}
\newcommand{\Nd}{{\mathbb N}}
\newcommand{\Ed}{{\mathbb E}}
\newcommand{\Mc}{{\mathcal M}}
\newcommand{\Cc}{{\mathcal C}}
\newcommand{\Nc}{{\mathcal N}}
\newcommand{\Pc}{{\mathcal P}}
\newcommand{\Qc}{{\mathcal Q}}
\newcommand{\Xc}{{\mathcal X}}
\newcommand{\Ac}{{\mathcal A}}
\newcommand{\Uc}{{\mathcal U}}


\appendix

\section*{Appendices for ``EquiReg: Equivariance Regularized Diffusion for Inverse Problems''}  %

These supplementary materials contain the following:

\begin{itemize}[leftmargin=5mm]
    \setlength\itemsep{1em}
\item \Cref{app:text_to_image} includes additional experiments on text-to-image guidance. We regularize DreamSampler~\citep{kim2024dreamsampler} with EquiReg for an improved performance (see \Cref{fig:supp_texttoimage_blue,fig:supp_texttoimage_boy,fig:supp_texttoimage_glasses,fig:supp_texttoimage_mus,fig:supp_texttoimage_scu}).

\item \Cref{app:add_exp_robustness} includes additional experiments on robustness including robustness to $\lambda_t$ and reduced number of measurent consistency steps.

\item \Cref{app:add_vis} includes qualitative analysis on the performance of methods with and without EquiReg. Results show a reduction of artifacts and an improved perceptual quality of the solution. This section also includes the equivariance error of a pre-trained encoder used in EquiReg (\Cref{fig:mpe_mpecon_f_train_encoder}).

\item \Cref{app:diverse} includes diversity experiments. Results show that EquiReg achieves improved fidelity without posterior collapse (\Cref{tab:diversity_analysis}, \Cref{fig:diversity_difficulty}, and \Cref{fig:diversity_qualitative_combined}).

\item \Cref{app:equireg_imp} demonstrates EquiReg experimental setup and implementation for PSLD, ReSample, and DPS (\Cref{alg:equipsld,alg:equiconpsld,alg:equiresample,alg:equiconresample,alg:equidps_app}). It also contains information about the EquiReg hyperparameters for image restoration tasks.

\item \Cref{app:pde} contains information on the PDE reconstruction experiment. It discusses the equations along with implementation details and hyperparameters.

\item \Cref{app:theory} provides a Wasserstein-flow perspective on posterior sampling and the proof of \Cref{prop:one}.

\item \Cref{app:add_back} contains additional background information on solving inverse problems, vanishing-error autoencoders, and equivariance.

\item \Cref{app:mpe} provides additional figures and tables.

\item \Cref{app:compute} discloses computing resources used to conduct the experiments.

\item \Cref{app:assets} credits code assets used for our experiments.

\item \Cref{app:resp_release} concludes the appendix with a ``responsible release'' statement.
\end{itemize}

The authors acknowledge the usage of LLMs on proofreading and improving the coherency of the manuscript. The authors have not used LLMs for content generation.


\section{EquiReg for Text-to-Image Guidance}\label{app:text_to_image}

Given a ``source'' image, DreamSampler~\citep{kim2024dreamsampler} transforms it according to a text prompt. When applying EquiReg to DreamSampler, perceptual quality and reduced artifacts are improved in the generated images. \Cref{fig:summary_fig} shows a source ``cat'' transformed into a ``corgi'', where Equi-DreamSampler produces more realistic results and resolves anatomical inconsistencies (e.g., correcting a three-front-legged corgi to two front legs). We also observe an implicit acceleration effect when EquiReg is imposed (\Cref{fig:supp_texttoimage_blue}). Equi-DreamSampler with $50$ DDIM steps produces images comparable to DreamSampler with substantially more steps. Increasing the regularization strength $\lambda_t$ at fixed $50$ DDIM steps results in effects similar to increasing DDIM steps in DreamSampler (from $50$ to $75$ to $100$), suggesting that EquiReg guides sampling trajectories closer to the data manifold. These experiments illustrate that EquiReg is readily applicable beyond inverse problems, including text-to-image guidance, and are offered as an illustrative extension rather than a primary contribution.

\begin{figure}[t]
    \centering
     \begin{subfigure}[b]{0.8\linewidth}
    \centering    
    \includegraphics[width=1.00\textwidth]{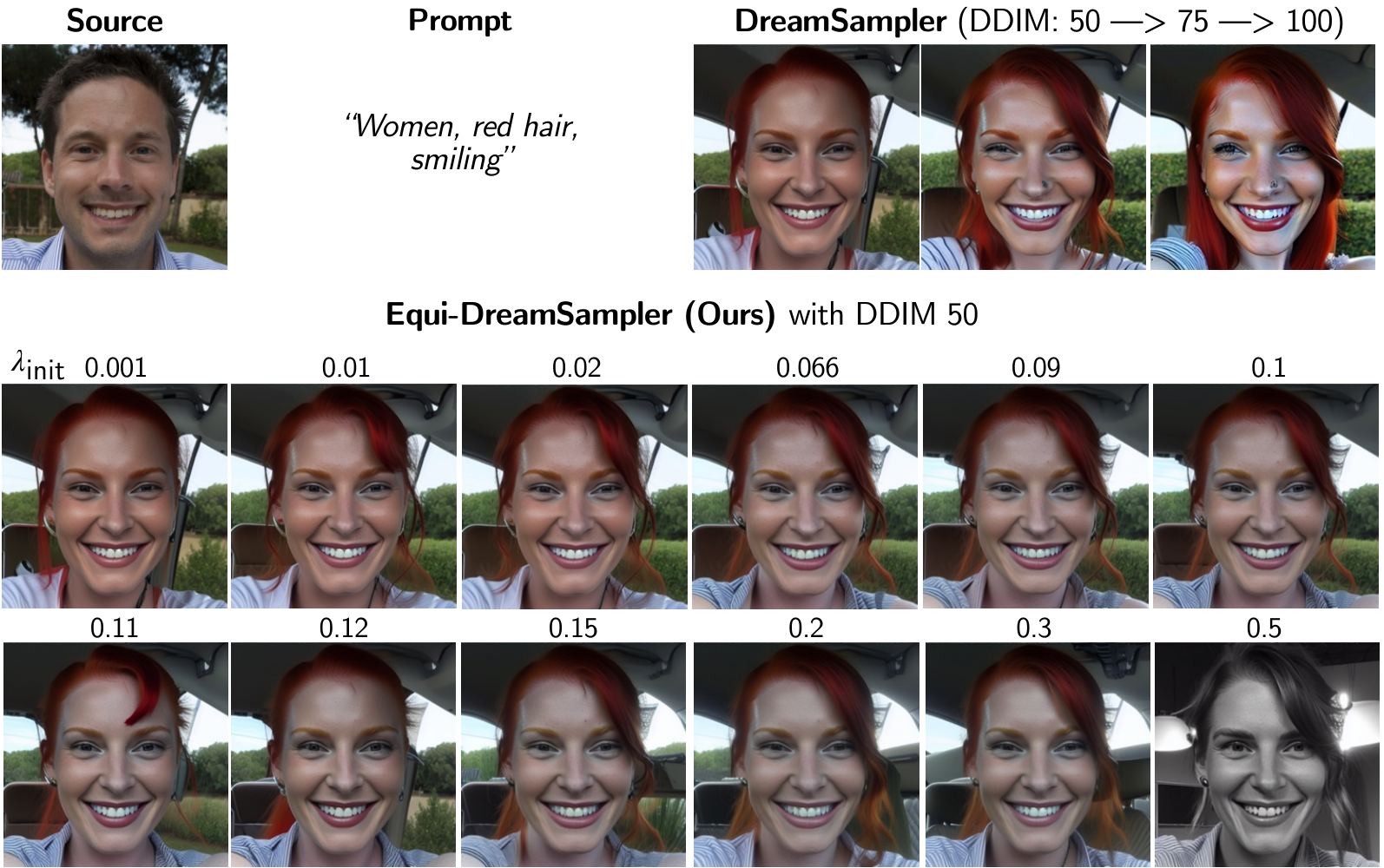} 
    \end{subfigure}
    \vspace{-1mm}
    \caption{\textbf{Impact of EquiReg parameter $\lambda_t$, implicit acceleration, and introduction of more image details on FFHQ 512$\times$512.} Women, red hair, smiling.}
    \label{fig:supp_texttoimage_blue}
    \vspace{-3mm}
\end{figure}
\begin{figure}[t]
    \centering
     \begin{subfigure}[b]{0.8\linewidth}
    \centering    
    \includegraphics[width=0.9\textwidth]{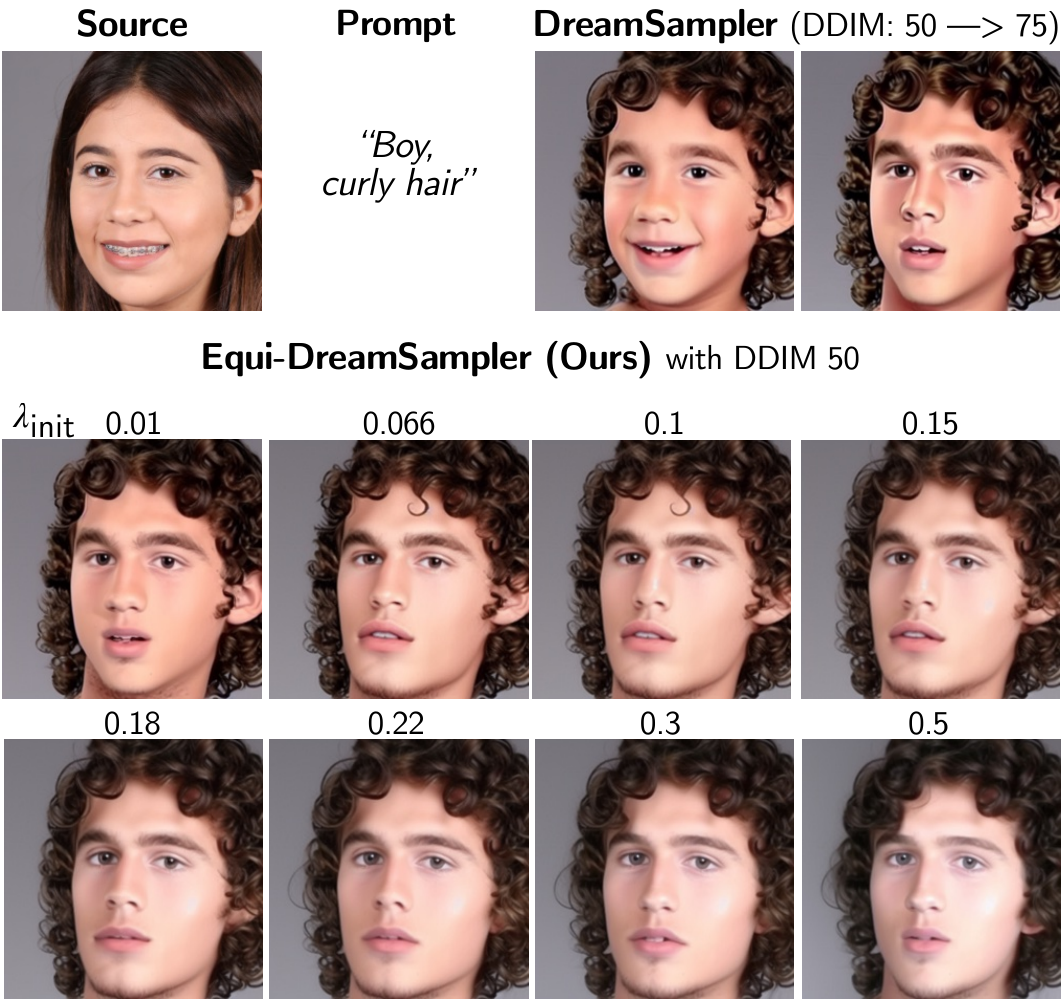} 
    \end{subfigure}
    \vspace{-1mm}
    \caption{\textbf{Adding EquiReg into the text-to-image guidance method DreamSampler for improved performance on FFHQ 512$\times$512.} Boy, curly hair.}
    \label{fig:supp_texttoimage_boy}
    \vspace{-5mm}
\end{figure}

\begin{figure}[H]
    \centering
    \begin{subfigure}[b]{0.8\linewidth}
    \centering    
    \includegraphics[width=0.8\textwidth]{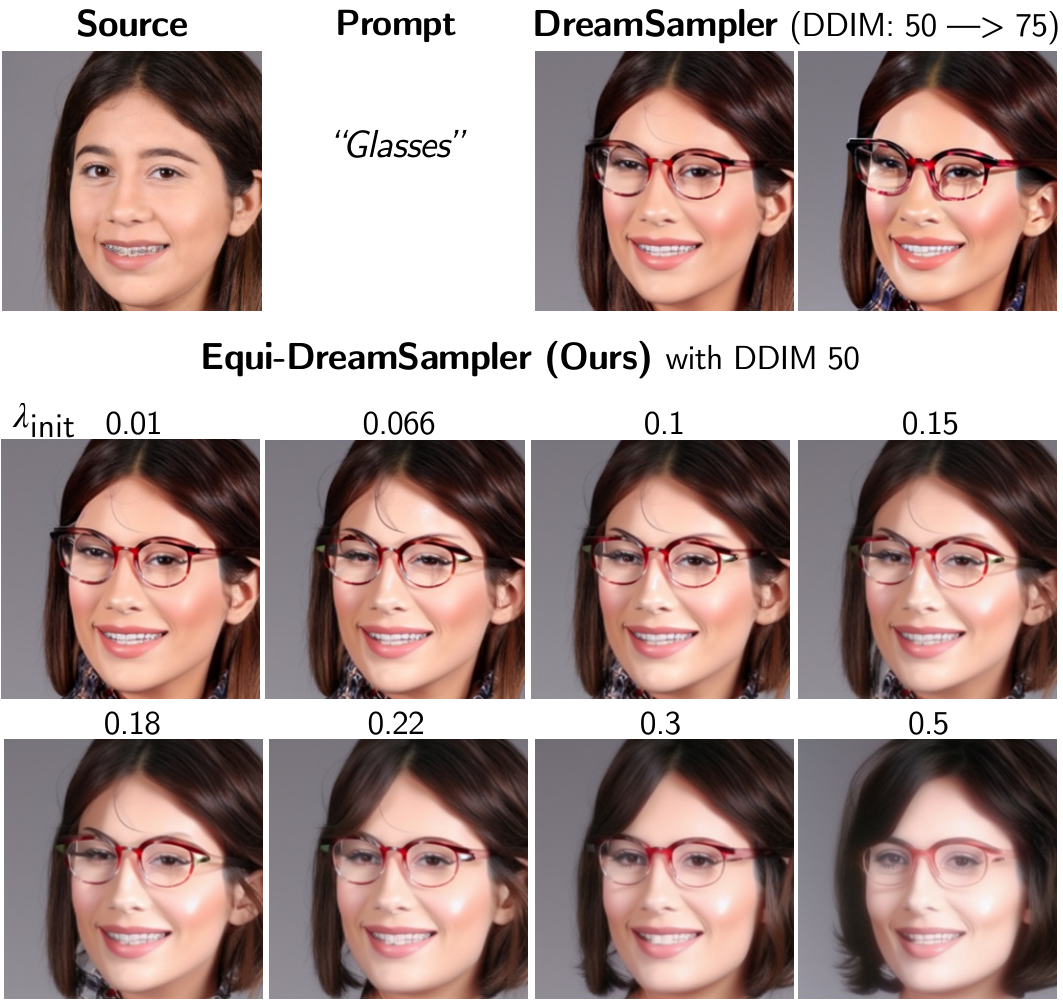} 
    \end{subfigure}
    \vspace{-1mm}
    \caption{\textbf{Adding EquiReg into the text-to-image guidance method DreamSampler for improved performance.} Glasses.}
    \label{fig:supp_texttoimage_glasses}
    \vspace{-5mm}
\end{figure}

\begin{figure}[t]
    \centering
    \begin{subfigure}[b]{0.8\linewidth}
    \centering    
    \includegraphics[width=0.8\textwidth]{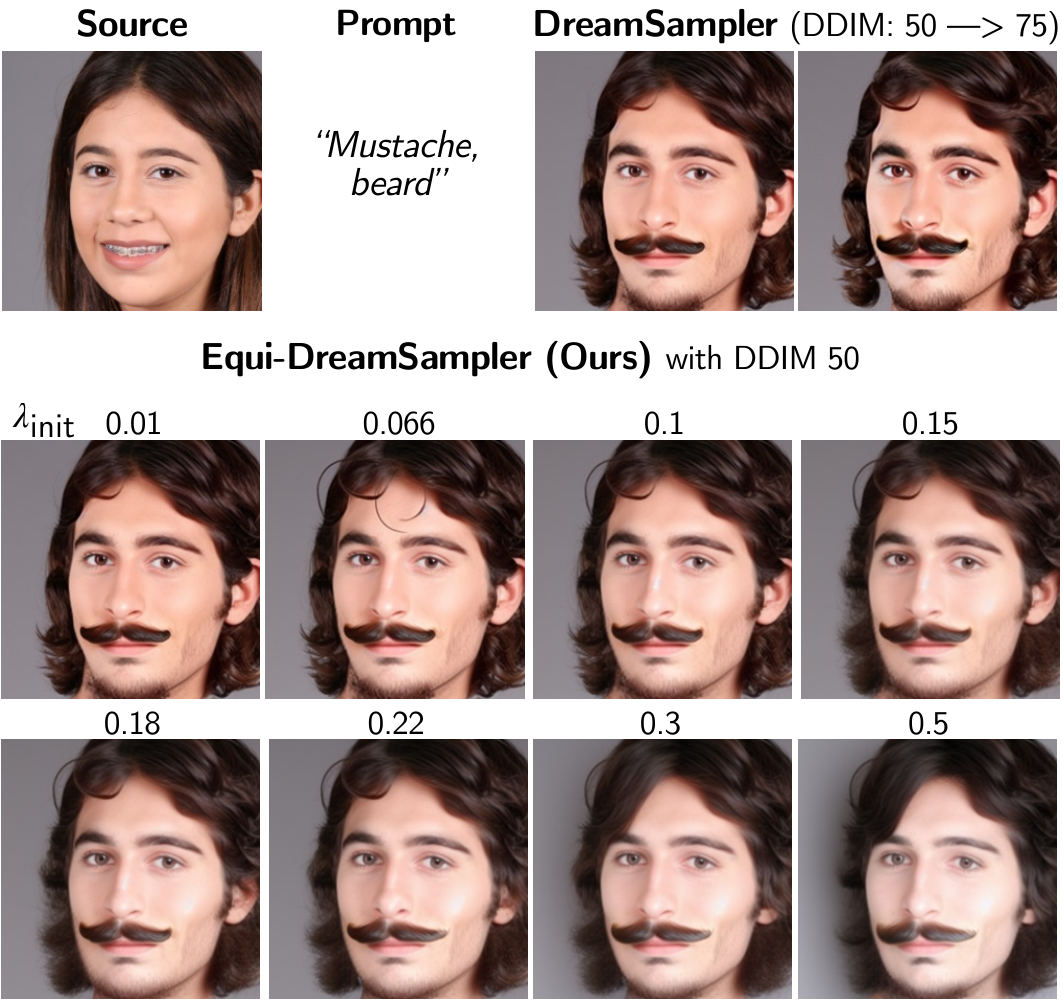} 
    \end{subfigure}
    \vspace{-1mm}
    \caption{\textbf{Adding EquiReg into the text-to-image guidance method DreamSampler for improved performance.} Mustache, beard.}
    \label{fig:supp_texttoimage_mus}
    \vspace{-5mm}
\end{figure}
\begin{figure}[H]
    \centering
    \begin{subfigure}[b]{0.8\linewidth}
    \centering    
    \includegraphics[width=0.8\textwidth]{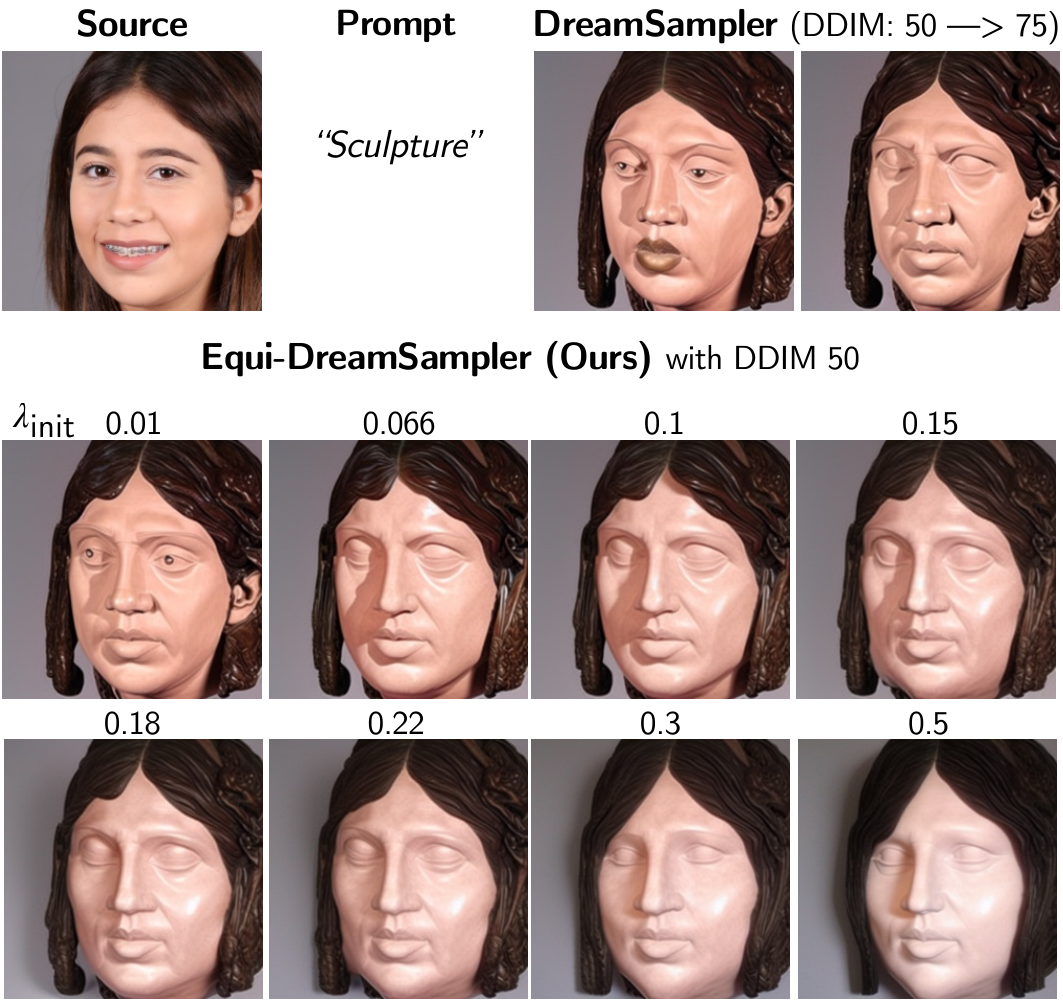} 
    \end{subfigure}
    \vspace{-1mm}
    \caption{\textbf{Adding EquiReg into the text-to-image guidance method DreamSampler for improved performance.} Sculpture.}
    \label{fig:supp_texttoimage_scu}
    \vspace{-5mm}
\end{figure}

\section{Additional Experiments on Robustness}\label{app:add_exp_robustness}

\begin{table}[H]
\centering
\caption{\textbf{EquiReg improves SITCOM under reduced measurement consistency steps ($K_{\text{meas}}$).} Motion deblur on FFHQ sampled with $50$ DDIM steps.}
\vspace{-2mm}
\label{tab:motiondeblur_ffhq_reducedopt_app}
\begin{subtable}{\linewidth}
\fontsize{8}{10}\selectfont
\setlength{\tabcolsep}{4.0pt}
\centering
\begin{tabular}{l*{4}{c}}
\toprule
$K_{\text{meas.}}$ & $K_{\text{EquiReg}}$ & PSNR$\uparrow$ & SSIM$\uparrow$ & Runtime (s)
\\
\midrule
10 & N/A & 28.06 & 0.81 & 21.57 \\
10 & 1   & 28.71 & 0.82 & 21.07 \\
5 & 5   & \bf 29.26 & \bf 0.83 & \bf 11.09 \\
\midrule
20 & N/A & 27.04 & 0.79 & 38.85 \\
20 & 1   & 28.54 & \bf 0.82 & 37.74 \\
10 & 10  & \bf 28.93 & \bf 0.82 & \bf 20.92 \\
\midrule
30 & N/A & 27.79 & 0.80 & 58.84 \\
30 & 1   & 28.35 & 0.81 & 55.51 \\
15 & 15   & \bf 29.63 & \bf 0.84 & \bf 30.19 \\
\midrule
40 & N/A & 30.40 & \bf 0.85 & 78.08 \\
40 & 1   & \bf 30.58 & \bf 0.85 & 69.83 \\
20 & 20  & 29.50 & 0.83 & \bf 41.02 \\
\midrule
60 & N/A & 28.35 & 0.81 & 108.57\\
60 & 1 & 27.02 & 0.78 & 95.62\\
30 & 30  & \bf 31.36 & \bf 0.87 & \bf 59.38 \\
\bottomrule
\end{tabular}
\label{tab:results_ffhq_sitcom_eff_supp}
\end{subtable}
\end{table}
\vspace{7mm}

\begin{table}[H]
\vspace{-4mm}
\centering
\caption{\textbf{EquiReg Effectiveness with Subset of Group Actions.}}
\begin{subtable}{\linewidth}
\fontsize{8}{9.5}\selectfont
\setlength{\tabcolsep}{3pt}
\renewcommand{\arraystretch}{1.1}
\centering
\begin{tabular}{r *{4}{c}}
\toprule
\multicolumn{2}{c}{\textbf{PSLD}} & \multicolumn{2}{c}{\textbf{Equi-PSLD} (90, 270 deg)} \\
\cmidrule(lr){1-2}\cmidrule(lr){3-4}
PSNR$\uparrow$ & SSIM$\uparrow$ & PSNR$\uparrow$ & SSIM$\uparrow$ \\
\midrule
15.86 (1.19) & 0.77 (0.03)
      & 17.60 (1.60) & 0.79 (0.03)\\
\bottomrule
\end{tabular}
\label{tab:results_subsetgroup}
\end{subtable}
\end{table}

\section{Visualizations for Image Restoration Experiments}\label{app:add_vis}

\begin{figure}[H]
    \centering  
    \begin{subfigure}[b]{0.9\linewidth}
    \centering  
    \includegraphics[width=0.7\linewidth]{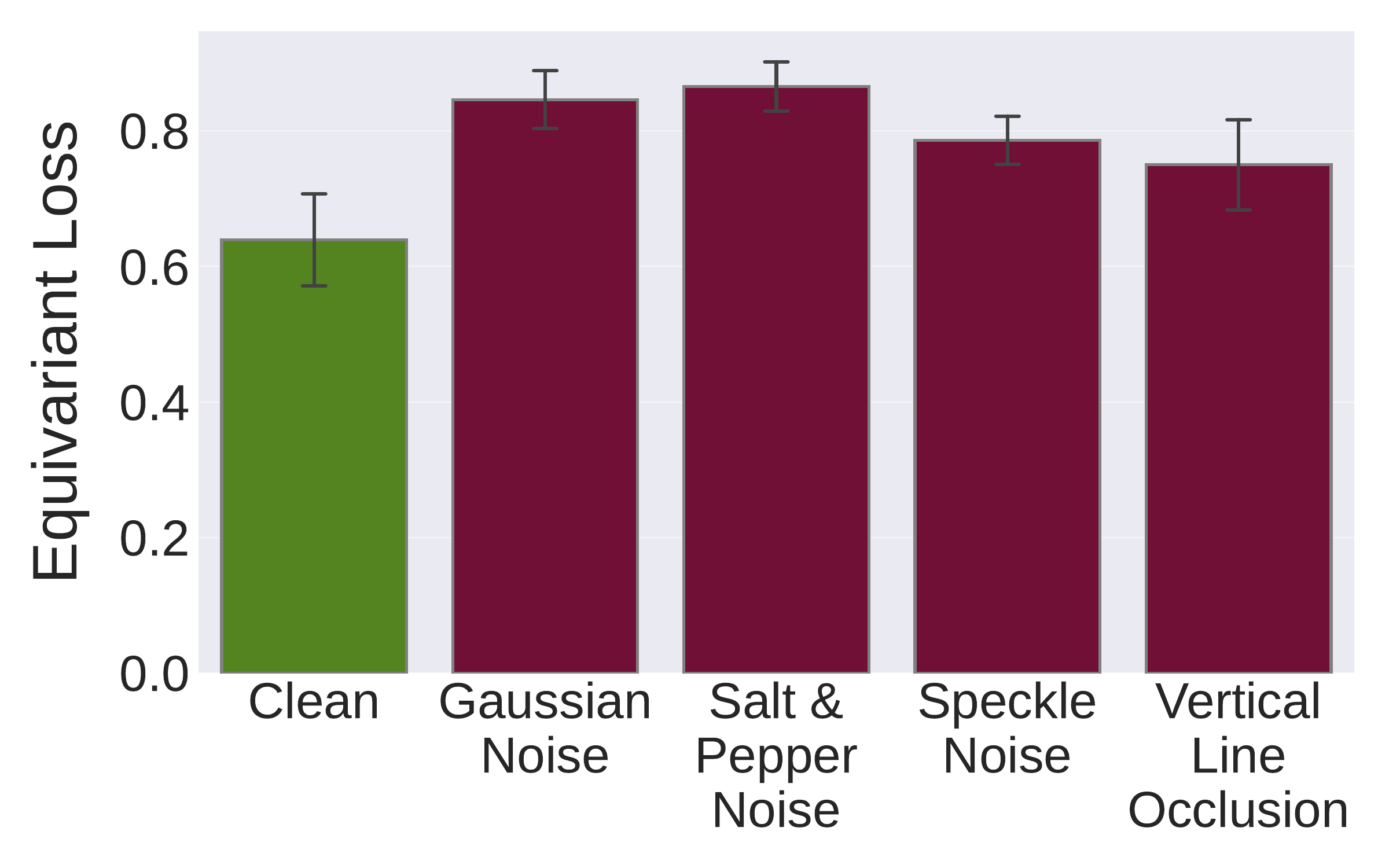}
    \caption{The equivariance error of the encoder is lower on clean, natural images than corrupted ones.}
    \label{fig:mpe_mpecon_f_train_encoder}
    \end{subfigure}
    \hfill
    \newline
    \begin{subfigure}[b]{0.99\linewidth}
    \centering  
    \vspace{4mm}
    \includegraphics[width=0.99\linewidth]{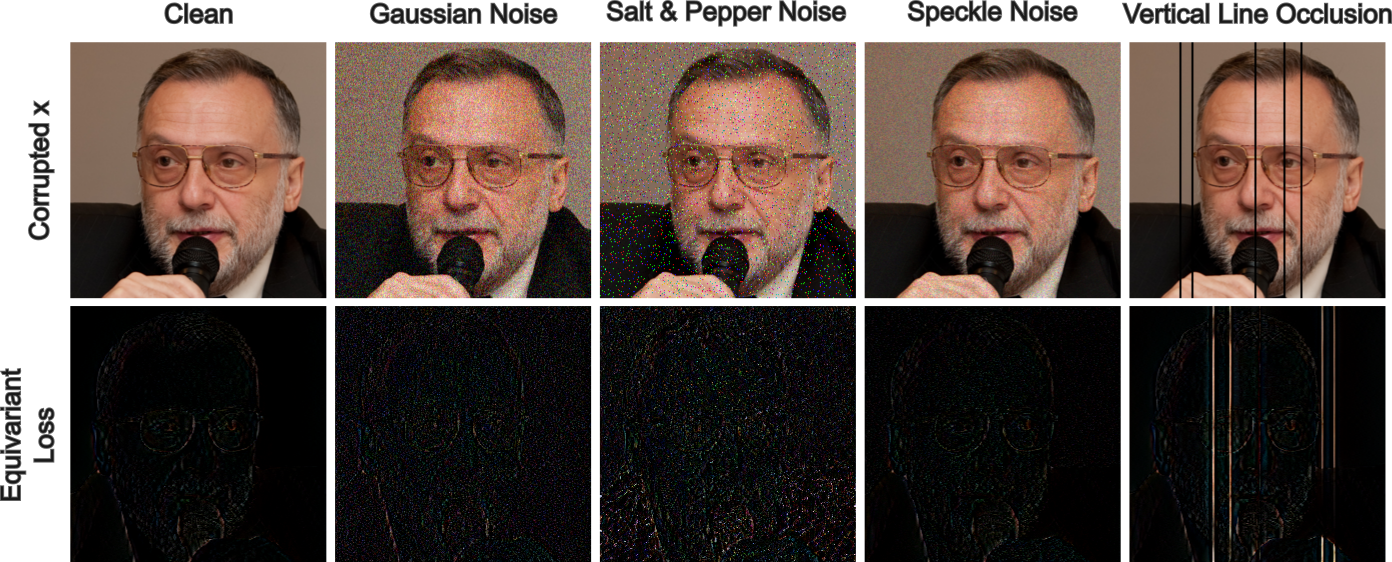}
    \caption{Example visualizations of used images and corresponding equivariance error computed using the decoder (see~\Cref{fig:mpe_mpecon_f_train}).}
    \label{fig:mpe_mpecon_f_train_decoder_vis}
    \end{subfigure}
    \caption{\textbf{Training induced equivariance for a pre-trained function.}}
    \label{fig:mpe_mpecon}
\end{figure}%


\begin{figure}[H]
\centering
\begin{subfigure}[b]{0.48\linewidth}
    \centering
    \includegraphics[width=\linewidth]{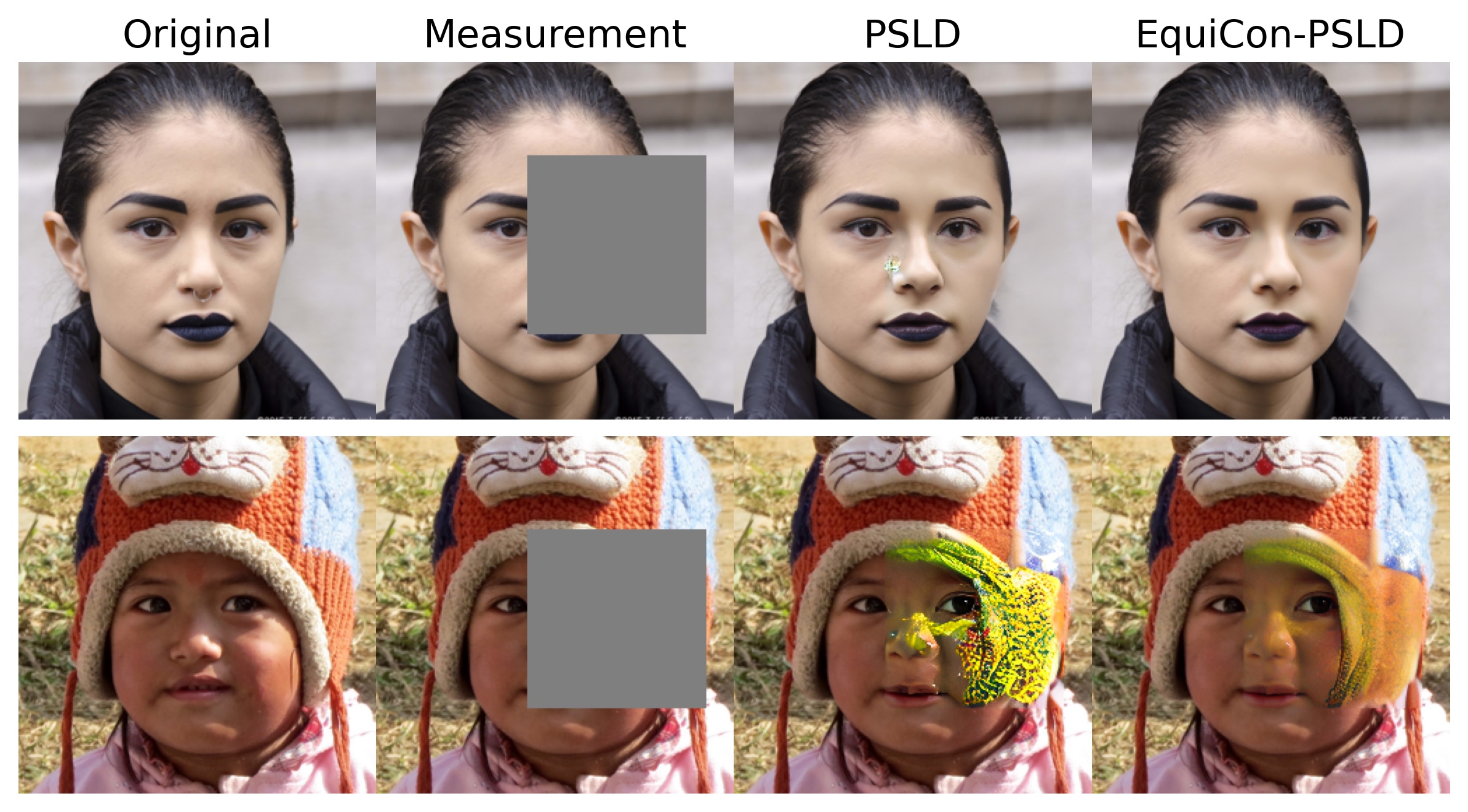}
    \caption{Box inpainting.}
    \label{fig:qual_box_inpainting}
\end{subfigure}
\hfill
\begin{subfigure}[b]{0.48\linewidth}
    \centering
    \includegraphics[width=\linewidth]{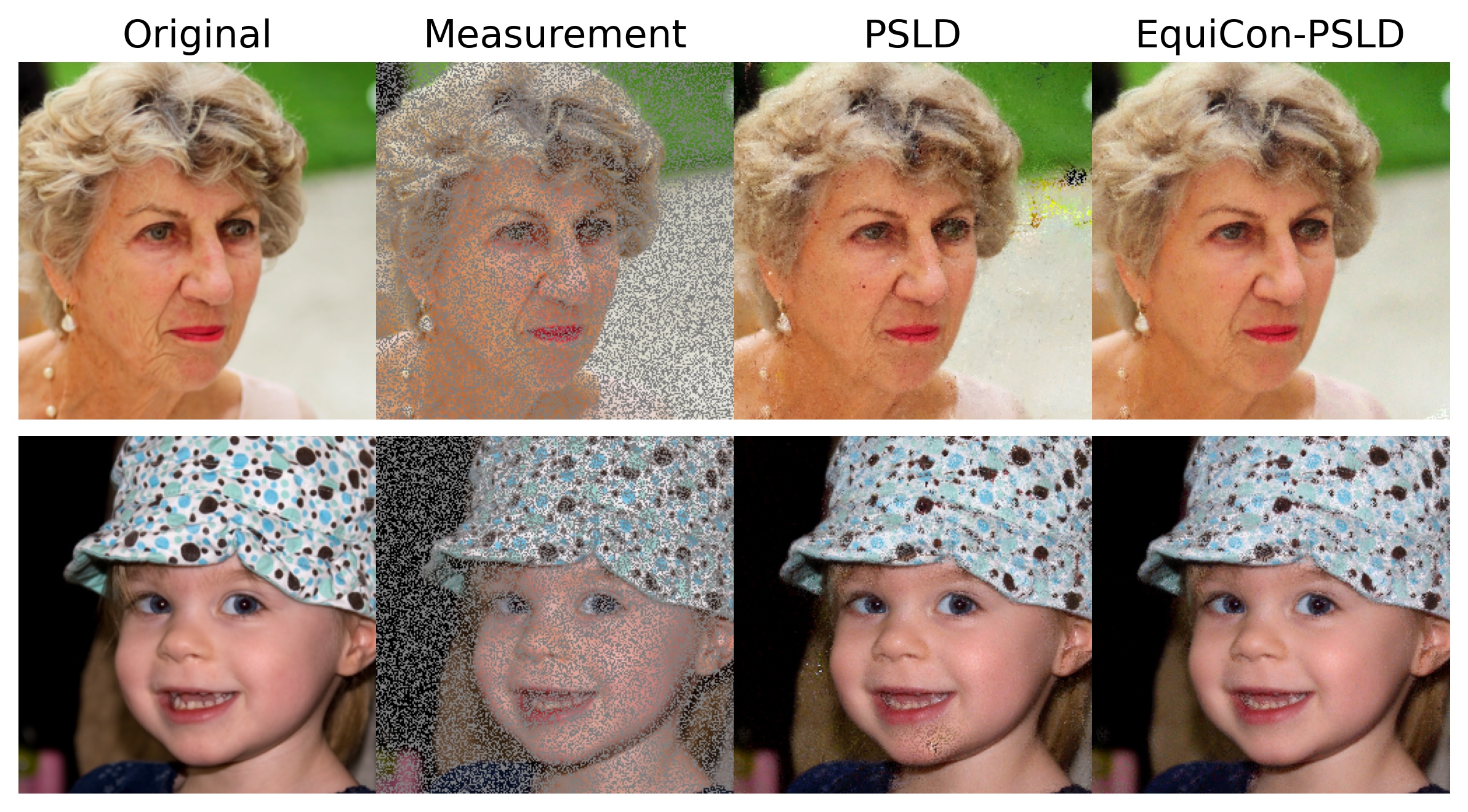}
    \caption{Random inpainting.}
    \label{fig:qual_random_inpainting_psld}
\end{subfigure}

\vspace{1em}

\begin{subfigure}[b]{0.48\linewidth}
    \centering
    \includegraphics[width=\linewidth]{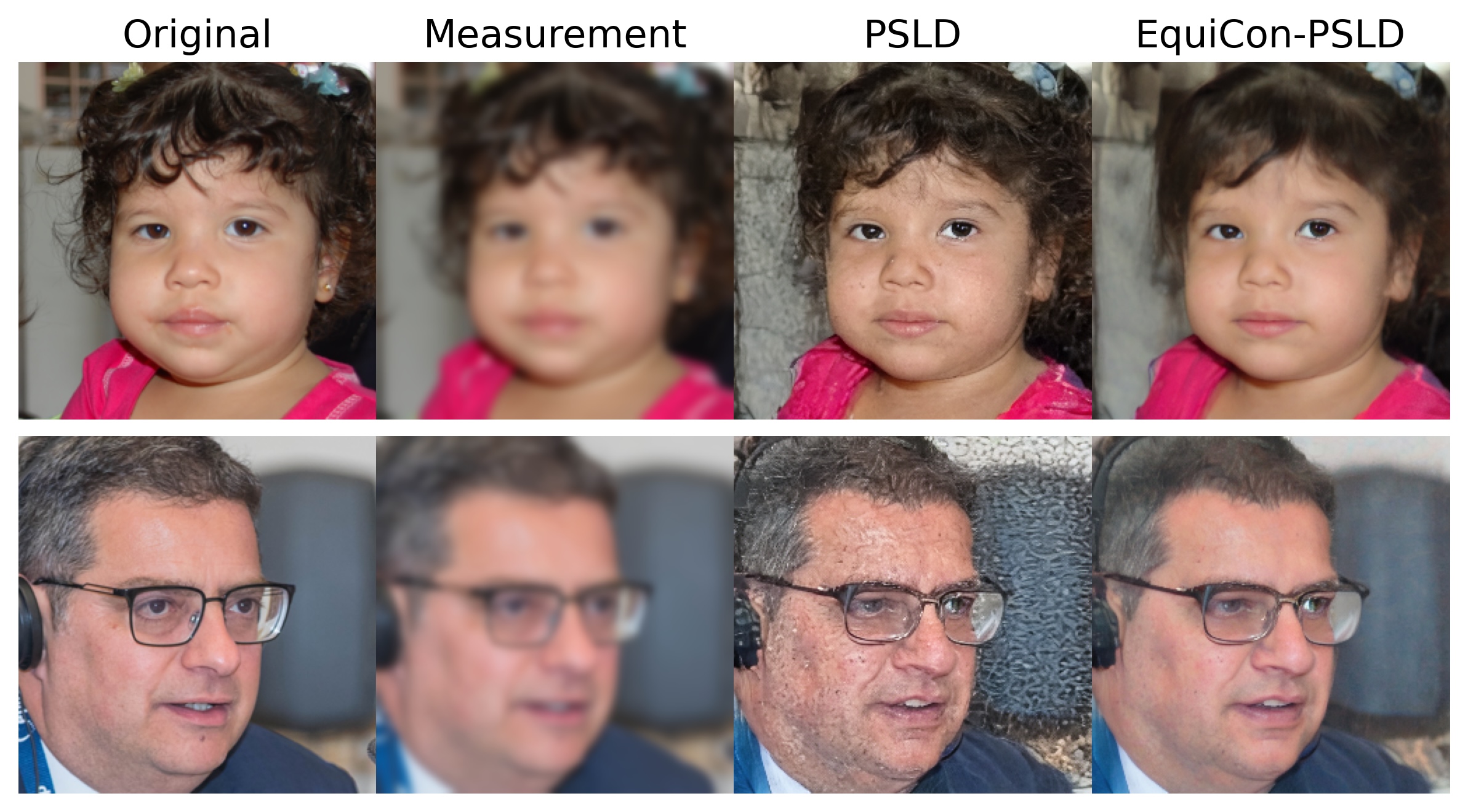}
    \caption{Gaussian deblur.}
    \label{fig:qual_gaussian_deblur}
\end{subfigure}
\hfill
\begin{subfigure}[b]{0.48\linewidth}
    \centering
    \includegraphics[width=\linewidth]{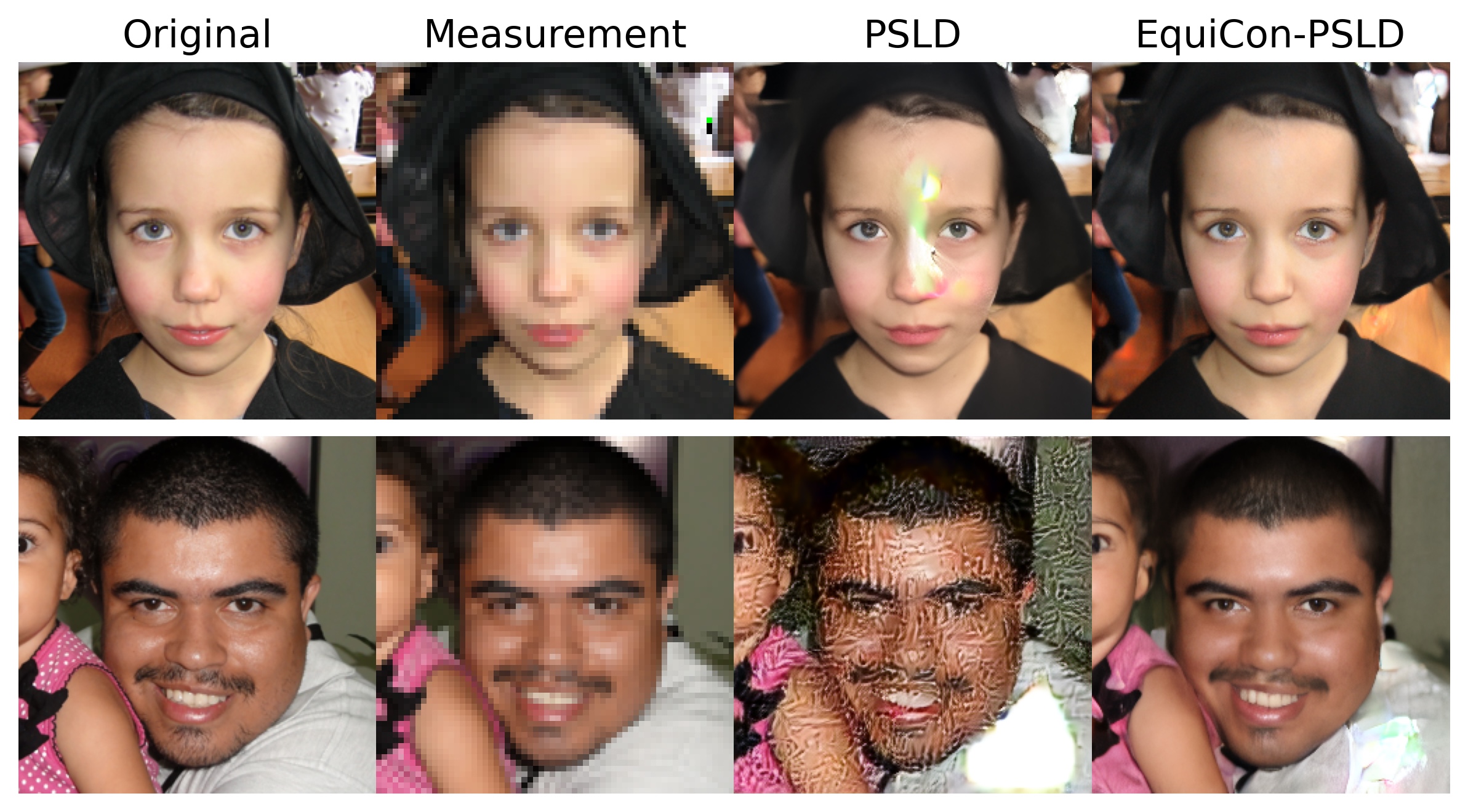}
    \caption{Super-resolution ($\times 4$).}
    \label{fig:qual_superres_psld}
\end{subfigure}

\vspace{0.8em}

\caption{\textbf{Qualitative comparison of EquiCon-PSLD and PSLD on FFHQ 256 $\times$ 256.} }
\label{fig:qual_comparison_tasks_psld}

\end{figure}

\section{Diversity Analysis}\label{app:diverse}

In the Bayesian setting, the objective of solving inverse problems with diffusion models is to sample from high-probability regions of the posterior distribution. While the goal is not to maximize ``diversity'', the true diversity emerges when the posterior  admits meaningful variability. In practice, diversity-related concerns in inverse problems arise when a method suffers from mode collapse, i.e., the sampler becomes biased and fails to explore multiple plausible modes of the posterior. Thus, the relevant question is whether a method properly explores the posterior rather than whether it maximizes diversity in an unconstrained sense.

Because closed-form posteriors are unavailable for real image restoration tasks, the standard practice in the diffusion inverse-problem literature is to evaluate diversity through variation among plausible reconstructions consistent with the measurement, without collapsing to a single solution. This is the notion of “diversity” our work adopts.

Given the goal of posterior sampling, EquiReg is not designed to maximize diversity for its own sake. Its objective is to incorporate data-inherent geometric structure (equivariance) to guide sampling toward high-probability regions of the posterior.  Hence, diversity arises naturally from the ill-posedness of the inverse problem; it is a consequence of posterior uncertainty, not the goal of the regularizer.

To quantify this effect, in addition to reconstruction quality, we analyzed the diversity of posterior samples produced by EquiReg. We evaluate diversity metrics across multiple tasks and difficulty levels to characterize the sampling behavior of our method.

We emphasize that the intra-LPIPS and pixel-std metrics reported below are diagnostics for posterior collapse, not measures of reconstruction quality. The relevant comparison is between an EquiReg-regularized solver and the same solver without regularization, on the same task. Finding similar or higher diversity under EquiReg while fidelity improves therefore indicates that EquiReg does not artificially shrink posterior coverage, not that diversity is itself a performance metric. A rigorous calibration analysis against a known posterior (for example, on a toy problem with closed-form posterior) is a worthwhile direction we leave to future work.

\subsection{Experimental Setup}
To evaluate diversity, we generate multiple posterior samples and measure variation across these samples. For each of 20 test images, we generate K=10 reconstructions using different random seeds. We evaluate diversity using two complementary metrics: Intra-LPIPS, which measures perceptual diversity by  computing the average LPIPS distance between all pairs of samples, and  Pixel-Std, which measures spatial diversity through pixel-wise standard  deviation across samples. Higher values for both metrics indicate greater  diversity. For Intra-LPIPS, we compute distances for all $\binom{K}{2} = 45$  pairs per image and average across all test images. For Pixel-Std, we compute  the standard deviation at each pixel location across the K samples, then  average across all pixels and test images. We evaluate diversity across three inverse problems (box inpainting, Gaussian  deblurring, and 4$\times$ super-resolution) comparing EquiReg against  DPS~\citep{chung2023diffusion} without equivariance regularization. To  investigate how diversity scales with task difficulty, we additionally vary  the inpainting mask size from $128 \times 128$ (standard) to $160 \times 160$ 
to $192 \times 192$ pixels.

\subsection{Results and Discussion}
\Cref{tab:diversity_analysis} shows that Equi-DPS achieves improved fidelity without posterior collapse across three inverse problems. For box inpainting and super-resolution, equivariance regularization improves both fidelity and diversity simultaneously. For Gaussian deblurring, Equi-DPS achieves 15-20\% better fidelity while retaining 80-85\% of baseline diversity, representing a modest but justified trade-off. These results demonstrate that equivariance constraints do not inherently suppress diversity; rather, they can guide sampling toward regions of higher data fidelity while maintaining posterior exploration.

\Cref{fig:diversity_difficulty} reveals linear diversity scaling with task difficulty. Diversity metrics grow proportionally with task difficulty, indicating Equi-DPS naturally expands sampling as problems become more ill-posed. This linear relationship demonstrates stable, predictable behavior across difficulty levels without artificial diversity suppression. \Cref{fig:diversity_qualitative,fig:diversity_qualitative_combined} provide qualitative results.

\begin{figure}[H]
\centering

\begin{subfigure}[b]{0.85\textwidth}
    \centering
    \includegraphics[width=\textwidth]{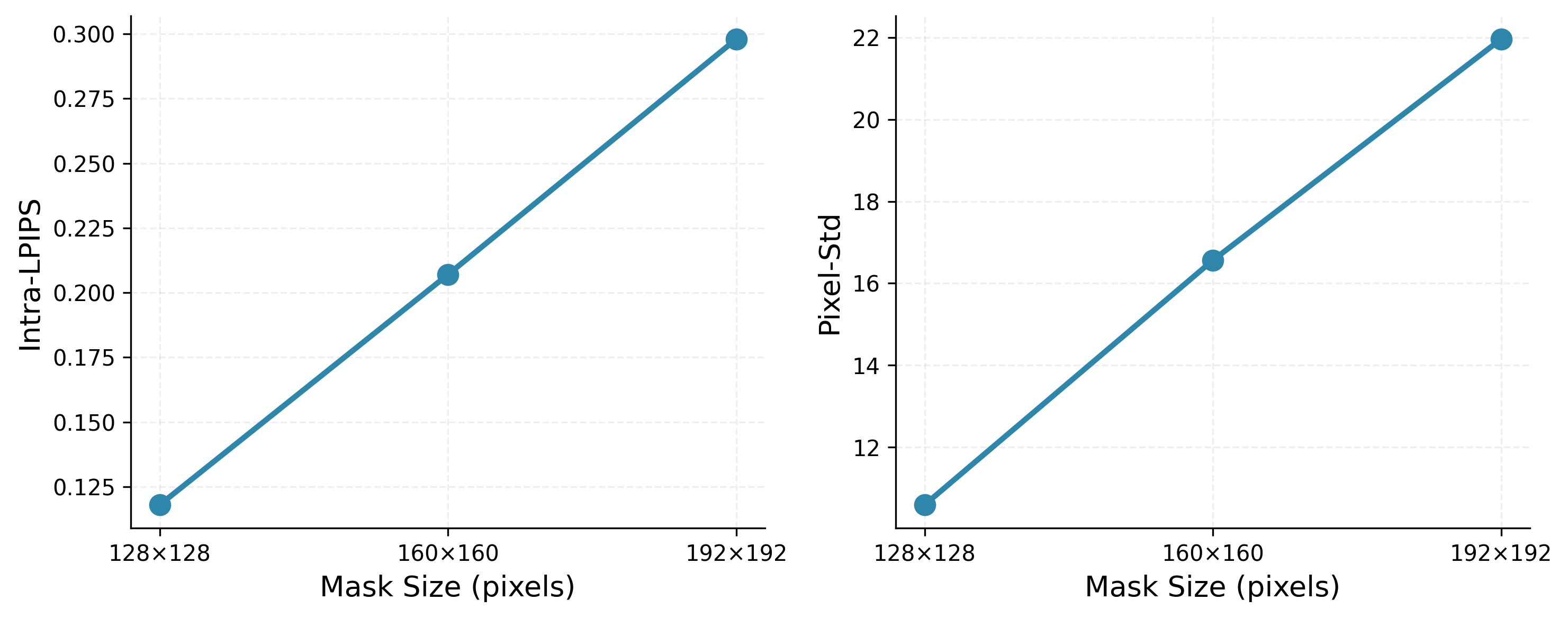}
    \caption{Box inpainting}
    \label{fig:div_difficulty_box}
\end{subfigure}

\vspace{10pt}

\begin{subfigure}[b]{0.85\textwidth}
    \centering
    \includegraphics[width=\textwidth]{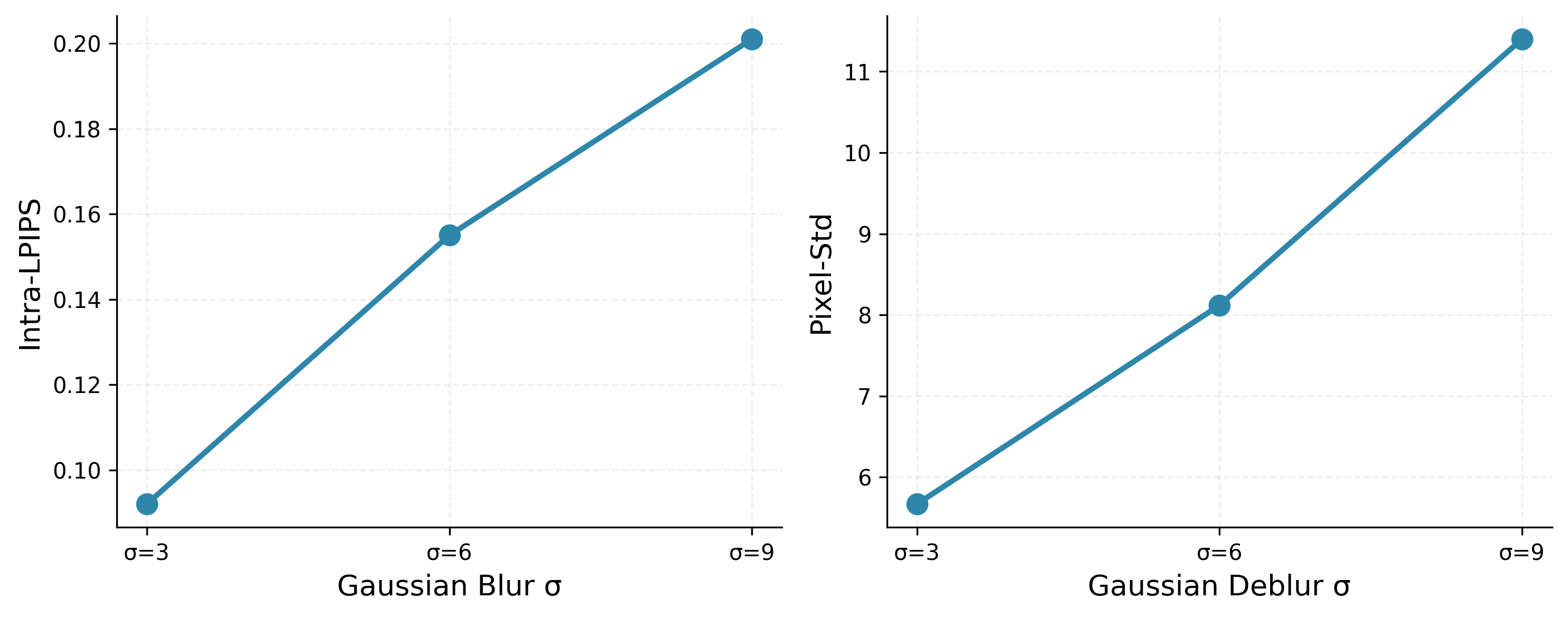}
    \caption{Gaussian deblur}
    \label{fig:div_difficulty_gaussian}
\end{subfigure}

\vspace{10pt}

\begin{subfigure}[b]{0.85\textwidth}
    \centering
    \includegraphics[width=\textwidth]{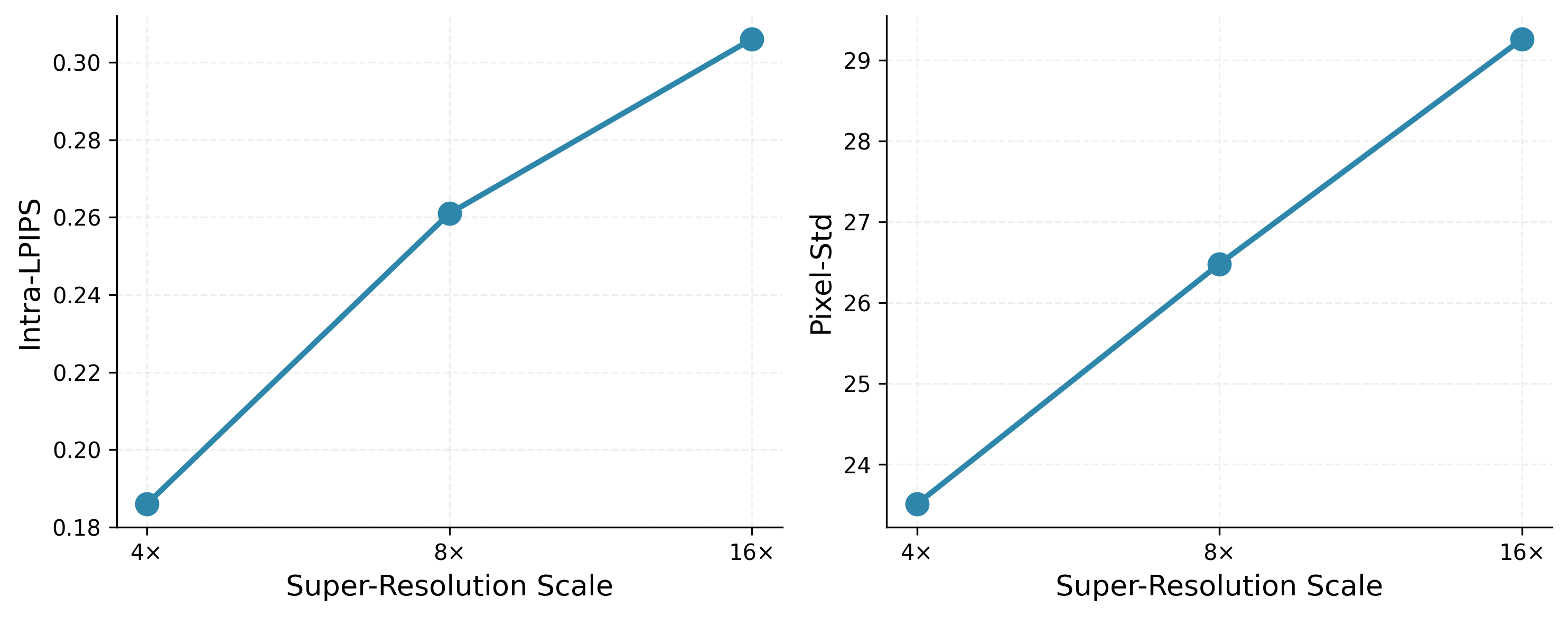}
    \caption{Super-resolution}
    \label{fig:div_difficulty_sr}
\end{subfigure}

\caption{\textbf{Diversity vs task difficulty across three inverse problems.} 
As task difficulty increases (larger inpainting mask, stronger blur, higher SR scale), both diversity metrics increase proportionally, demonstrating that Equi-DPS maintains healthy posterior sampling behavior across a wide difficulty spectrum.}
\label{fig:diversity_difficulty}
\end{figure}

\begin{figure}[H]
\vspace{-2mm}
\centering

\begin{subfigure}[b]{0.95\textwidth}
    \centering
    \includegraphics[width=\textwidth]{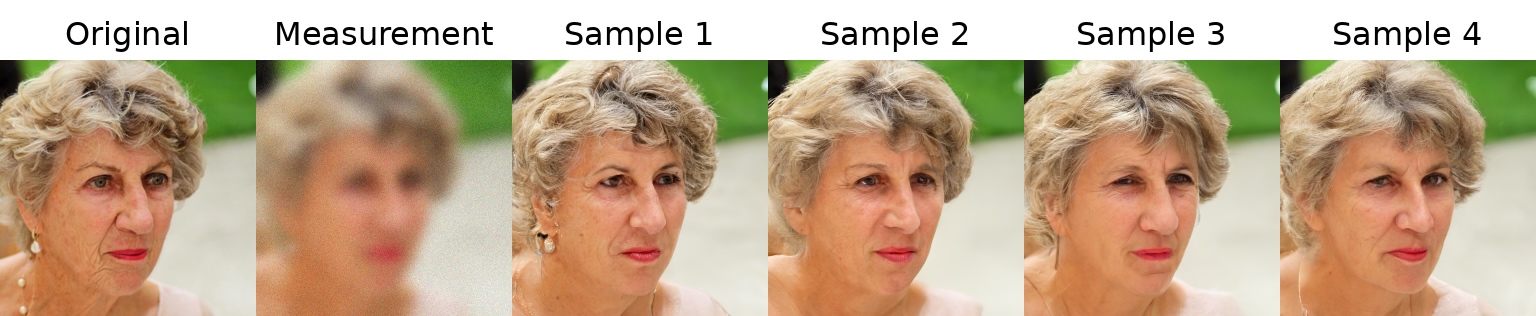}
    \vspace{2mm}
    \includegraphics[width=\textwidth]{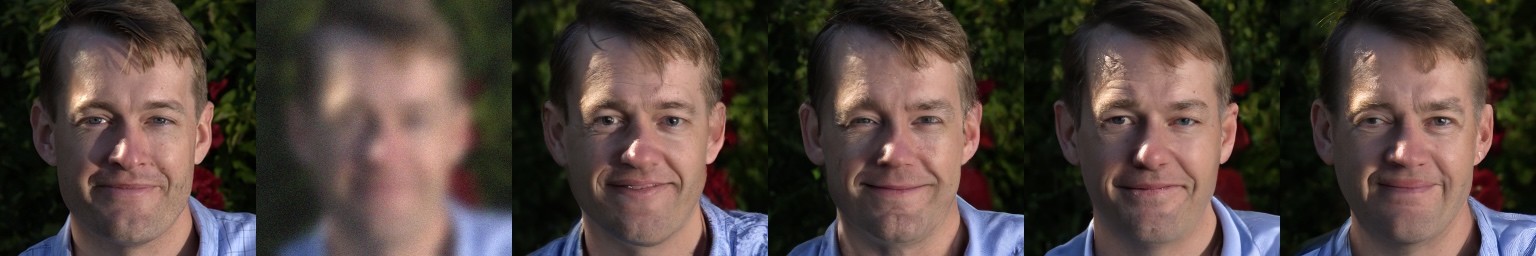}
    \caption{Gaussian deblur}
    \label{fig:diversity_gaussian}
\end{subfigure}

\vspace{2mm}

\begin{subfigure}[b]{0.95\textwidth}
    \centering

    \includegraphics[width=\textwidth]{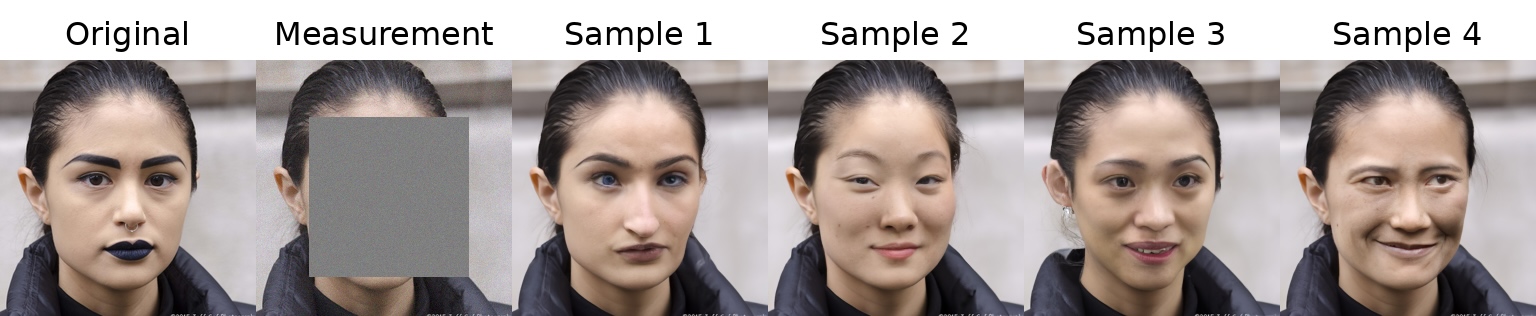}
    \vspace{2mm}
    \includegraphics[width=\textwidth]{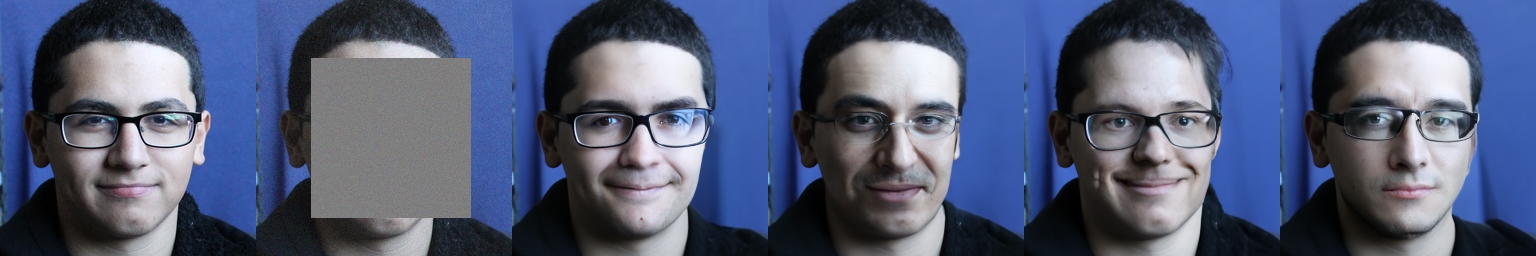}
    
    \caption{Box inpainting}
    \label{fig:diversity_bip}
\end{subfigure}

\vspace{0mm}

\caption{\textbf{Qualitative diversity examples across Gaussian deblur and box inpainting.} 
Each subfigure shows $K=4$ posterior samples for two different test images. 
(a) Gaussian deblur: samples differ in facial expressions and accessories (i.e., earrings in first test image).
(b) Box inpainting: Inputs were obstructed with $160 \times 160$ masks. Samples exhibit perceptually distinct facial features (i.e., expressions, eye gaze, facial structure).
Across both tasks, EquiReg produces diverse plausible reconstructions rather than collapsing to a single mode.}
\label{fig:diversity_qualitative_combined}
\vspace{-4mm}
\end{figure}

\subsection{Conclusion}

Finally, we highlight that EquiReg improves both fidelity and diversity on 2 of the 3 considered tasks, an encouraging outcome that is uncommon given the general behavior of classical regularizers. Hand-crafted regularizers such as TV and $\ell_1$ may suppress diversity by shrinking solutions toward simple structures. By contrast, EquiReg leverages data-dependent regularization that captures the richness and structural complexity of the underlying data manifold, enabling it to preserve manifold-consistent variability while suppressing implausible samples.

High diversity without fidelity is not meaningful for posterior sampling. A method that samples the entire solution space, including low-probability and artifacted regions, may score well on diversity but fail to provide useful reconstructions.
Equi-DPS avoids this failure mode: it maintains meaningful diversity while reducing artifacts and improving perceptual quality. In the experiments conducted during the rebuttal, our goal was to demonstrate clearly that EquiReg preserves meaningful diversity, reflecting the posterior uncertainty, rather than unstructured or unconstrained variability.

\section{Implementation Details for Image Restoration Tasks}\label{app:equireg_imp}
\textbf{Experimental Setup. } We evaluate EquiReg on a variety of linear and nonlinear restoration tasks for natural images. We fix sets of $100$ images from FFHQ and ImageNet as our validation sets. All images are normalized from $[0, 1]$. For the majority of experiments, we use noise level $\sigma_{\y} = 0.05$ (we indicate $\sigma_{\y}$ in our tables). For linear inverse problems, we consider (1) box inpainting, (2) random inpainting, (3) Gaussian deblur, (4) motion deblur, and (5) super-resolution. We apply a random $128 \times 128$ pixel box for box inpainting, and a $70\%$ random mask for random inpainting. For Gaussian and motion deblur, we use kernels of size $61 \times 61$, with standard deviations of $3.0$ and $0.5$, respectively. For super-resolution, we downscale images by a factor of $4$ using a bicubic resizer. For nonlinear inverse problems, we consider (1) phase retrieval, (2) nonlinear deblur, and (3) high dynamic range (HDR). We use an oversampling rate of $2.0$ for phase retrieval, and due to instability of the task, we generate four independent reconstructions and take the best result (as also done in DPS \citep{chung2023diffusion}, DAPS \citep{zhang2025daps}, and DiffStateGrad \citep{zirvi2025diffusion}). We use the default setting from \citep{tran2021explore} for nonlinear deblur, and a scale factor of 2 for HDR.

\textbf{Hyperparameters.}\quad Our method introduces a single hyperparameter $\lambda_t$ that controls the amount of regularization applied. Below we include a table detailing the use of this hyperparameter in the main experiments (\Cref{tab:hyperparams_ffhq}). For all experiments, we keep $\lambda_t$ constant throughout iterations. We note that for the DPS method, the effective EquiReg regularization is given by $\zeta_t \lambda_t$. Therefore, scheduling the measurement guidance to decrease over time proportionally reduces the effective equivariance regularization as well.

\setlength{\tabcolsep}{4pt} 
\begin{table}[H]
  \centering
  \caption{Equivariance regularization weight $\lambda_t$ used in main experiments.}
  \label{tab:hyperparams_ffhq}
  \begin{tabular}{lccccc}
    \toprule
    Method
      & \shortstack{Box\\Inpainting}
      & \shortstack{Random\\Inpainting}
      & \shortstack{Gaussian\\Deblur}
      & \shortstack{Motion\\Deblur}
      & \shortstack{Super-resolution\\($\times4$)} \\
    \midrule
    \multicolumn{6}{l}{\textit{FFHQ $256 \times 256$}} \\
    Equi-PSLD        & 0.05 & 0.05 & 0.03 & 0.03 & 0.02 \\
    EquiCon-PSLD     & 0.01 & 0.01 & 0.01 & 0.01 & 0.01 \\
    Equi-ReSample    & 0.03 & 0.05 & 0.02 & 0.02 & 0.05 \\
    EquiCon-ReSample & 0.001 & 0.001 & 0.001 & 0.001 & 0.001 \\
    Equi-DPS         & 0.0001 & 0.001 & 0.001 & 0.001 & 0.1 \\
    
    \midrule
    \multicolumn{6}{l}{\textit{ImageNet $256 \times 256$}} \\
    EquiCon-PSLD     & 0.0015 & 0.05 & 0.06 & 0.07 & 0.001 \\
    
    \bottomrule
  \end{tabular}
\end{table}

\begin{figure}[H]
    \begin{algorithm}[H]
        \caption{Equi-PSLD for Image Restoration Tasks}
        \begin{algorithmic}[1]
        \label{alg:equipsld}
            \Require $T, \bm{y}, \{ \eta_t \}_{t=1}^T, \{ \gamma_t \}_{t=1}^T, \{ \tilde{\sigma}_t \}_{t=1}^T$
            \Require $\mathcal{E}, \mathcal{D},  \mathcal{A} \x_0^\ast, \mathcal{A}, \bm{s}_\theta, $\textcolor{equi}{\ $T_g$ and $S_g$, $\{\lambda_t\}_{t=1}^T$}
            \State $\z_T \sim \mathcal{N}(\bm{0}, \bm{I})$
            \For {$t = T-1$ \textbf{to} $0$}
                \State $\hat{\bm{s}} \gets \bm{s}_\theta(\z_t, t)$
                \State ${\z}_{0|t} \gets \frac{1}{\sqrt{\Bar{\alpha_t}}}(\z_t + (1 - \Bar{\alpha_t})\hat{\bm{s}})$
                \State $\bm{\epsilon} \sim \mathcal{N}(\bm{0}, \bm{I})$
                \State $\z'_{t-1} \gets \frac{\sqrt{\alpha_t}(1-\Bar{\alpha}_{t-1})}{1-\Bar{\alpha}_t} \z_t + \frac{\sqrt{\Bar{\alpha}_{t-1}}\beta_t}{1-\Bar{\alpha}_t} {\z}_{0|t} + \Tilde{\sigma}_t \bm{\epsilon}$
                \State $\z''_{t-1} \gets \z'_{t-1} - \eta_t \nabla_{\z_t} \|\bm{y} - \mathcal{A}(\mathcal{D}({\z}_{0|t}))\|^2_2$
                \State $\z_{t-1} \gets \z''_{t-1} - \gamma_t \nabla_{\z_t} \|{\z}_{0|t} - \mathcal{E}(\mathcal{A}^T \mathcal{A} \x_0^\ast + (\bm{I} - \mathcal{A}^T \mathcal{A})\mathcal{D}({\z}_{0|t}))\|^2_2$
                \textcolor{equi}{
                \State $\z_{t-1} \gets \z_{t-1} - \lambda_t \nabla_{\z_t} \| S_g (\mathcal{D}({\z}_{0|t})) - \mathcal{D}(T_g({\z}_{0|t}))\|_2^2$}
            \EndFor
            \State \Return $\mathcal{D}({\z}_{0|t})$
        \end{algorithmic}
    \end{algorithm}
\end{figure}

\textbf{PSLD.}\quad We integrate EquiReg into PSLD by simply adding an additional gradient update step using our regularization term (\Cref{alg:equipsld,alg:equiconpsld}).

In our experiments, we use the official PSLD implementation from \citet{rout2023solving}, running with its default settings to reproduce the baseline results. We note that in our code, we do not square the norm when computing the gradient, aligning with PSLD's implementation.

\begin{figure}[H]
    \begin{algorithm}[H]
        \caption{EquiCon-PSLD for Image Restoration Tasks}
        \begin{algorithmic}[1]
        \label{alg:equiconpsld}
            \Require $T, \bm{y}, \{ \eta_t \}_{t=1}^T, \{ \gamma_t \}_{t=1}^T, \{ \tilde{\sigma}_t \}_{t=1}^T$
            \Require $\mathcal{E}, \mathcal{D},  \mathcal{A} \x_0^\ast, \mathcal{A}, \bm{s}_\theta, $\textcolor{equi}{\ $T_g$ and $S_g$, $\{\lambda_t\}_{t=1}^T$}
            \State $\z_T \sim \mathcal{N}(\bm{0}, \bm{I})$
            \For {$t = T-1$ \textbf{to} $0$}
                \State $\hat{\bm{s}} \gets \bm{s}_\theta(\z_t, t)$
                \State ${\z}_{0|t} \gets \frac{1}{\sqrt{\Bar{\alpha_t}}}(\z_t + (1 - \Bar{\alpha_t})\hat{\bm{s}})$
                \State $\bm{\epsilon} \sim \mathcal{N}(\bm{0}, \bm{I})$
                \State $\z'_{t-1} \gets \frac{\sqrt{\alpha_t}(1-\Bar{\alpha}_{t-1})}{1-\Bar{\alpha}_t} \z_t + \frac{\sqrt{\Bar{\alpha}_{t-1}}\beta_t}{1-\Bar{\alpha}_t} {\z}_{0|t} + \Tilde{\sigma}_t \bm{\epsilon}$
                \State $\z''_{t-1} \gets \z'_{t-1} - \eta_t \nabla_{\z_t} \|\bm{y} - \mathcal{A}(\mathcal{D}({\z}_{0|t}))\|^2_2$
                \State $\z_{t-1} \gets \z''_{t-1} - \gamma_t \nabla_{\z_t} \|{\z}_{0|t} - \mathcal{E}(\mathcal{A}^T \mathcal{A} \x_0^\ast + (\bm{I} - \mathcal{A}^T \mathcal{A})\mathcal{D}({\z}_{0|t}))\|^2_2$
                \textcolor{equi}{
                \State $\z_{t-1} \gets \z_{t-1} - \lambda_t \nabla_{\z_t} \|  {\z}_{0|t} - \mathcal{E}(S_g^{-1}(\mathcal{D}(T_g({\z}_{0|t}))))\|_2^2$}
            \EndFor
            \State \Return $\mathcal{D}({\z}_{0|t})$
        \end{algorithmic}
    \end{algorithm}
\end{figure}

\textbf{ReSample.}\quad  We integrate EquiReg into ReSample by adding our regularization term into the hard data consistency step (\Cref{alg:equiresample,alg:equiconresample}). We note that the ReSample algorithm employs a two-stage approach; initially, it performs pixel-space optimization, and later it performs latent-space optimization. We apply EquiReg in the latent-space optimization stage.

In our experiments, we use the official ReSample implementation from \citet{song2023solving}, running with its default settings to reproduce the baseline results.

\begin{algorithm}[H]
    \caption{Equi-ReSample for Image Restoration Tasks}
    \begin{algorithmic}[1]
        \label{alg:equiresample}
        \Require Measurements $\y$, $\mathcal{A}(\cdot)$, Encoder $\mathcal{E}(\cdot)$, Decoder $\mathcal{D}(\cdot)$, Score function $\s_\theta(\cdot, t)$, Pretrained LDM Parameters $\beta_t$, $\bar{\alpha}_t$, $\eta$, $\delta$, Hyperparameter $\gamma$ to control $\sigma_t^2$, Time steps to perform resample $C$, \textcolor{equi}{\ $T_g$ and $S_g$, $\{\lambda_t\}_{t=1}^T$}
        \State $\z_T \sim \mathcal{N}(\mathbf{0}, \boldsymbol{I})$ \Comment{Initial noise vector}
        \For{$t = T - 1, \ldots, 0$}
            \State $\boldsymbol{\epsilon}_1 \sim \mathcal{N}(\mathbf{0}, \boldsymbol{I})$
            \State $\hat{\boldsymbol{\epsilon}}_{t+1} = \s_\theta(\z_{t+1}, t + 1)$ \Comment{Compute the score}
            \State $\hat{\z}_0(\z_{t+1}) = \frac{1}{\sqrt{\bar{\alpha}_{t+1}}}(\z_{t+1} - \sqrt{1 - \bar{\alpha}_{t+1}}\hat{\boldsymbol{\epsilon}}_{t+1})$ \Comment{Predict $\hat{\z}_0$ using Tweedie's formula}
            \State $\z'_t = \sqrt{\bar{\alpha}_t}\hat{\z}_0(\z_{t+1}) + \sqrt{1 - \bar{\alpha}_t - \eta\delta^2}\hat{\boldsymbol{\epsilon}}_{t+1} + \eta\delta\boldsymbol{\epsilon}_1$ \Comment{Unconditional DDIM step}
            \If{$t \in C$} \Comment{ReSample time step}
                \State {\color{equi} Initialize $\hat{\z}_0(\y)$ with $\hat{\z}_0(\z_{t+1})$}
                \For{{\color{equi} each step in gradient descent}}
                    \State {\color{equi} $\bm{g} \gets \nabla_{\hat{\z}_0(\y)} \frac{1}{2}\|\y - \mathcal{A}(\mathcal{D}(\hat{\z}_0(\y)))\|_2^2 + \lambda_t \nabla_{\hat{\z}_0(\y)} \| S_g (\mathcal{D}(\hat{\z}_0(\y))) - \mathcal{D}(T_g(\hat{\z}_0(\y)))\|_2^2$}
                    \State {\color{equi} Update $\hat{\z}_0(\y)$ using gradient $\bm{g}$}
                \EndFor
                \State $\z_t = \text{StochasticResample}(\hat{\z}_0(\y), \z'_t, \gamma)$ \Comment{Map back to $t$}
            \Else
                \State $\z_t = \z'_t$ \Comment{Unconditional sampling if not resampling}
            \EndIf
        \EndFor
        \State $\x_0 = \mathcal{D}(\z_0)$ \Comment{Output reconstructed image}
        \State \Return $\x_0$
    \end{algorithmic}
\end{algorithm}

\begin{algorithm}[H]
    \caption{EquiCon-ReSample for Image Restoration Tasks}
    \begin{algorithmic}[1]
        \label{alg:equiconresample}
        \Require Measurements $\y$, $\mathcal{A}(\cdot)$, Encoder $\mathcal{E}(\cdot)$, Decoder $\mathcal{D}(\cdot)$, Score function $\s_\theta(\cdot, t)$, Pretrained LDM Parameters $\beta_t$, $\bar{\alpha}_t$, $\eta$, $\delta$, Hyperparameter $\gamma$ to control $\sigma_t^2$, Time steps to perform resample $C$, \textcolor{equi}{\ $T_g$ and $S_g$, $\{\lambda_t\}_{t=1}^T$}
        \State $\z_T \sim \mathcal{N}(\mathbf{0}, \boldsymbol{I})$ \Comment{Initial noise vector}
        \For{$t = T - 1, \ldots, 0$}
            \State $\boldsymbol{\epsilon}_1 \sim \mathcal{N}(\mathbf{0}, \boldsymbol{I})$
            \State $\hat{\boldsymbol{\epsilon}}_{t+1} = \s_\theta(\z_{t+1}, t + 1)$ \Comment{Compute the score}
            \State $\hat{\z}_0(\z_{t+1}) = \frac{1}{\sqrt{\bar{\alpha}_{t+1}}}(\z_{t+1} - \sqrt{1 - \bar{\alpha}_{t+1}}\hat{\boldsymbol{\epsilon}}_{t+1})$ \Comment{Predict $\hat{\z}_0$ using Tweedie's formula}
            \State $\z'_t = \sqrt{\bar{\alpha}_t}\hat{\z}_0(\z_{t+1}) + \sqrt{1 - \bar{\alpha}_t - \eta\delta^2}\hat{\boldsymbol{\epsilon}}_{t+1} + \eta\delta\boldsymbol{\epsilon}_1$ \Comment{Unconditional DDIM step}
            \If{$t \in C$} \Comment{ReSample time step}
                \State {\color{equi} Initialize $\hat{\z}_0(\y)$ with $\hat{\z}_0(\z_{t+1})$}
                \For{{\color{equi} each step in gradient descent}}
                    \State {\color{equi} $\bm{g} \gets \nabla_{\hat{\z}_0(\y)} \frac{1}{2}\|\y - \mathcal{A}(\mathcal{D}(\hat{\z}_0(\y)))\|_2^2 + \lambda_t \nabla_{\hat{\z}_0(\y)} \| \hat{\z}_0(\y) - \mathcal{E}(S_g^{-1}(\mathcal{D}(T_g(\hat{\z}_0(\y)))))\|_2^2$}
                    \State {\color{equi} Update $\hat{\z}_0(\y)$ using gradient $\bm{g}$}
                \EndFor
                \State $\z_t = \text{StochasticResample}(\hat{\z}_0(\y), \z'_t, \gamma)$ \Comment{Map back to $t$}
            \Else
                \State $\z_t = \z'_t$ \Comment{Unconditional sampling if not resampling}
            \EndIf
        \EndFor
        \State $\x_0 = \mathcal{D}(\z_0)$ \Comment{Output reconstructed image}
        \State \Return $\x_0$
    \end{algorithmic}
\end{algorithm}

\textbf{DPS.}\quad Similar to PSLD, we integrate EquiReg into DPS by simply adding an additional gradient update step using our regularization term (\Cref{alg:equidps_app}).

In our experiments, we use the official DPS implementation from \citet{chung2023diffusion}, running with its default settings to reproduce the baseline results.

\begin{figure}[H]
    \begin{algorithm}[H]
        \caption{Equi-DPS for Image Restoration Tasks}
        \begin{algorithmic}[1]
            \label{alg:equidps_app}
            \Require $T, \bm{y}, \{\zeta_t\}_{t=1}^T, \{\tilde{\sigma}_t\}_{t=1}^T, \bm{s}_\theta, $\textcolor{equi}{\ $\mathcal{E}$, $T_g$ and $S_g$, $\{\lambda_t\}_{t=1}^T$}
            \State $\x_T \sim \mathcal{N}(\bm{0}, \bm{I})$
            \For {$t = T-1$ \textbf{to} $0$}
                \State $\hat{\bm{s}} \gets \bm{s}_\theta(\x_t, t)$
                \State ${\x}_{0|t} \gets \frac{1}{\sqrt{\Bar{\alpha_t}}}(\x_t + (1 - \Bar{\alpha_t})\hat{\bm{s}})$
                \State $\bm\epsilon \sim \mathcal{N}(\bm{0}, \bm{I})$
                \State $\x'_{t-1} \gets \frac{\sqrt{\alpha_t}(1-\Bar{\alpha}_{t-1})}{1-\Bar{\alpha}_t} \x_t + \frac{\sqrt{\Bar{\alpha}_{t-1}}\beta_t}{1-\Bar{\alpha}_t} {\x}_{0|t} + \tilde{\sigma}_t \bm\epsilon$
                \State $\x_{t-1} \gets \x'_{t-1} - \zeta_t \nabla_{\x_t} \|\bm{y} - \mathcal{A}({\x}_{0|t})\|^2_2$
                \textcolor{equi} {
                \State $\x_{t-1} \gets \x_{t-1} - \lambda_t \nabla_{\x_t} \| S_g (\mathcal{E}({\x}_{0|t})) - \mathcal{E}(T_g({\x}_{0|t}))\|_2^2$}
            \EndFor
            \State \Return ${\x}_{0}$
        \end{algorithmic}
    \end{algorithm}
\end{figure}

\textbf{SITCOM.}\quad We augment the original SITCOM algorithm by introducing an additional equivariant refinement stage at each reverse diffusion step. After completing the standard measurement and backward-consistency gradient updates, we perform a second optimization over the equivariance loss, enforcing consistency between $\mathcal{E}(T_g(v))$ and $T_g(\mathcal{E}(v))$ (\Cref{alg:equisitcom}).

In our experiments, we use the official SITCOM implementation from \citet{alkhouri2025sitcom}, running with its default settings to reproduce the baseline results.
\begin{algorithm}[H]
  \caption{Equi-SITCOM for Image Restoration Tasks}
  \label{alg:equisitcom}
  \begin{algorithmic}[1]
    \Require Measurements $\mathbf{y}$, forward operator $\mathcal{A}(\cdot)$, pre-trained DM $\epsilon_\theta(\cdot,\cdot)$, diffusion steps $N$, schedule $\bar{\alpha}_i$, measurement gradient steps $K$, \textcolor{equi}{equivariant gradient steps $K_{\text{equi}}$}, stop $\delta$, lr $\gamma$, reg. $\lambda$.
    \Ensure Restored image $\hat{\mathbf{x}}$.
    \State \textbf{Initialize} $\mathbf{x}_N \sim \mathcal{N}(\mathbf{0},\mathbf{I})$, $\Delta t=\left\lfloor \tfrac{T}{N} \right\rfloor$.
    \For{$i = N, N-1, \ldots, 1$} \Comment{Reducing diffusion sampling steps}
      \State $\mathbf{v}_i^{(0)} \gets \mathbf{x}_i$ \Comment{Init for closeness (C3)}
      \For{$k = 1, \ldots, K$} \Comment{Adam on measurement/backward consistency (C1, C2)}
        \State $\mathbf{v}_i^{(k)} \gets \mathbf{v}_i^{(k-1)} - \gamma \nabla_{\mathbf{v}_i}\!\left[
          \left\| \mathcal{A}\!\left( \tfrac{1}{\sqrt{\bar{\alpha}_i}}
          \big(\mathbf{v}_i - \sqrt{1-\bar{\alpha}_i}\,\epsilon_\theta(\mathbf{v}_i, i\Delta t)\big)\right)-\mathbf{y}\right\|_2^2
          + \lambda \|\mathbf{x}_i - \mathbf{v}_i\|_2^2 \right] \Big|_{\mathbf{v}_i=\mathbf{v}_i^{(k-1)}}$
        \If{$\left\| \mathcal{A}\!\left( \tfrac{1}{\sqrt{\bar{\alpha}_i}}
          \big(\mathbf{v}_i^{(k)} - \sqrt{1-\bar{\alpha}_i}\,\epsilon_\theta(\mathbf{v}_i^{(k)}, i\Delta t)\big)\right)-\mathbf{y}\right\|_2^2 < \delta^2$}
          \State \textbf{break} \Comment{Prevent noise overfitting}
        \EndIf
      \EndFor
      \begingroup\color{equi}
      \State $\mathbf{v}_i^{(0)} \gets \mathbf{v}_i^{(k)}$ \Comment{Initialize to optimized $\mathbf{v}_i$}
      \For{$k = 1, \ldots, K_{\text{equi}}$}
        \State $\mathbf{v}_i^{(k)} \gets \mathbf{v}_i^{(k-1)} - \gamma \nabla_{\mathbf{v}_i}\!\left[
          \left\| \mathcal{E}\!\big(T_g(\mathbf{v}_i^{(k)})\big) - T_g\!\big(\mathcal{E}(\mathbf{v}_i^{(k)})\big) \right\|_2^2
          \right] \Big|_{\mathbf{v}_i=\mathbf{v}_i^{(k-1)}}$
        \If{$\left\| \mathcal{E}\!\big(T_g(\mathbf{v}_i^{(k)})\big) - T_g\!\big(\mathcal{E}(\mathbf{v}_i^{(k)})\big)\right\|_2^2 < \delta^2$}
          \State \textbf{break}
        \EndIf
      \EndFor
      \endgroup
      \State $\hat{\mathbf{v}}_i \gets \mathbf{v}_i^{(k)}$ \Comment{Backward diffusion consistency (C2)}
      \State $\hat{\mathbf{x}}'_0 \gets \tfrac{1}{\sqrt{\bar{\alpha}_i}}\!\left[ \hat{\mathbf{v}}_i - \sqrt{1-\bar{\alpha}_i}\,\epsilon_\theta(\hat{\mathbf{v}}_i, i\Delta t) \right]$ \Comment{Backward consistency (C2)}
      \State $\mathbf{x}_{i-1} \gets \sqrt{\bar{\alpha}_{i-1}}\,\hat{\mathbf{x}}'_0 + \sqrt{1-\bar{\alpha}_{i-1}}\,\boldsymbol{\eta}_i$, \ \ $\boldsymbol{\eta}_i \sim \mathcal{N}(\mathbf{0},\mathbf{I})$ \Comment{Forward consistency (C3)}
    \EndFor
    \State \Return $\hat{\mathbf{x}} = \mathbf{x}_0$
  \end{algorithmic}
\end{algorithm}

\section{Experiment Setup for PDE Reconstructions}
\label{app:pde}

\textbf{Helmholtz equation.}\quad The Helmholtz equation represents wave propagation in heterogeneous media:
\begin{equation}
    \nabla^2 u(x) + k^2u(x) = a(x), \quad x \in (0,1)^2,
\label{eq:helmholtz}
\end{equation}
where $\nabla^2$ denotes the Laplacian operator, with $k=1$ and $u|_{\partial\Omega} = 0$. Coefficient fields $a(x)$ are generated according to $a \sim \mathcal{N}(0, (-\Delta + 9\mathbf{I})^{2})$. The forward problem is to predict the solution $u$ from sparse observations of the coefficient field $a$; the inverse problem is to predict $a$ from sparse observations of $u$. We note that this system has reflection equivariance along $x_1=\frac{1}{2},x_2=\frac{1}{2},x_1=x_2$ and rotation equivariance by $\frac{\pi}{2},\pi,\frac{3\pi}{2}$.

\textbf{Navier-Stokes equations.}\quad Following the methodology of \citep{li2020fno}, we model the time evolution of a vorticity field, $u(x,t)$, governed by:
\begin{align}
    \partial_t u(x,t) + \wb(x,t) \cdot \nabla u(x,t) &= \nu \Delta u(x,t) + f(x), \quad x \in (0,1)^2, \, t \in (0,T], \\
    \nabla \cdot \wb(x,t) &= 0, \quad x \in (0,1)^2, \, t \in [0,T], \\
    u(x,0) &= a(x), \quad x \in (0,1)^2,
\end{align}
where $\wb$ is the velocity field; $\nu=\frac{1}{1000}$, viscosity; and $f$, a fixed forcing term. The initial condition $a(x)$ is drawn from $\mathcal{N}(0, 7^{3/2}(-\Delta + 49\mathbf{I})^{-5/2})$ under periodic boundary conditions. The forcing term is $f(x) = 0.1 \left(\sin(2\pi(x_1+x_2)) + \cos(2\pi(x_1+x_2))\right)$. We borrow the dataset from \citep{huang2024diffusionpde}. The forward problem is to predict the final-time vorticity field $u(\cdot, T)$ from sparse observations of the initial condition $a$; the inverse problem is to predict $a$ from sparse observations of $u(\cdot, T)$. We note that this system has a reflection symmetry along the $x_1=x_2$ axis.

\textbf{Implementation details.}\quad
EquiReg, as a regularizer for diffusion posterior sampling, can be adapted to many inverse solvers in a plug-in manner. For PDE experiments, we use the same model weights and configurations as FunDPS~\citep{yao2025guideddiffusionsamplingfunction}. Error rates are calculated using the $L^2$ relative error between the predicted and true solutions, averaged on 100 randomly selected test samples. We provide the information on the EquiReg scaling weights in~\Cref{tab:fundps-hyperparams}.

\begin{table*}[ht]
  \centering
  \setlength{\tabcolsep}{0.5em}
  \caption{\textbf{EquiReg loss used in PDE experiments.}}
  \label{tab:fundps-hyperparams}
  \resizebox{0.6\textwidth}{!}{
  \begin{tabular}{lcccc}
  \toprule
   & \multicolumn{2}{c}{Helmholtz}
   & \multicolumn{2}{c}{Navier-Stokes} \\
  \cmidrule(lr){2-3}
  \cmidrule(lr){4-5}
   & Forward & Inverse & Forward & Inverse \\
  \midrule
  EquiReg Norm Type
   & MSE & L2
   & MSE & L2 \\
  EquiReg Weight $\lambda$
   & 100 & 100
   & 100 & 1000 \\
  \bottomrule
  \end{tabular}%
  }
\end{table*}

\section{A Wasserstein-flow perspective on posterior sampling}\label{app:theory}

\subsection{Summary}

We offer a Wasserstein-gradient-flow perspective on conditional diffusion sampling. \Cref{prop:one} identifies the specific functional whose Wasserstein-$2$ gradient flow recovers the ideal reverse-conditional dynamics. This perspective motivates the use of manifold-preference regularization in practical posterior sampling; no exact Wasserstein-gradient-flow characterization of the regularized EquiReg dynamics is claimed.

\begin{proposition}\label{prop:one}
    Let $\rho(\x,t)$ be the distribution of $\x_{T-t}$ driven by the ideal reverse dynamics (\cref{eqn:ideal_reverse}). Then, the evolution of $\rho$ follows the Wasserstein-2 gradient flow associated with minimizing functional $\Phi(\rho,t)$ defined as $\beta_{T-t}\int[ \rho\phi(\x,t)+\frac 1 2\rho\log\rho ]d\x$, where $\phi(\x,t)=-(\log p_{T-t}(\x|\y)+\frac 1 4 \|\x\|^2)$.
\end{proposition}

\subsection{Preliminary and Notations}
We first remind the readers of gradient flow under the Wasserstein-2 metric and introduce the notations related to the diffusion model.
\paragraph{Wasserstein Gradient Flow}  
Let \(\mathcal{F} : \mathcal{P}_2(\mathbb{R}^d) \to \mathbb{R} \cup \{+\infty\}\) be a functional of probability distributions. The Wasserstein gradient flow of \(\mathcal{F}\) is characterized by the minimizing movement scheme (also known as JKO scheme) introduced by \citep{jordan1998variational}. For a fixed time step \(\tau > 0\), the sequence \((\rho_k)_{k \in \mathbb{N}}\) of probability densities is defined recursively by:
\[
\rho_{k+1} \in \arg\min_{\rho \in \mathcal{P}_2(\mathbb{R}^d)} \left\{ \frac{1}{2\tau} W_2^2(\rho, \rho_k) + \mathcal{F}(\rho) \right\},
\]
where \(W_2\) denotes the 2-Wasserstein distance, and each \(\rho_k\) is a probability density representing the distribution at time \(t = k\tau\). In the limit \(\tau \to 0\), this discrete-time scheme recovers the continuous-time gradient flow of \(\mathcal{F}\) under the \(W_2\) metric.

\paragraph{Diffusion Model}  
A diffusion model defines a forward stochastic process \((\x_t)_{t \in [0, T]}\) governed by the Itô SDE:
\begin{equation}
\mathrm{d}\x_t = f(\x_t, t)\,\mathrm{d}t + \sqrt{\beta_t}\,\mathrm{d}\bm{w}_t,
\end{equation}
where \(\bm{w}_t\) is standard Brownian motion, \(\beta_t > 0\) is a time-dependent variance schedule, and \(f(\x, t)\) is a drift term.
For instance, $f\equiv 0$ for a variance-exploding SDE and $f(\x,t)=-\frac{\beta_t}2 \x$ for a variance-preserving SDE defined in \citep{song2021score}. In this work, we carry out our analysis under a more general setting.

\begin{assumption}\label{asumpt}
    The drift term is a gradient field, $f(\x,t)=\nabla h(\x,t)$ for a scalar function $h$.
\end{assumption}

 This process progressively transforms an initial data distribution \(\x_0 \sim p_0\) into a tractable reference distribution (e.g., approximately a Gaussian $\mathcal{N}(0,I)$) at time \(T\).

Sampling is performed by simulating the \emph{reverse-time SDE}:
\begin{equation}
\mathrm{d}\x_t = \left[f(\x_t, t) - \beta_t \nabla_\x \log p_t(\x_t) \right] \mathrm{d}t + \sqrt{\beta_t}\,\mathrm{d}\bm{\bar{w}}_t,
\end{equation}
where \(p_t\) is the marginal density of \(\x_t\), and \(\bm{\bar{w}}_t\) is a standard Brownian motion in reverse time.

In practice, the score function \(\nabla_\x \log p_t(\x)\) is approximated by a neural network \(s_\theta(\x, t)\) trained to estimate the score of the forward process. For \emph{conditional sampling}, where we sample \(\x_0\) given some observed variable \(y\), the score is replaced by \(\nabla_\x \log p_t(\x|\y)\) and decomposed as
\begin{equation}
\nabla_\x \log p_t(\x|\y) = \nabla_\x \log p_t(\x) + \nabla_\x \log p_t(\y|\x),
\end{equation}
based on Bayes' rule.

To simplify notation in the sequel, we perform a time reparameterization \(t = T - t'\), so that the reverse process is written as a forward SDE over \(t \in [0, T]\):
\begin{equation}\label{ideal_apdx_sde}
\mathrm{d}\x_t = -\left[f(\x_t, T - t) - \beta_{T - t} [\nabla_\x \log p_{T - t}(\x_t) + \nabla_\x \log p_t(y|\x_t)]\right] \mathrm{d}t + \sqrt{\beta_{T - t}}\,\mathrm{d}\bm{w}_t,
\end{equation}
This form describes the generative process as evolving forward from \(t = 0\) to \(t = T\), matching the usual direction of analysis in gradient flow frameworks.

\subsection{Proof of \Cref{prop:one}}
In this work, we consider Wasserstein gradient flow under the setting where the functional $\mathcal{F}$ depends on time.

\begin{lemma}\label{lemma}
Consider a time-dependent functional $\mathcal{F}(\rho,t)=\int\rho(\x)V(\x,t)\mathrm{d}x+\int \alpha(t)\rho\log\rho \mathrm{d}\x$. Then the particle description of Wasserstein-2 gradient flow associated with this functional derived by JKO scheme is
\begin{equation}
    \mathrm{d}\x_t=-\nabla V(\x_t,t)\mathrm{d}t+{\sqrt{2\alpha(t)}}\mathrm{d}\w_t.
\end{equation}
\end{lemma}
\begin{proof}
    Consider the following optimization
    \begin{equation}\label{eq_opt}
        \min_{\rho'}\mathcal{F}(\rho',t+\Delta t)-\mathcal{F}(\rho,t)+\frac {1}{2\Delta t}W_2^2(\rho,\rho'),
    \end{equation}
    where the change of density is restricted to the Liouville equation
    \begin{equation}
        \partial_t\rho=-\nabla\cdot(\rho v(\x,t)),\ \text{and}\ \rho'(x)=\rho(x)-\Delta t\nabla\cdot(\rho(\x )v(\x))+o(\Delta t).
    \end{equation}

Using the static formulation of $W_2$ distance, we have
\begin{equation}
    W_2^2(\rho, \rho') = \int \rho(\x)\|\x - T^*(\x)\|^2\,\mathrm{d}\x = \Delta t^2 \int \rho(\x)\|v^*(\x)\|^2\,\mathrm{d}\x,
\end{equation}
where $T^*(\x)$ is the optimal transport map, and $v^*(\x)$ is the associated optimal velocity field.

Thus, we can rewrite the \cref{eq_opt} as
\begin{align}\label{eq18}
    \inf_{v} \ \mathcal{F}(\rho,t) &- \Delta t \int \nabla \cdot (\rho(\x)v(\x))\, \frac{\delta \mathcal{F}(\rho,t)}{\delta \rho}(\x)\,\mathrm{d}\x +\Delta t\int \big[\rho(\x)\partial_t V(\x,t)+\dot{\alpha}(t)\rho\log\rho \big]\mathrm{d}\x\\
    &- \mathcal{F}(\rho,t) + \frac{\Delta t}{2} \int \rho(\x)\|v(\x)\|^2\,\mathrm{d}\x,
\end{align}
which simplifies to
\begin{equation}
    \min_{v} \ \int \rho(\x) \left\langle v(\x), \nabla \frac{\delta \mathcal{F}(\rho,t)}{\delta \rho}(\x) \right\rangle\,\mathrm{d}\x + \frac{1}{2} \int \rho(\x)\|v(\x)\|^2\,\mathrm{d}\x,
\end{equation}
since the last term in the first line of (\ref{eq18}) does not depend on $v$.
and further to
\begin{equation}
    \min_{v} \ \int \rho(\x) \left\| v(\x) + \nabla \frac{\delta \mathcal{F}(\rho,t)}{\delta \rho}(\x) \right\|^2\,\mathrm{d}\x.
\end{equation}

From the optimality condition of the above problem, we obtain
\begin{equation}
    v(\x,t) = - \nabla \frac{\delta \mathcal{F}(\rho,t)}{\delta \rho}(\x)=-(\nabla V(\x,t)+\alpha(t)\nabla \log \rho(\x,t)).
\end{equation}
We note that By Hörmander's theorem, a smooth density $\rho(\x,t)$ exists for $t > 0$, ensuring that the above $v$ is well-defined. The corresponding evolution of probability density is
\begin{align}
    \partial_t\rho(\x,t)&=-\nabla\cdot(\rho(\x,t)v(\x,t))\\
    &=\nabla\cdot(\rho(\x,t)(\nabla V(\x,t)+\alpha(t)\frac{\nabla\rho(\x,t)}{\rho}))\\
    &=-\nabla\cdot(\rho(\x,t)(-\nabla V(\x,t))+\alpha(t)\Delta\rho(\x,t)),
\end{align}
which is exactly the Fokker-Planck equation describing the evolution of the probability density describing the particles following
\begin{equation}\label{eq27}
    \mathrm{d}\x_t=-\nabla V(\x_t,t)\mathrm{d}t+{\sqrt{2\alpha(t)}}\mathrm{d}\w_t.
\end{equation}

\end{proof}

Now we come back to \Cref{prop:one}.
From \cref{ideal_apdx_sde} we know that choosing
\begin{equation}\label{valpha}
    V(\x,t)=h(\x,T-t)- \beta_{T - t} [\log p_{T - t}(\x) +  \log p_{T-t}(\y|\x)]\ \ \text{and}\ \alpha(t)=\frac {\beta_{T-t}}2
\end{equation}
in \Cref{lemma} completes the proof, where $h$ is defined in \Cref{asumpt}.

\section{Additional Background Information}\label{app:add_back}

\textbf{Solving inverse problems with deep learning prior to diffusion models.}\quad Earlier works~\citep{metzler2016denoising, romano2017little, zhang2017learning, metzler2017learned} used deep neural networks as denoisers to solve inverse problems. Furthermore, deep generative models such as variational autoencoders (VAEs)~\citep{kingma2013auto}, and generative adversarial networks (GANs)~\citep{goodfellow2014generative} were employed. Notable applications include compressed sensing~\citep{bora2017compressed} and MRI~\citep{jalal2021robust}.

\textbf{Applications on diffusion models to solve inverse problems.}\quad Most popular applications include image restoration~\citep{chung2023diffusion, chung2022improving, kawar2022denoising, lugmayr2022repaint, saharia2022image, song2023solving, rout2023solving, zhu2023denoising, zhang2025daps, zirvi2025diffusion}, medical imaging~\citep{song2022solving, chung2022score, chung2022mr, hung2023med, dorjsembe2024conditional, li2024rethinking, kazerouni2023diffusion, bian2024diffusion}, and solving partial differential equations (PDEs)~\citep{isakov2006inverse, huang2024diffusionpde,shysheya2024conditional,liu2023genphys,li2025generative,baldassari2023conditional,mammadovdiffusion,yao2025guideddiffusionsamplingfunction}. On the methodology side, there has been numerous advancements~\citep{chung2023diffusion, chung2022improving, kawar2022denoising, lugmayr2022repaint, saharia2022image, song2023solving, rout2023solving, zhu2023denoising, zhang2025daps, zirvi2025diffusion, song2022solving, chung2022score, chung2022mr, hung2023med, dorjsembe2024conditional, li2024rethinking, kazerouni2023diffusion, bian2024diffusion, huang2024diffusionpde,shysheya2024conditional,mammadov2024amortizedposteriorsamplingdiffusion, cardoso2024monte}.

\textbf{Resources for~\Cref{def:equierror_distdept_con} on vanishing-error autoencoders.}\quad Manifold constrained distribution-dependent equivariance error uses the notion of \emph{vanishing-error autoencoders}~\citep{shao2018riemannian, anders2020fairwashing, he2024manifold} (\Cref{def:vanishea}), also known as an asymptotically-trained autoencoder~\citep{anders2020fairwashing} or a perfect autoencoder~\citep{he2024manifold}. Vanishing-error autoencoders have previously been employed by diffusion-based inverse solvers to preserve the diffusion process on the manifold~\citep{he2024manifold}.
\begin{definition}[Vanishing-Error Autoencoder]\label{def:vanishea}
A vanishing-error autoencoder under the manifold $\mathcal{M}$ with encoder $\mathcal{E}: \mathcal{X} \rightarrow \mathcal{Z}$ and decoder $\mathcal{D}: \mathcal{Z} \rightarrow \mathcal{X}$ with $\mathcal{Z} = \R^{k}$ where $k < d$, has zero reconstruction error under the support of the data distribution $\mathcal{X}$, i.e., $\forall \x \in \mathcal{X} \subset \mathcal{M}$, $\x = \mathcal{D}(\mathcal{E}(\x))$. It follows that the decoder is surjective on the data manifold, $\mathcal{D}: \mathcal{Z} \rightarrow \mathcal{M}$~\citep{he2024manifold}, and the encoder-decoder composition forms an identity map, i.e., $\forall \z \in \mathcal{M}, \z = \mathcal{E}(\mathcal{D}(\z))$.
\end{definition}

Using the vanishing-error autoencoder, we define a manifold-constrained variant of the distribution-dependent equivariance error of~\Cref{def:equierror_distdept}.
\begin{definition}[Manifold-Constrained Distribution-Dependent Equivariant Functions]\label{def:equierror_distdept_con}
Let $G$ act on $\mathcal Z$ via $T_g:\mathcal Z\to\mathcal Z$ and on $\mathcal X$ via $S_g:\mathcal X\to\mathcal X$. The manifold-constrained equivariance error of the function $f: \mathcal{Z} \rightarrow \mathcal{X}$ under the data distribution $p$ is $\sup_{g} \E_{\z \sim p} [\| \z - h(S_g^{-1}(f(T_g(\z)))) \|]$ where $h:\mathcal{X} \rightarrow \mathcal{Z}$, and the pair $(f, h)$ forms a vanishing-error autoencoder (\Cref{def:vanishea}).
\end{definition}
This is the equivariance error used by the EquiCon family of algorithms in our experiments. The qualifier \emph{manifold-constrained} reflects that, when $(f, h)$ is a vanishing-error autoencoder, the composition $h \circ S_g^{-1} \circ f \circ T_g$ remains on the manifold, and the loss measures how far this constrained reconstruction lies from the original $\z$ under the group action. The vanishing-error property is essential for this interpretation; without it, the composition is not guaranteed to return to the manifold and the loss loses its meaning as a manifold-distance proxy.

\textbf{Equivariance.}\quad Let $\z \in \R^{d}$ and $\x = f(\z) \in \R^{d}$. For rotation and reflection equivariance, the transformations $T_g$ and $S_g$ can be defined by a rotation matrix $\bm{R} \in \R^{d \times d}$; then, a function $f$ with the rotation equivariant property would satisfy $\bm{R} \x = f(\bm{R} \z)$. For translation equivariance, the transformations would be $T_g(\z) = \z + g$ and $S_g(\x) = \x + g$, where $g \in \R^{d}$. Hence, for a translation equivariance function $f$, we would have $\x + g = f(\z + g)$. For the case where the output dimension is larger than the input, $f: \R^{k} \rightarrow \R^{d}$ with $d > k$, translation equivariance can be defined up to a discrete scale, i.e., $T_g(\z) = \z + g$ and $S_g(\x) = T_{sg}(\z)$ where $s = \nicefrac{d}{k}$. The equivariance properties of translation, rotation, and reflections, combined, are referred to as E(3) symmetries. Without reflections, the symmetries form a Euclidean group SE(3)~\citep{thomas2018tensor, fuchs2020se}.

E(3), SE(3), and SO(3) are important symmetry groups in 3D Euclidean space, with well-established applications in physics and chemistry, computer vision, and reinforcement learning~\citep{cohen2016group, thomas2018tensor, hoogeboom2022equivariant, xu2024equivariant, park2025approximate}. Finally, our contributions are complementary to, and can be combined with, the growing literature on meta-learning and automatic symmetry discovery to learn symmetry groups and their actions directly from data~\citep{zhou2021meta, quessard2020learning, dehmamy2021automatic, mohapatra2025symmetry}.

\textbf{Data manifold hypothesis.}\quad Let data $\x \in \mathcal{X} \subset \R^{d}$ be in an ambient space of dimension $d$ with support $\mathcal{X}$ distribution. We assume that data are sampled from a low-dimensional manifold $\mathcal{M}$~\citep{cayton2005algorithms, ma2012manifold} embedded in a high-dimensional space (\Cref{assum:manifoldhyp}). This hypothesis is popular in machine learning~\citep{bordt2023manifold}, and has been studied mathematically in the literature~\citep{narayanan2010sample, debortoli2022convergence}. Moreover, empirical evidence in image processing supports the manifold hypothesis~\citep{weinberger2006unsupervised,fefferman2016testing}, and diffusion-based solvers assume this property~\citep{he2024manifold,chung2022improving,chung2023diffusion}.
\begin{assumption}[Manifold Hypothesis]\label{assum:manifoldhyp}
Let $\x \in \mathcal{X} \subset \R^{d}$ be a data sample. The support $\mathcal{X}$ of the data distribution lies on a $k$ dimensional manifold $\mathcal{M}$ within an ambient space $\R^{d}$ where $k \ll d$.
\end{assumption}

\newpage
\section{Additional figures/tables}\label{app:mpe}

\begin{table}[t]
    \centering
    \caption{DPS superresolution with $\lambda = 0.01$ using different MPE functions.}
    \label{tab:dps_sr_mpe}
    \begin{subtable}{\textwidth}
        \centering
        \caption{FFHQ 256.}
        \label{tab:dps_sr_mpe_ffhq}
        \begin{tabular}{lcccc}
            \toprule
            MPE function & PSNR & SSIM & LPIPS\\
            \midrule
            None & 26.521 (1.863) & 0.757 (0.046) & 0.133 (0.035)\\
            LDM Encoder (FFHQ)        & 26.581 (2.457) & 0.773 (0.044) & 0.120 (0.030) \\ 
            CNN Autoencoder (FFHQ)    & 26.866 (1.943) & 0.771 (0.044) & 0.116 (0.029) \\ 
            Pretrained ResNet50       & 26.873 (1.941) & 0.771 (0.044) & 0.116 (0.029) \\
            Pretrained CLIP           & 26.860 (1.942) & 0.771 (0.044) & 0.116 (0.029) \\ 
            \bottomrule
        \end{tabular}
    \end{subtable}

    \vspace{0.5em}

    \begin{subtable}{\textwidth}
        \centering
        \caption{ImageNet.}
        \label{tab:dps_sr_mpe_imagenet}
        \begin{tabular}{lcccc}
            \toprule
            MPE function & PSNR & SSIM & LPIPS\\
            \midrule
            None & 21.573 (4.149) & 0.612 (0.115) & 0.308 (0.107)\\
            LDM Encoder (ImageNet)        & 22.200 (4.295) & 0.568 (0.146) & 0.384 (0.130)\\ 
            CNN Autoencoder (ImageNet)    & 22.178 (4.294) & 0.568 (0.148) & 0.375 (0.125)\\ 
            Pretrained ResNet50           & 22.176 (4.290) & 0.568 (0.148) & 0.375 (0.125)\\
            Pretrained CLIP               & 22.177 (4.293) & 0.568 (0.148) & 0.376 (0.125)\\ 
            \bottomrule
        \end{tabular}
    \end{subtable}
\end{table}

\begin{figure}[H]
    \centering
    \begin{subfigure}{0.9\textwidth}
        \centering
        \includegraphics[width=0.5\linewidth]{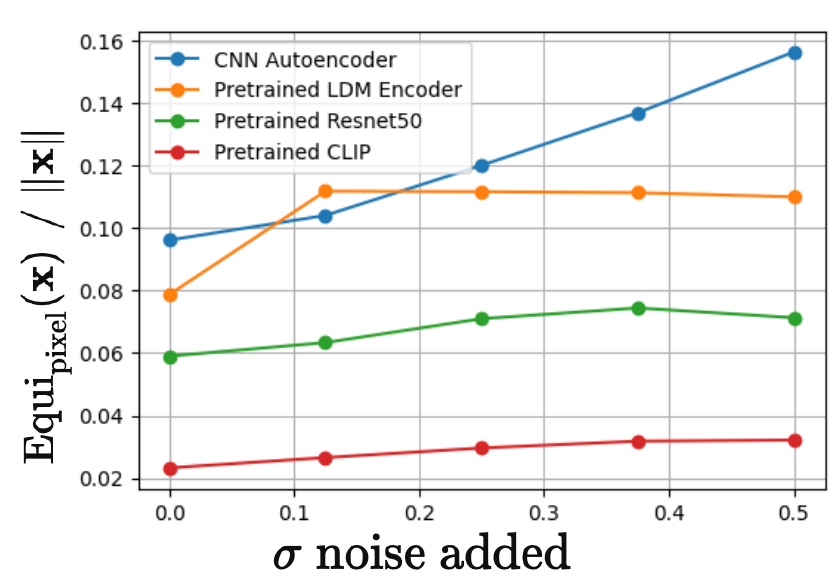}
        \caption{FFHQ 256.}
        \label{fig:mpe_loss_ffhq}
    \end{subfigure}\hfill
    
    \begin{subfigure}{0.9\textwidth}
        \centering
        \includegraphics[width=0.5\linewidth]{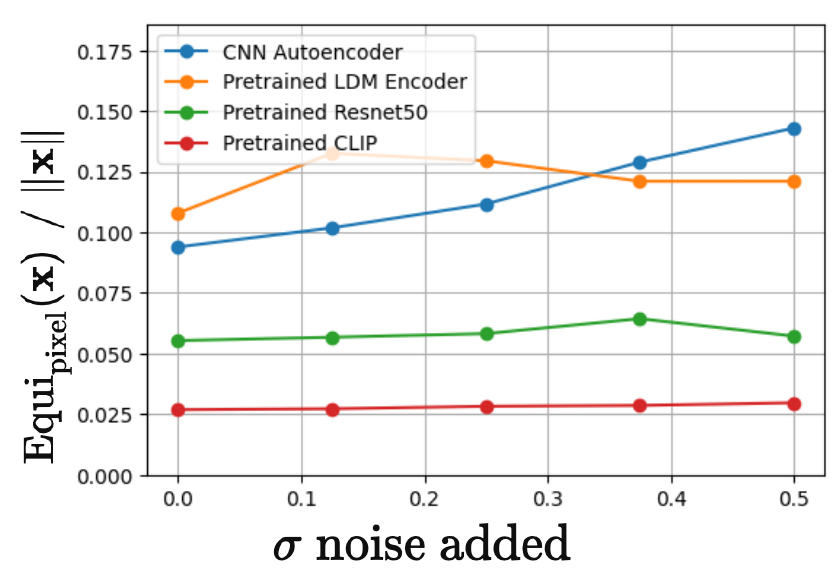}
        \caption{ImageNet.}
        \label{fig:mpe_loss_imagenet}
    \end{subfigure}
    \caption{\textbf{Equivariance error vs.\ $\sigma$ noise added.} As more noise is added, equivariance error, computed with all MPE functions, increases.}
    \label{fig:mpe_loss}
\end{figure}

\begin{figure}[H]
\centering
\vspace{-1mm}
    \includegraphics[width=0.7\linewidth]{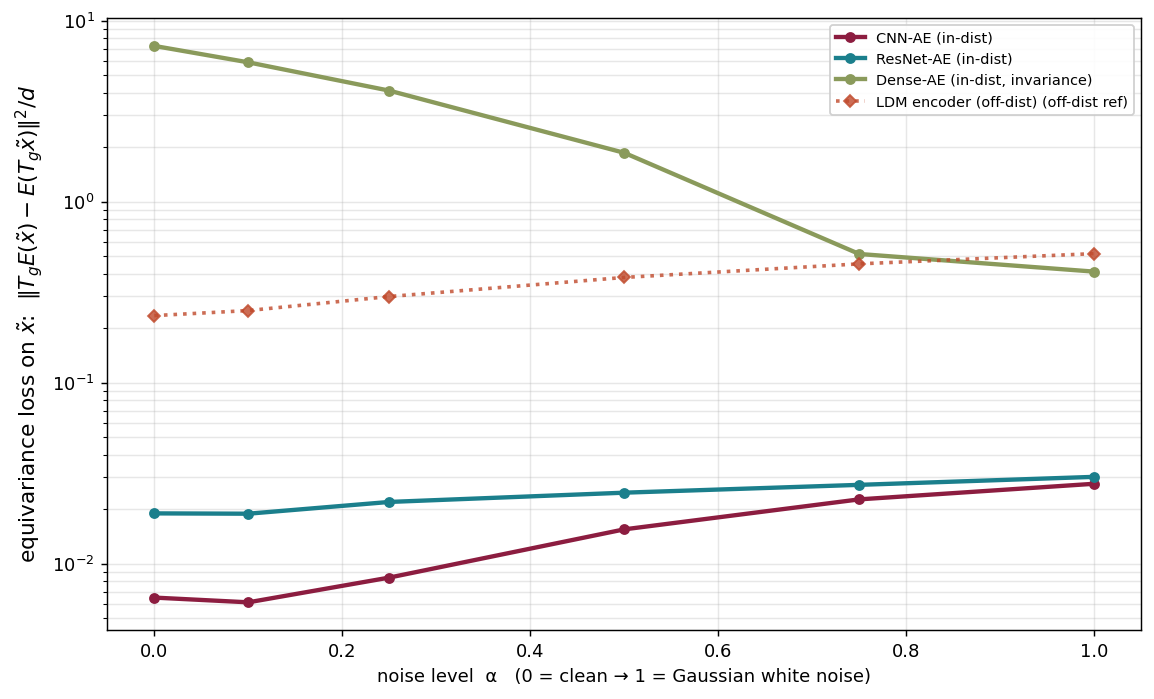}
\vspace{-4mm}
\caption{\textbf{Role of architecture on MPE strength under same data distribution.} MPE strength can be governed by architectural inductive biases. These analyses are on FFHQ datasets under flip transformation for only characterization. Hence, the resulting models are not used as general purpose pre-trained MPE functions.}
\label{fig:mpe_arch}
\vspace{-2mm}
\end{figure}
\begin{figure}[H]
\centering
\vspace{-1mm}
    \includegraphics[width=0.7\linewidth]{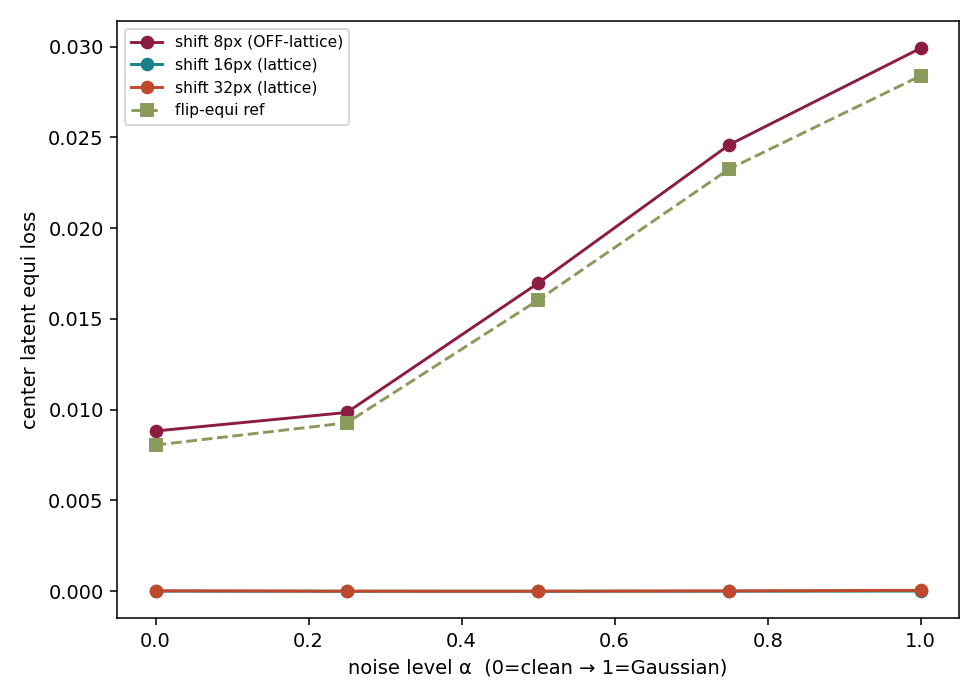}
\vspace{-4mm}
\caption{\textbf{An equivariant-by-construction function cannot be an MPE function for that group.} A CNN example. The function does not show the MPE property for on-lattice (stride) translation as the group transformation as the function would be equivariant to it globally. On the other hand, the function shows the MPE property for off-lattice translation as the group transformation.}
\label{fig:cnn_equi}
\vspace{-2mm}
\end{figure}

\begin{figure}[H]
\centering
\vspace{-1mm}
    \includegraphics[width=0.8\linewidth]{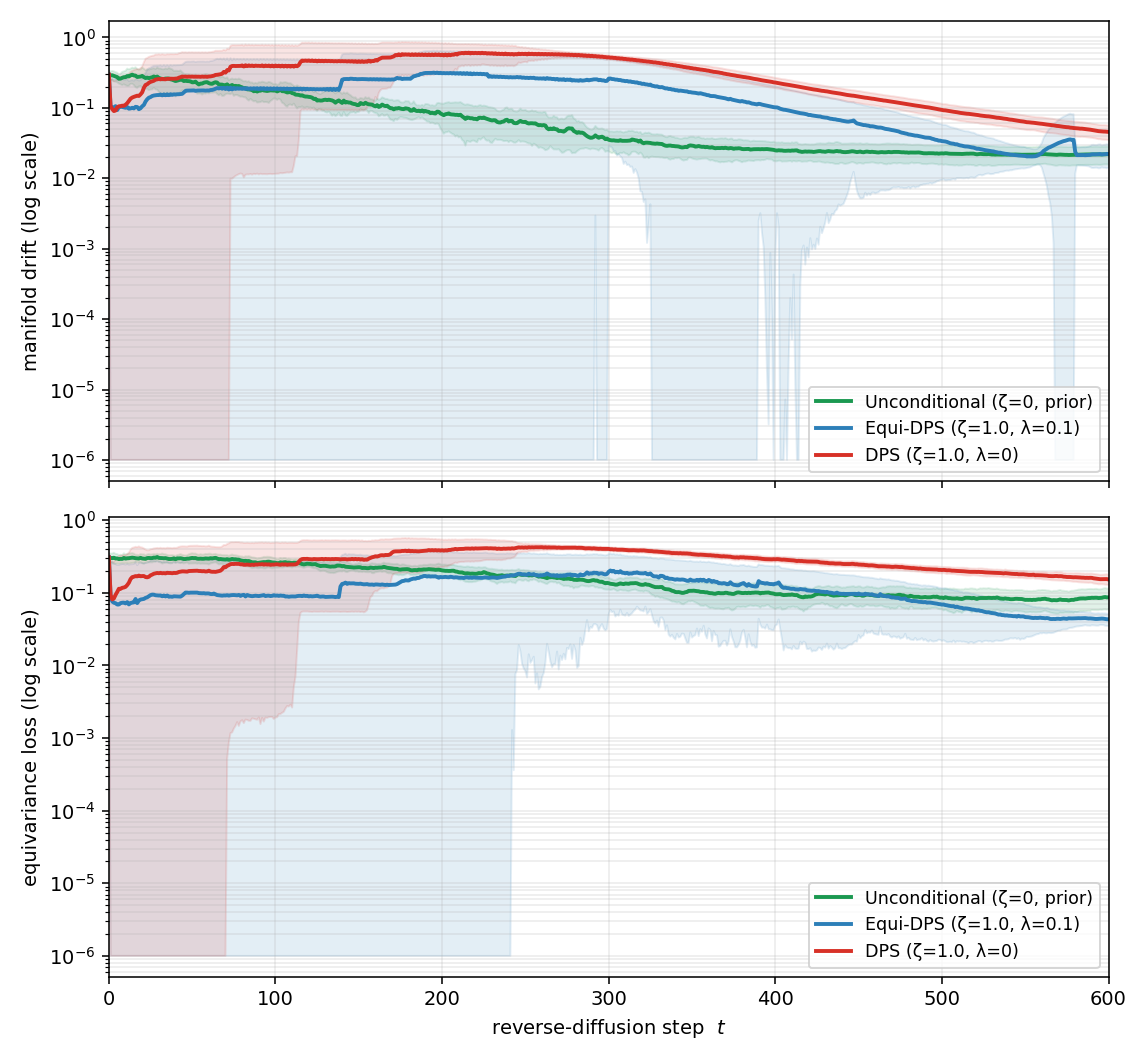}
\vspace{-4mm}
\caption{\textbf{Assessing manifold distance error and equivariance loss during unconditional and conditional generation.} The analysis is performed for first 600 reverse conditional sampling out of 1000 iterations.}
\label{fig:}
\vspace{-2mm}
\end{figure}
\begin{figure}[H]
\centering
\vspace{-1mm}
    \includegraphics[width=0.9\linewidth]{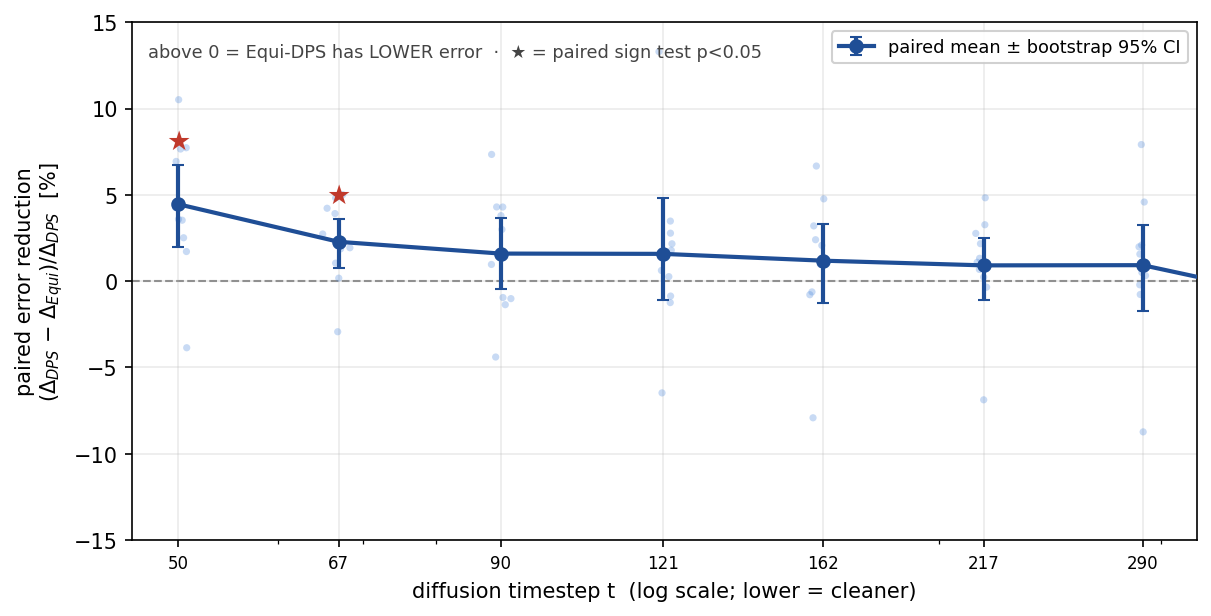}
\vspace{-4mm}
\caption{\textbf{Assessing likelihood error with and without EquiReg during conditional generation with DPS.}}
\label{fig:}
\vspace{-2mm}
\end{figure}
\begin{figure}[H]
\centering
\vspace{-1mm}
    \includegraphics[width=0.9\linewidth]{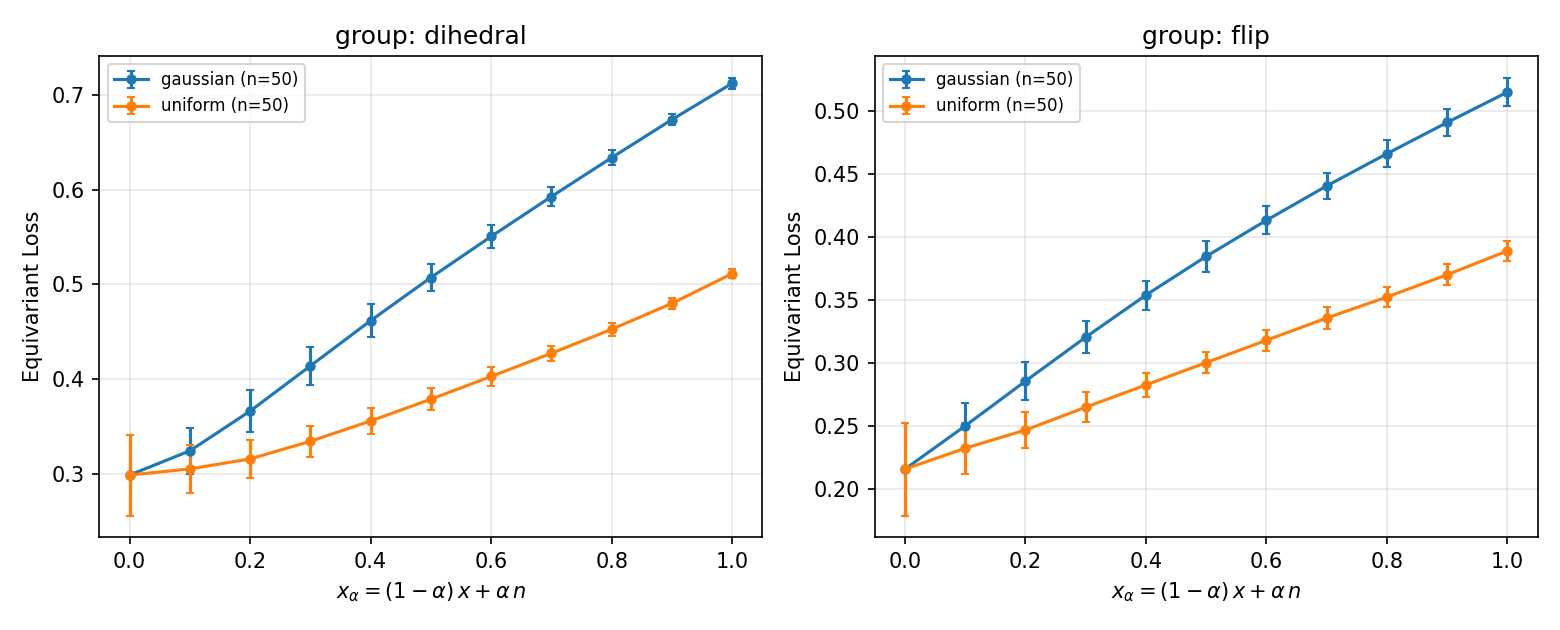}
\vspace{-4mm}
\caption{\textbf{White-noise experiment}.}
\label{fig:}
\vspace{-2mm}
\end{figure}
\begin{figure}[H]
\centering
\vspace{-1mm}
    \begin{subfigure}{.49\textwidth}
        \centering
\includegraphics[width=1.0\linewidth]{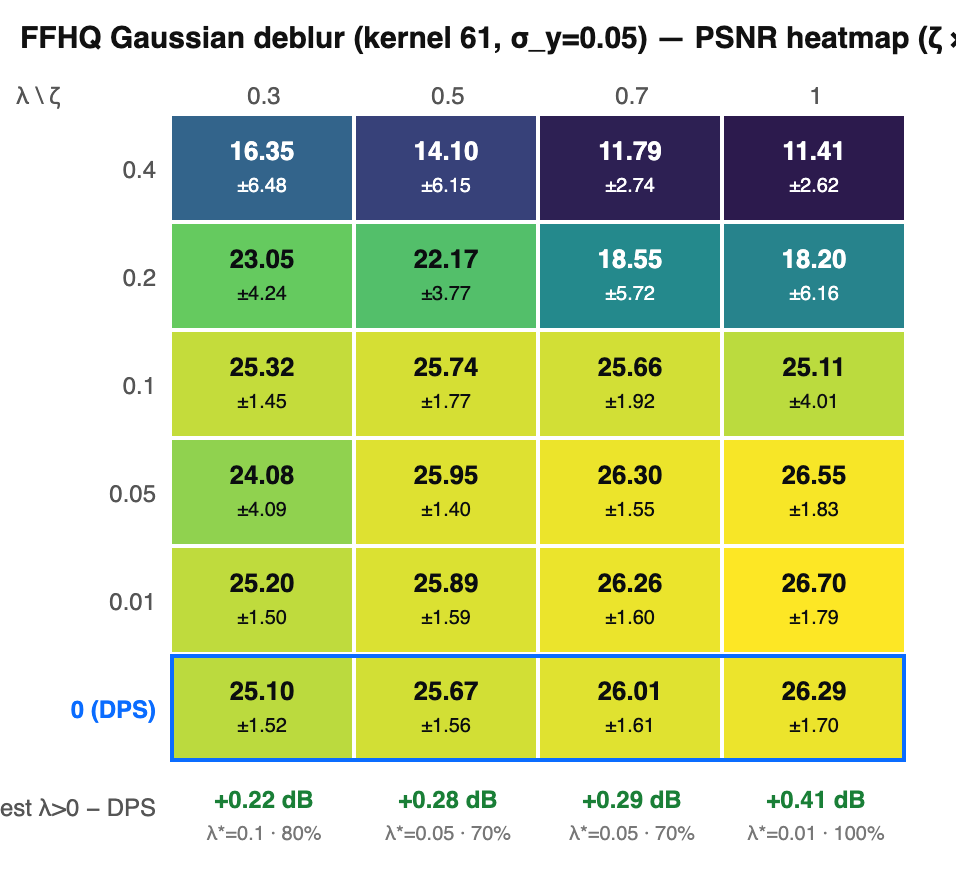}
    \end{subfigure}
    \begin{subfigure}{0.49\textwidth}
        \centering
\includegraphics[width=1.0\linewidth]{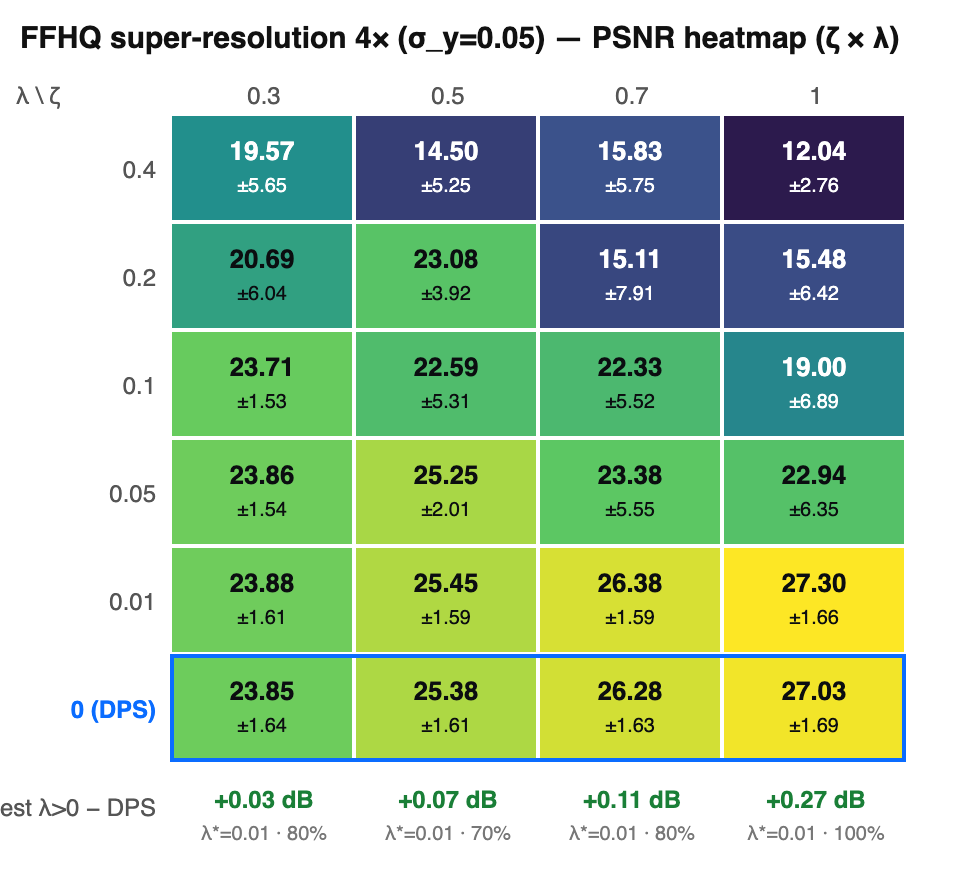}
    \end{subfigure}
\vspace{-4mm}
\caption{\textbf{Characterization of EquiReg on DPS as a function of regularization parameter $\lambda$ and the measurement step size $\zeta$.} For small values of $\zeta$, the inverse-solver control signal is weak, resulting in only minor deviations from the data manifold. Consequently, off-manifold trajectories are limited and EquiReg provides little benefit. However, in this regime, DPS often exhibits poor measurement consistency due to the weak guidance from the measurement term. As $\zeta$ increases, the measurement control signal becomes stronger, improving measurement consistency but also increasing the tendency of the sampling trajectory to move off the data manifold. In this setting, EquiReg becomes important, regularizing the trajectory toward the manifold while preserving the stronger measurement guidance. Therefore, the benefits of EquiReg become more pronounced as $\zeta$ increases.}
\label{fig:lambdaeta}
\vspace{-2mm}
\end{figure}

\begin{table}[t]
\centering
\caption{\textbf{EquiReq Characterization under change of EquiReg period, $\lambda_t$, and DDIM steps.}}
\begin{subtable}{\linewidth}
\fontsize{8}{11}\selectfont
\setlength{\tabcolsep}{5pt}
\renewcommand{\arraystretch}{1.1}
\centering
\begin{tabular}{r c cccc cccc}
\toprule
& & \multicolumn{4}{c}{\textbf{Super Resolution}} & \multicolumn{4}{c}{\textbf{Gaussian Blur}} \\
\cmidrule(lr){3-6}\cmidrule(lr){7-10}
\textbf{Method} & \textbf{Period} 
& Runtime (s) & PSNR$\uparrow$ & LPIPS$\downarrow$ & FID$\downarrow$
& Runtime (s) & PSNR$\uparrow$ & LPIPS$\downarrow$ & FID$\downarrow$ \\
\midrule
 DPS & N/A & 46.20 & 26.52\,(1.86) & 0.133\,(0.04) &
  -- & 46.50 & 25.87\,(2.06) & 0.125\,(0.03) & --\\
Equi-DPS & 1   & 51.10 & 26.73 {\scriptsize (1.99)} & 0.12 {\scriptsize (0.03)} & 87.97
               & 52.20 & 26.08 {\scriptsize (2.25)} &  0.12 {\scriptsize (0.03)} & 87.11 \\
Equi-DPS & 2   & 48.90 & 26.73 {\scriptsize (1.99)} & 0.12 {\scriptsize (0.03)} & 87.98
               & 49.10 & 26.06 {\scriptsize (2.24)} & 0.12 {\scriptsize (0.03)} & 87.19 \\
Equi-DPS & 5   & 47.10 & 26.73 {\scriptsize (1.99)} &  0.12 {\scriptsize (0.03)} & 87.98
               & 47.30 & 26.06 {\scriptsize (2.24)} &  0.12 {\scriptsize (0.03)} & 87.32 \\
Equi-DPS & 10  & 46.90 & 26.73 {\scriptsize (1.99)} & 0.12 {\scriptsize (0.03)} & 87.99
               & 47.00 & 26.05 {\scriptsize (2.24)} &  0.12 {\scriptsize (0.03)} & 87.04 \\
\bottomrule
\end{tabular}
\caption{Robustness and computational efficiency of applying EquiReg under various periods during sampling. EquiReg maintains performance when applied every $\{1,2,5,10\}$ DDIM steps while incurring minimal computational overhead.}
\end{subtable}

\begin{subtable}{\linewidth}
\fontsize{8}{9.5}\selectfont
\setlength{\tabcolsep}{3pt}
\renewcommand{\arraystretch}{1.1}
\centering
\begin{tabular}{r *{3}{c}}
\toprule
 & \multicolumn{3}{c}{\textbf{PSLD}} \\
\cmidrule(lr){2-4}
\textbf{$\lambda_t^{\mathrm{PSLD}}$} & PSNR$\uparrow$ & SSIM$\uparrow$ & LPIPS$\downarrow$ \\
\midrule
0.0   & 23.83 (2.61) & 0.63 (0.12)   & 0.315 (0.07) \\
0.01  & 25.35 (2.24) & 0.70 (0.09)   & 0.280 (0.07) \\
 0.1   &  26.63 (1.68) & 0.74 (0.08) & 0.337 (0.06) \\
0.25  & 26.22 (1.57) & 0.72 (0.08)   & 0.366 (0.05) \\
1.0   & 24.74 (1.28) & 0.66 (0.07)   & 0.438 (0.05) \\
\bottomrule
\end{tabular}
\caption{Robustness of EquiReg to the choice of $\lambda_t$. Sensitivity analysis for latent diffusion PSLD.}
\end{subtable}
\begin{subtable}{\linewidth}
\fontsize{8}{9.5}\selectfont
\setlength{\tabcolsep}{3pt}
\renewcommand{\arraystretch}{1.1}
\centering
\vspace{-4mm}
\end{subtable}
\vspace{-3mm}
\label{tab:equireg_charac}
\end{table}
%

\section{Computing Resources}\label{app:compute}

We conduct experiments on two NVIDIA GeForce RTX 4090 GPUs with 24 GB of VRAM. We note that we use pre-trained models and perform inference, so not much compute is required.

\section{Assets}\label{app:assets}

We use the publicly available code from PSLD (\url{https://github.com/LituRout/PSLD}), ReSample (\url{https://github.com/soominkwon/resample}), DPS (\url{https://github.com/DPS2022/diffusion-posterior-sampling}), and SITCOM (\url{https://github.com/sjames40/SITCOM}).

\section{Responsible Release}\label{app:resp_release}

Our approach uses only publicly available datasets and standard pre‑trained diffusion models, introducing no novel dual‑use or privacy risks. Consequently, no additional safeguards are required.

\end{document}